\documentclass{article}

 \usepackage[main,final]{conference_2026}

\usepackage[utf8]{inputenc} 
\usepackage[T1]{fontenc}    
\usepackage{hyperref}       
\usepackage{url}            
\usepackage{booktabs}       
\usepackage{amsfonts}       
\usepackage{nicefrac}       
\usepackage{microtype}      
\usepackage{xcolor}         
\usepackage{multirow}

\usepackage{pifont}
\usepackage{algorithm}
\usepackage{algpseudocode}
\usepackage{makecell}
\usepackage{dsfont}
\usepackage[most]{tcolorbox}
\usepackage{wrapfig}
\usepackage{enumitem}

\usepackage{amsmath}
\usepackage{amssymb}
\usepackage{mathtools}
\usepackage{amsthm}
\usepackage{svg}
\usepackage{threeparttable}
\usepackage[capitalize,noabbrev]{cleveref}

\theoremstyle{plain}
\newtheorem{theorem}{Theorem}[section]

\newtheorem{lemma}[theorem]{Lemma}

\theoremstyle{definition}
\newtheorem{definition}[theorem]{Definition}
\newtheorem{assumption}[theorem]{Assumption}
\theoremstyle{remark}
\newtheorem{remark}[theorem]{Remark}

\newcommand{\cmark}{\textcolor{green!70!black}{\ding{51}}} 
\newcommand{\xmark}{\textcolor{red!70!black}{\ding{55}}}   

\title{Hypergradient-based Bilevel Reinforcement Learning with Improved Sample Complexity}

\author{%
  Naman Saxena\\
  Purdue University\\
  West Lafayette, IN 47907 \\
  \texttt{saxen147@purdue.edu} \\
  \And
  Mudit Gaur\\
  Purdue University\\
  West Lafayette, IN 47907 \\
  \texttt{mgaur@purdue.edu} \\
  \And
  Vaneet Aggarwal\\
  Purdue University\\
  West Lafayette, IN 47907 \\
  \texttt{vaneet@purdue.edu} \\
}

\begin{document}

\maketitle

\begin{abstract}
 Bilevel reinforcement learning (RL) is an important framework within the literature of RL that can be used to formalize various categories of problems, such as meta-learning, hierarchical task decomposition, and reinforcement learning from human feedback (RL-HF). Most of the bilevel RL algorithms are either not scalable because of using hypergradient with Hessian, or they suffer from high sample complexity because of using penalty-based approximation methods. In this work, we propose a hypergradient-based bilevel RL algorithm using the optimality of the Boltzmann policy for the entropy regularized discounted RL objective function. Our proposed algorithm is Hessian-free and obtains an iteration complexity of $O(\epsilon^{-1})$ and state-of-the-art sample complexity of $\tilde{O}(\epsilon^{-2})$ under mild regularity conditions. Further, in our convergence analysis, we are able to remove the assumption of the Polyak-Łojasiewicz (PL) condition on the outer-level objective function present in the prior state-of-the-art sample complexity work.
 \end{abstract}

\section{Introduction}
The bilevel reinforcement learning (RL) framework has started receiving a lot of attention lately because of its ability to model different types of problems in RL, such as hierarchical task decomposition, meta-learning, hyper-parameter tuning, and reinforcement learning from human feedback (RL-HF). Specifically for RL-HF, the problem can be formulated as a bilevel RL problem where the outer-level problem optimizes the reward parameter and the inner-level problem optimizes the policy parameters. Most of the works in the bilevel optimization literature \citep{kwon2023fully,sow2022convergence, yang2023achieving} assume convexity for the inner-level objective function and cannot be extended to address the setting of bilevel RL, where the inner-level problem is non-convex. There are very few bilevel optimization works, such as \citep{kwon2024on,chen2024finding}, that have laid the foundation for extending bilevel optimization methods to bilevel RL. 

Recently, \citep{gaur2025sample,shen2025principled} proposed a penalty method-based bilevel RL algorithm. \citep{gaur2025sample}'s penalty method-based algorithm uses an approximate outer-level objective gradient based on the conclusion of \citep{kwon2024on,chen2024finding} and also requires the Polyak-Łojasiewicz (PL) condition to hold true for the outer-level objective function \citep{chen2024finding}. Because of using an approximate outer-level objective gradient, their penalty method-based bilevel algorithm obtains a sub-optimal sample complexity of $O(\epsilon^{-3})$. On the other hand, avoiding the assumption of the PL condition on the outer-level objective function, \citep{chakraborty2024parl} uses an exact gradient for the outer-level objective function and obtains a Hessian-based hypergradient bilevel RL algorithm. Bilevel RL algorithm provided by both \citep{chakraborty2024parl} and \citep{gaur2025sample} obtains an iteration complexity of $O(\epsilon^{-1})$. However, because of using the Hessian-based hypergradient, the algorithm proposed by \citep{chakraborty2024parl} faces scalability issues with function approximators with a large number of parameters. In this work, we provide a scalable bilevel RL algorithm that uses Hessian-free hypergradient, avoiding the approximate gradient for the outer-level objective function. At the same time, we remove the assumption of the PL condition on the outer-level objective function.  
\begin{definition}[Informal]\label{def:inf-real}
A policy class $\Theta$ is called realizable if it can represent all the Boltzmann policies.    
\end{definition}
Our algorithm is motivated by the optimality of the Boltzmann policy for the entropy regularized discounted reward objective function for RL \citep{nachum2017bridging}. \citep{yang2025bilevel} also exploited the optimality of the Boltzmann policy for finite spaces. However, we assume a non-realizable policy class and the hypergradient developed for finite space in \cite{yang2025bilevel} doesn't hold directly for the non-realizable parameterized policy class because of the lack of a feasible expression for the gradient of Q-value and value function (Remark \ref{rem:hypergrad}). In the case of parameterized policy, optimality of the Boltzmann policy holds if the policy class is realizable (Lemma \ref{lm:boltz}). Therefore, to obtain a hypergradient for an unrealizable parameterized policy class, we provide conditions (Assumption \ref{as:tv}) under which the optimality of the Boltzmann policy could be used. Further, we design gradient shifted functions (Section \ref{sc:4.1}) to act as a surrogate for the gradient of Q-value and value function to define an approximate hypergradient. This approximate hypergradient allows us to remove the PL condition on the outer-level objective function, present in prior state-of-the-art work \citep{gaur2025sample}. The PL condition on the outer level objective function was earlier necessary as it allowed \citep{gaur2025sample} to use their $O(\sigma)-$approximate gradient. Further, the use of Boltzmann policy optimality allows us to remove the Hessian inverse term from the hypergradient expression, which is otherwise present in the hypergradient provided by \citep{chakraborty2024parl}. We note here that Boltzmann policy optimality helps us by giving a closed-form expression of the optimal policy for the hypergradient expression. However, our inner-level still assumes a generally parameterized non-realizable policy class. Moreoever, the $O(\sigma)$ approximate outer level gradient used by \citep{gaur2025sample} leads to a tradeoff between gradient approximation error $O(\sigma^2)$ and sampling error $O(1/(\sigma^2B))$. The use of our hypergradient removes the tradeoff and improves the sample complexity from $\tilde{O}(\epsilon^{-3})$ to $\tilde{O}(\epsilon^{-2})$. The following are broad contributions of our work:
\begin{enumerate}[leftmargin=1.5em]
\item \textbf{Boltzmann policy optimality with unrealizable policy class}: We provide conditions under which Boltzmann policy optimality could be used with an unrealizable policy class (see Definition \ref{def:real}) and discuss its impact on convergence of the proposed algorithm (see Remark \ref{rem:conv}). 
\item \textbf{Algorithm proposal}: We proposed a Hessian-free scalable bilevel RL algorithm (Algorithm \ref{alg:1}) called Approximate Hypergradient Optimization (AHO) based on the optimality of the Boltzmann policy for the entropy regularized discounted reward RL objective function.
\item \textbf{Improving sample complexity}: We obtain an iteration complexity of $O(\epsilon^{-1})$ and a sample complexity of $\tilde{O}(\epsilon^{-2})$ (Theorem \ref{thm:conv}) for our algorithm AHO by improving the previous state-of-art sample complexity of $\tilde{O}(\epsilon^{-3})$ (see Table \ref{tab:comparsion1}).
\item \textbf{Removing PL/Unique Minimizer Assumption}: PARL \citep{chakraborty2024parl} assumes unique minimizer for the inner level. Further, the previous state-of-the-art bilevel RL algorithm \citep{gaur2025sample} assumes PL for the outer level objective. We remove both assumptions of the PL condition on the outer-level objective function and unique minimizer. 
\end{enumerate}

\begin{table*}[ht]
\begin{threeparttable}
    \centering
    \caption{ Comparison of Bilevel RL works}
    \begin{tabular}{|c|c|c|c|c|c|}
    \hline
     References& \makecell{PL\tnote{a} /Unique\\Minimizer\tnote{b}} & \makecell{Continuous \\ Space}  & \makecell{Hessian\\ Free} & \makecell{Iteration \\ Complexity} & \makecell{Sample \\ Complexity} \\    
    \hline
     \citep{hong2023two}& -\tnote{c} & \xmark  & \xmark &$O(\epsilon^{-2.5})$ & \xmark\\    
     \citep{chakraborty2024parl}& \cmark\tnote{b} & \cmark  & \xmark & $O(\epsilon^{-1})$ & \xmark\\    
     \citep{NEURIPS2024_e66309ea}& -\tnote{c} & \xmark & \cmark &$O(\epsilon^{-2})$ & \xmark\\    
     \citep{yang2025bilevel}& -\tnote{c} &\xmark  & \cmark & $O(\epsilon^{-1.5})$ & $O(\epsilon^{-3.5})$\\    
     \citep{shen2025principled}& -\tnote{c} & \xmark  & \cmark&$O(\epsilon^{-1})$& - \\    
     \citep{gaur2025sample}& \cmark\tnote{a} & \cmark  & \cmark &$O(\epsilon^{-1})$& $\tilde{O}(\epsilon^{-3})$\\    
     Ours &\xmark\tnote{a} & \cmark  & \cmark & $O(\epsilon^{-1})$ & $\tilde{O}(\epsilon^{-2})$\\    
    \hline
    \end{tabular}
    \label{tab:comparsion1}
\begin{tablenotes}
\footnotesize
\item[a] Assumes PL on outer level objective.
\item[b] Assumes a unique minimizer for the inner level problem.
\item[c] Assumes finite space.
\end{tablenotes}
\end{threeparttable}
\end{table*}

\section{Problem Formulation}\label{sc:3}
\textbf{Markov Decision Process (MDP).} A Markov Decision Process is characterized by the tuple $\{\mathcal{S}, \mathcal{A}, r, \pi, \gamma, \mathcal{P}\}$. Here, $\mathcal{S} \subseteq \mathbb{R}^{n}$ represents continuous state space, $\mathcal{A}\subseteq \mathbb{R}^{m}$ represents continuous actions space. Further, $r:\mathcal{S}\times\mathcal{A}\times\mathcal{X} \mapsto [-R_{max}, R_{max}]$ is the reward function parameterized by $x \in \mathcal{X}$, $\mathcal{P}(\cdot|s,a): \mathcal{S}\times\mathcal{A} \mapsto \mu_1(\cdot)$ is the transition probability function where $\mu_1(\cdot)$ is a probability measure over the space $\mathcal{S}$, $\pi(\cdot|s,\theta): \mathcal{S} \mapsto \mu_2(\cdot)$ is the stochastic policy parameterized by $\theta$ where $\mu_2(\cdot)$ is a probability measure over the space $\mathcal{A}$ and $\gamma$ is the discount factor. For an MDP, we solve for the entropy regularized objective function $J(\theta, x)$ in Eq. \eqref{eq:obj} as the Boltzmann policy is optimal for that function.
\begin{equation}\label{eq:obj}
\max_{\theta}\;\; J(\theta,x):= \;\; \frac{1}{1-\gamma}\mathbb{E}_{(s,a)\sim d^{\pi(\theta)}}\Big[r(s,a,x) - \tau\log\pi(a|s, \theta)\Big]    
\end{equation}
Here, $d^{\pi_\theta}(s,a) = d^{\pi_{\theta}}_{s}(s)\pi(a|s,\theta)$, where $d^{\pi_{\theta}}_s$ is the discounted state occupancy probability density for parameterized policy with parameter $\theta$. Further, $Q$-value function, for parameterized policy $\pi$ with parameter $\theta$ and reward function parameterized by $x$, is defined as $Q^{\pi}(s,a, \theta, x) = r(s,a,x) + \gamma\mathbb{E}_{s' \sim P(\cdot|s,a)}[V^{\pi}(s', \theta, x)]    
$ and value function is defined as $V^{\pi}(s, \theta, x) = \mathbb{E}_{a \sim \pi(\cdot|s, \theta)}[-\tau\log\pi(a|s,\theta) + Q^{\pi}(s,a, \theta, x)]$, where $\tau$ is the temperature coefficient that controls the exploratory nature of the policy. Using the definition of value and Q-value function, we now define the Boltzmann policy ($\pi^{B}$) in Eq. \eqref{eq:boltz} for a policy $\pi$ with parameter $\theta$. Let $V^{\pi,B}(s,\theta, x)$, defined in Eq. \eqref{eq:boltz}, be the Boltzmann value function for policy parameter $\theta$ and reward parameter $x$. 
\begin{equation}\label{eq:boltz}
\begin{split}
\pi^{B}(a|s,\theta, x)&:= \frac{\exp(Q^{\pi}(s,a, \theta, x)/\tau)}{\exp(V^{\pi,B}(s,\theta, x)/\tau)}\\
where \quad V^{\pi,B}(s,\theta, x) &= \tau \log\Big(\int \exp(Q^{\pi}(s,a,\theta,x)/\tau)\;da\Big)
\end{split}
\end{equation}
\begin{definition}\label{def:real}
A policy class $\Theta$ is called realizable if  $\forall\theta\in\Theta\;\exists\theta_{B}(x)\in\Theta\; \pi(a|s,\theta_{B}(x)) = \pi^{B}(a|s,\theta,x)$.    
\end{definition}
\begin{lemma}\label{lm:boltz}
Let $J(\theta,x)$ (Eq. \eqref{eq:obj}) be the objective function for RL defined for policy class $\Theta$ where $\theta \in \Theta$ is the policy parameter, and $x$ is the reward parameter. If policy class $\Theta$ is realizable and there exist $\theta'(x) \in \Theta$ such that $\pi(a|s,\theta'(x)) = \pi^{B}(a|s, \theta'(x),x)$ then $\theta'(x)$ is the optimal policy parameter for $J(\theta,x)$.
\end{lemma}
A non-parameterized Boltzmann policy is the optimal policy for Eq. \eqref{eq:obj} in the case of a finite state and action space \citep{nachum2017bridging}. In Lemma \ref{lm:boltz}, we show that if the Boltzmann policy lies in the class of parameterized policies, it is optimal for the parameterized policy setting as well. A proof for the same is given in the Appendix \ref{lm:app-boltz}. Therefore, Lemma \ref{lm:boltz} along with Eq. \eqref{eq:boltz} provides a closed-form solution for the optimal policy and is used to eliminate the Hessian from the hypergradient as discussed in Section \ref{sc:4}. Next, we define how the objective function defined in Eq. \eqref{eq:obj} is used in the context of Bilevel reinforcement learning. 

\textbf{Bilevel Reinforcement Learning.} The bilevel RL problem consists of an inner-level and an outer-level objective function. Entropy regularized RL objective function (Eq. \eqref{eq:obj}) is used as the inner level objective function. The objective function at the outer level is optimized at the optimal solution of the inner-level objective function. The bilevel RL problem is described by Eq. \eqref{eq:biobj} where $\phi(x,\theta^{*}(x))$ is the non-convex outer level objective function and $\theta^{*}(x)$ is a localized selection map as described in Assumption \ref{as:selector} and represents one of the optimal solutions of the inner-level objective function. Please note that $\theta^{*}(x)$ is a single-valued map. Here, $J(\theta, x)$ is the non-convex inner level objective function. 
\begin{equation}\label{eq:biobj}
\begin{split}
&\min_{x} \phi(x, \theta^{*}(x)) \\ 
&\textbf{s.t}\;\; \theta^{*}(x) \in \arg\min_{\theta} -J(\theta, x)    
\end{split}
\end{equation}
Here, $\phi(x,\theta) = -\mathbb{E}_{d_i \sim P(d_i, \theta)} [l(\{d_i\}_{0}^{I-1},x)]$ 
,where $\{d_i\}_{0}^{I-1}$ is a collection of $I$ trajectories and each trajectory is defined as $d=\{s_t,a_t\}_{t=0}^{H-1}$ is state-action trajectory and sampled from the distribution $P(d,\theta)$ defined in Eq. \eqref{eq:P_traj}. For, $I=2$, $\mathbb{E}_{d_i \sim P(d_i, \theta)} [l(\{d_i\}_{0}^{1},x)]$ is defined in Eq. \eqref{eq:RL-HF} using Bradley-Terry model \citep{bradley1952rank}. We note that our method isn't limited to the Bradley-Terry model and can handle any preference-based model.  Let $(d_0,l_0,d_1,l_1)$ represents pair of trajectory with $l_1=1,l_0=0$ denoting $d_1$ is preferred over $d_0$ and vice-versa and $P(d_a \succ d_b)$ is defined as Eq. \eqref{eq:pdadb} with $R_x(d)=\sum_{s_t,a_t\in d}r(s_t,a_t,x)$.
\begin{equation}\label{eq:RL-HF}
\begin{split}
\mathbb{E}_{d_i \sim P(d_i, \theta)} [l(\{d_i\}_{0}^{1},x)] =&  \mathbb{E}_{d_0,d_1,l_0,l_1 \sim \pi(\theta)}\Big[l_0\log P(d_0\succ d_1|x) + l_1\log P(d_1\succ d_0|x)\Big]  
\end{split}
\end{equation}
\begin{equation}\label{eq:pdadb}
P(d_a\succ d_b|x) = \frac{\exp(R_x(d_a))}{\exp(R_x(d_a)) + \exp(R_x(d_b))}
\end{equation}
\begin{equation}\label{eq:P_traj}
\log P(d, \theta) = \log\rho_{0}(s_0) + \sum_{i=0}^{H-2}(\log\pi(a_{i}|s_{i},\theta) + \log P(s_{i+1}|s_{i}, a_{i})) + \log\pi(a_{H-1}|s_{H-1},\theta)    
\end{equation}
Here, $J(\theta, x)$ is defined as in Eq. \eqref{eq:obj}, which is an entropy regularized discounted reward performance metric. In the next section, we will discuss our approach for optimizing the outer level objective function $-\mathbb{E}_{d_i \sim P(d_i, \theta^*(x))} [l(\{d_i\}_{0}^{I-1},x)]$ for the optimal policy parameter $\theta^{*}(x)$. 

\section{Proposed Approach}\label{sc:4}
In this section, we will first define gradient shifted value functions and then provide the hypergradient theorem for the outer objective function, followed by a discussion on the proposed bilevel RL algorithm. We note that for the rest of the paper, $\nabla$ and $\nabla_x$  refer to the full derivative and partial derivative w.r.t. the reward parameter $x$, respectively. 
\subsection{Gradient Shifted Value functions}\label{sc:4.1}
We define $U^{\pi}(s,\theta, x)$ as the gradient shifted value function and $W^{\pi}(s, a, \theta,x)$ is defined as the gradient shifted Q-value function for the policy parameter $\theta$ and the reward parameter $x$ in Eq. \eqref{eq:uw}. $U^{\pi}(s,\theta(x), x)$ represents the long-term discounted summation of the gradient of the reward ($\nabla  r(s,a,x)$) starting from state $s$ under the policy $\pi^{B}(\cdot|s_t,\theta,x)$. On the other hand, $W^{\pi}(s, a, \theta,x)$ denotes long term discounted summation of reward gradient ($\nabla  r(s,a,x)$) starting from state $s$ and action $a$ under the policy $\pi^{B}(\cdot|s,\theta,x)$. $U^{\pi}(s,\theta^*(x), x)$ and $W^{\pi}(s, a,\theta^*(x), x)$ are designed to act as surrogate for $\nabla V^{\pi,B}(s,\theta^{*}(x), x)$ and $\nabla Q^{\pi}(s, a,\theta^{*}(x), x)$ and help in defining hypergradient theorem in the next subsection.   
\begin{equation}\label{eq:uw}
\begin{split}
U^{\pi}(s,\theta, x) &= \mathbb{E}^{\pi^{B}}\Big[\sum_{t=0}^{\infty}\gamma^t\nabla  r(s_t,a_t,x)\big | s_0 = s, a_t \sim \pi^{B}(\cdot|s_t,\theta,x)\Big]     \\
W^{\pi}(s, a, \theta,x) & = \mathbb{E}^{\pi^{B}}\Big[\sum_{t=0}^{\infty}\gamma^t\nabla  r(s_t,a_t,x)\big | s_0 = s, a_0 = a, a_t \sim \pi^{B}(\cdot|s_t,\theta,x)\Big]
\end{split}
\end{equation}
\begin{assumption}\label{as:tv}
Let $\Theta$ be the class of parameterized policy and  $\pi(\cdot|s,\theta^{*}(x))$ be the optimal policy w.r.t. the objective function $-J(\theta,x)$. Further, $\pi^{B}(a|s,\theta^{*}(x),x)$ be the Boltzmann policy. Because of the limited representational capacity of the policy class, we make the following assumption:
\begin{enumerate}
    \item $\int \pi(a|s,\theta^{*}(x))|\log\pi^{B}(a|s,\theta^*(x),x) - \log\pi(a|s,\theta^{*}(x))|\;da \leq \epsilon_{kl}$
    \item $\int \pi(a|s,\theta^{*}(x))\|\nabla_{x}\log\pi^{B}(a|s,\theta^*(x),x) - \nabla_{x}\log\pi(a|s,\theta^{*}(x))\|\;da \leq \epsilon_{fd}$
\end{enumerate}
\end{assumption}
Assumption \ref{as:tv} arises because of the limited capacity of the policy class $\Theta$. If we can ensure a sufficiently rich policy class such that all Boltzmann policies can be represented by some $\theta \in \Theta$, then $\epsilon_{kl}$ and $\epsilon_{fd}$ can be driven to zero.

\begin{lemma}\label{lm:nablaqv}
Let $U^{\pi}(s,\theta, x)$ be the gradient shifted value function and $W^{\pi}(s, a, \theta, x)$ be the gradient shifted Q-value function. Let $\theta^{*}(x)$ be the optimal policy parameter for the entropy regularized RL objective function (Eq. \eqref{eq:obj}). \\
\begin{enumerate}[leftmargin=1.5em]
\item The following recurrence relation is satisfied by the functions $U^{\pi}$ and $W^{\pi}$:
\begin{equation}
\begin{split}
&U^{\pi}(s,\theta^*(x),x) = \int \pi^{B}(a|s,\theta^{*}(x),x)\big(\nabla  r(s,a,x) + \gamma\mathbb{E}_{s' \sim P(\cdot|s,a)}[U^{\pi}(s',\theta^*(x),x)]\big) \;da\\
&W^{\pi}(s, a, \theta^*(x),x) 
 = \nabla  r(s,a, x) + \gamma\mathbb{E}_{s' \sim P(\cdot|s,a)}[U^{\pi}(s',\theta^*(x),x)]\\
\end{split}    
\end{equation}
\item $U^{\pi}(s,\theta^*(x), x)$ and $W^{\pi}(s, a,\theta^*(x), x)$ act as surrogate for $\nabla V^{\pi,B}(s,\theta^*(x), x)$ and $\nabla Q^{\pi}(s, a,\theta^*(x), x)$ with the following error bound: 
\begin{equation}
\begin{split}
&\|U^{\pi}(s,\theta^*(x),x) - \nabla  V^{\pi,B}(s,\theta^*(x),x)\| \leq  O(\tau (\epsilon_{kl} + \epsilon_{fd}))  \\
&\|W^{\pi}(s,a,\theta^*(x),x) - \nabla  Q^{\pi}(s, a, \theta^*(x),x)\|
\leq O(\tau(\epsilon_{kl} + \epsilon_{fd})) 
\end{split}    
\end{equation}
\end{enumerate}
Here, $\epsilon_{kl}$ and $\epsilon_{fd}$ are errors due to unrealizable policy class as defined in Assumption $\ref{as:tv}$ and $\tau$ is the temperature parameter defined for the Boltzmann policy.
\end{lemma}
A proof of Lemma \ref{lm:nablaqv} is provided in the Appendix \ref{lm:app-nablaqv}. The recurrence relations provided for the gradient shifted value functions help us in obtaining their estimate using a Q-learning type algorithm by considering the gradient of the reward function ($\nabla r(s,a,x)$) as reward. Moreover, $U^{\pi}(s,\theta^*(x), x)$ and $W^{\pi}(s, a,\theta^*(x), x)$ becomes exactly equal to $\nabla V^{\pi,B}(s,\theta^*(x), x)$ and $\nabla Q^{\pi}(s, a,\theta^*(x), x)$ respectively when the policy class is realizable and $\epsilon_{kl},\epsilon_{fd}=0$ (Assumption \ref{as:tv}). In the next subsection, we will discuss our hypergradient theorem.

\subsection{Approximate Hypergradient Theorem}
 The hypergradient of the outer level objective function ($\phi(x,\theta^{*}(x))$) w.r.t. the reward parameter $x$ provided by \citep{chakraborty2024parl} involves the gradient of the optimal solution of the inner level optimization problem ($\nabla\theta^{*}(x)$). The expression for $\nabla\theta^{*}(x)$ obtained using the Implicit Function Theorem \citep{lorraine2020optimizing} introduces a Hessian-based term in the hypergradient. However, we exploit the optimality of the Boltzmann policy (Eq. \eqref{eq:boltz}) for the entropy regularized discounted performance metric (Eq. \eqref{eq:obj}) to remove the Hessian in the expression. We provide the expression of the hypergradient in Theorem \ref{thm:gradient}. Further, we provide the expression of approximate hypergradient in Eq. \eqref{eq:approx-grad} using gradient shifted value functions defined in Section \ref{sc:4.1}. 
\begin{equation}\label{eq:approx-grad}
\begin{split}
\nabla \tilde{\phi}&(x,\theta^{*}(x))= \mathbb{E}_{d_i \sim P(d_i, \theta^*(x))} [\nabla_x l(\{d_i\}_{0}^{I-1}, x)] \\
&+ \tau^{-1}\mathbb{E}_{d_i \sim P(d_i, \theta^*(x))} \Big[l(\{d_i\}_{0}^{I-1}, x) \sum_{i}\sum_{t}(W^{\pi}(s_t^{i},a_t^{i},\theta^{*}(x),x) - U^{\pi}(s_t^{i},\theta^{*}(x),x))\Big]\\
\end{split}    
\end{equation}

\begin{theorem}\label{thm:gradient}
The hypergradient ($\nabla \phi(x,\theta^{*}(x))$) for the outer objective function $\phi(x,\theta^{*}(x))$ is expressed as
\begin{equation}
\begin{split}
\nabla \phi(x,\theta^{*}(x))=&\mathbb{E}_{d_i \sim P(d_i, \theta^*(x))} [\nabla_x l(\{d_i\}_{0}^{I-1}, x)] \\
&+ \mathbb{E}_{d_i \sim P(d_i, \theta^*(x))} \Big[l(\{d_i\}_{0}^{I-1}, x) \sum_{i}\sum_{h}\nabla \log\pi(a_h^{i}|s_h^{i},\theta^{*}(x))\Big]
\end{split}    
\end{equation}
Let the approximate hypergradient be defined in Eq. \ref{eq:approx-grad}, then the approximation error is given by $\Psi(x,\theta^{*}(x))) = \nabla \phi(x,\theta^{*}(x)) - \nabla \tilde{\phi}(x,\theta^{*}(x)) $ and $\|\Psi(x,\theta^{*}(x)))\| = O(\epsilon_{fd}) + O(\epsilon_{kl})$. Further, $I$ is the number of trajectories, $H$ is the length of each trajectories, $\epsilon_{fd}$ and $\epsilon_{kl}$ are errors due to unrealizable policy class as defined in Assumption \ref{as:tv}, $\tau$ is the temperature coefficient from Eq. \eqref{eq:boltz} and $|l(\{d_i\}_{0}^{I-1}, x)|\leq C_l$ (Lemma \ref{lem:l-bound}). Further, $W^{\pi}(s,a,\theta^{*}(x),x)$ and $U^{\pi}(s,\theta^{*}(x),x)$ are the gradient-shifted Q-value function and the gradient-shifted value function defined in Eq. \eqref{eq:uw}.   
\end{theorem}
\begin{proof}
A detailed proof of the theorem is presented in Appendix \ref{thm:app-gradient}.  
\end{proof}
\begin{remark}\label{rem:hypergrad}
We obtain an approximate hypergradient for the outer level objective because the policy class is not guaranteed to be realizable. The hypergradient approximation error is of the order of $O(\epsilon_{fd}) + O(\epsilon_{kl})$. If the policy class is realizable, the error would disappear. Further, the expression for the hypergradient defined by \citep{yang2025bilevel} for finite spaces using optimality of Boltzmann policy includes the gradient of Q-value ($\nabla Q^{\pi}$) and the value function ($\nabla V^{\pi}$). Because of non-realizability, the exact hypergradient contains the gradient of the Boltzmann value function ($\nabla V^{\pi,B}$). The Boltzmann value function is related to value function as $V^{\pi}(s, \theta, x)
= V^{\pi,B}(s,\theta, x) - \tau D_{KL} (\pi(\cdot|s,\theta)||\pi^{B}(\cdot|s,\theta, x))$ (Lemma \ref{lm:valueboltz}), and therefore, doesn't yield feasible recursive definition in conjunction with gradient of Q-value function because of the $D_{KL} (\pi(\cdot|s,\theta)||\pi^{B}(\cdot|s,\theta, x))$ term. Hence, the approximate hypergradient was defined using gradient shifted value functions ($U^{\pi}$ and $W^{\pi}$), which act as a surrogate for $\nabla V^{\pi,B}$ and $\nabla Q^{\pi}$.      
\end{remark}
\subsection{Approximate Hypergradient Optimization (AHO) Algorithm}
Algorithm \ref{alg:1} describes our proposed bilevel RL algorithm called Approximate Hypergradient Optimization (AHO). A more detailed description of AHO is provided in Appendix \ref{sc:algo}. The update rule for the inner level problem uses the policy gradient theorem \citep{sutton1999policy} and is provided as Eq. \eqref{eq:inner-update}.
\begin{equation}\label{eq:inner-update}
 \nabla_{\theta}J(\theta^{k}_t, x_t, B) = \frac{1}{B}\sum_{j=0}^{B-1}\nabla_{\theta}\log\pi(a_j|s_j,\theta^{k}_t)\hat{Q}^{\pi}_{y}(s_j, a_j, \theta^{k}_t, x_t) 
\end{equation}
In Eq. \eqref{eq:inner-update}, we consider $B$ samples of a state-action pair for the first term of the gradient and $B$ trajectories of length $H$ for the second term of the gradient. Further, $\hat{Q}^{\pi}_{y}$ is an estimator of Q-value function $Q^{\pi}$ with parameter $y$ considering reward as $r(s,a,x_t) -\tau\log\pi(a|s,\theta_t^k)$. Parameter $y$ is updated using the Q-learning type Algorithm \ref{alg:3} given in Appendix \ref{sc:algo}. 
The update rule for the outer level objective is based on Theorem \ref{thm:gradient} and is provided as Eq. \eqref{eq:outer-update}. Also, the definition of $\nabla \tilde{\phi}^{(i)}(x_t, \theta^{K}_{t},B)$ and $\nabla \tilde{\phi}^{(ii)}(x_t, \theta^{K}_{t},B)$ is provided in Eq. \eqref{eq:phi12}.  
\begin{equation}\label{eq:outer-update}
\begin{split}
\nabla \tilde{\phi}(x_t, \theta^{K}_{t},B) =& \nabla \tilde{\phi}^{(i)}(x_t, \theta^{K}_{t},B) + \frac{1}{\tau}\nabla \tilde{\phi}^{(ii)}(x_t, \theta^{K}_{t},B) \\
\end{split}
\end{equation}
\begin{equation}\label{eq:phi12}
\begin{split}
\nabla \tilde{\phi}^{(i)}(x_t, \theta^{K}_{t},B) :=& \frac{1}{B}\sum_{j=0}^{B-1}\nabla_x l(\{d_{i,j}\}_{0}^{I-1}, x)\\
\nabla \tilde{\phi}^{(ii)}(x_t, \theta^{K}_{t},B):=&\frac{1}{B}\sum_{j=0}^{B-1}\Big(l(\{d_{i,j}\}_{0}^{I-1}, x_t) \sum_{i}\sum_{h}\big(\widehat{W}^{\pi}_{off,w}(s_h^{i,j},a_h^{i,j},\theta^K_t,x_t) \\ 
&- \widehat{U}^{\pi}_{off,z}(s_h^{i,j},\theta^K_t,x_t)\big)\Big) 
\end{split}
\end{equation}
In Eq. \eqref{eq:phi12}, we consider $B$ samples of trajectory of horizon length $H$ with $j$-th trajectory denoted as $d_{i,j} = \{s^{i,j}_t, a^{i,j}_t\}_{h=0}^{H-1}$. Further, $\widehat{W}^{\pi}_{off,w}$ is an estimator of $W^{\pi}_{off}$ with parameter $w$. The parameter $w$ is updated using the Bellman error based on recursion for $W^{\pi}_{off}$ in Eq. \eqref{eq:uoff}. Also, $\widehat{U}^{\pi}_{off,w}$ is an estimator of $U^{\pi}_{off}$ with parameter $z$ and $z$ is also updated in the same manner as $w$ using the recursion for $U^{\pi}$ in Eq. \eqref{eq:uoff}.   
\begin{equation}\label{eq:uoff}
\begin{split}
U^{\pi}_{off}(s,\theta^{*}(x),x) &= \int \pi(a|s,\theta^{*}(x))\big(\nabla  r(s,a,x) + \gamma\mathbb{E}_{s' \sim P(\cdot|s,a)}[U^{\pi}_{off}(s',\theta^{*}(x), x)]\big) \;da \\   
W^{\pi}_{off}(s, a, \theta^{*}(x),x) 
& = \nabla  r(s,a, x) + \gamma\mathbb{E}_{s' \sim P(\cdot|s,a)}[U^{\pi}_{off}(s',\theta^{*}(x), x)]\\
\end{split}
\end{equation}
\begin{remark}\label{rem:uoff}
We have defined quantities $U^{\pi}_{off}(s,\theta,x)$ and $W^{\pi}_{off}(s, a, \theta,x)$ in Eq. \eqref{eq:uoff} based on $U^{\pi}(s,\theta,x)$ and $W^{\pi}(s, a, \theta,x)$ from Lemma \ref{lm:nablaqv} because the former depends $\pi(a|s,\theta^{*}(x))$ which can be obtained from the Algorithm \ref{alg:1} while the latter depends on $\pi^{B}(a|s,\theta^{*}(x),x)$ which cannot be obtained for unrealizable policy class.      
\end{remark}
\begin{wrapfigure}{L}{0.5\textwidth}
\vspace{-20pt}
\begin{minipage}{0.5\textwidth}
\begin{algorithm}[H]
\caption{Approximate Hypergradient Optimization (AHO) Algorithm}\label{alg:1}
\begin{algorithmic}
\State \textbf{Input:} State space $\mathcal{S}$, Action space $\mathcal{A}$, Number of outer level updates $T$, Number of inner level updates $K$, Batch size $B$, Horizon length $H$. Initial policy parameters $\theta_{0}^{0}$. Initial reward parameter $x_0$. 
\For{ $t \in \{0,\ldots,T-1\}$}
\For{ $k \in \{0,\ldots,K-1\}$}
\State $\theta^{k+1}_{t} = \theta^{k}_t + \alpha_t \nabla_{\theta} J(\theta^{k}_t,x_t,B) $
\EndFor
\State $x_{t+1} = x_{t} - \beta_t\nabla \tilde{\phi}(x_t, \theta^{K}_t,B)$
\EndFor
\State \textbf{Output:} $x_{T}$
\end{algorithmic}    
\end{algorithm}
\end{minipage}
\vspace{-59pt}
\end{wrapfigure}
In the next section, we show through our theoretical analysis that even with an approximate hypergradient, the first-order stationary convergence guarantee still holds for our proposed algorithm.       
\section{Theoretical Analysis}\label{sc:5}

In this section, first, we will discuss the assumptions taken by us to provide the iteration and sample complexity bounds. Subsequently, we will provide a proof sketch for our main theorem (Theorem \ref{thm:conv}) for the convergence rates.\\


\begin{assumption}\label{as:lip}
Let $\pi(a|s,\theta)$ be the parameterized policy, $V^{\pi}(s,\theta,x)$ and $Q^{\pi}(s,a,\theta,x)$ be the corresponding value function and Q-value function, and $r(s,a,x)$ be the reward function for the reward parameter $x$. These functions satisfy the following Lipschitz continuity condition: 
\begin{enumerate}
\item $\|V^{\pi}(s,\theta_1,x_1) - V^{\pi}(s,\theta_2,x_2)\| \leq L_{x,V}\|x_1 - x_2\| + L_{\theta,V}\|\theta_1 - \theta_2\|$    
\item $\|Q^{\pi}(s,a,\theta_1,x_1) - Q^{\pi}(s,a,\theta_2,x_2)\| \leq L_{x,Q}\|x_1 - x_2\| + L_{\theta,Q}\|\theta_1 - \theta_2\|$
\item $\|\pi(a|s,\theta_1) - \pi(a|s,\theta_2)\| \leq L_{\pi}\|\theta_1 - \theta_2\|$
\item $\|\log\pi(a|s,\theta_1) - \log\pi(a|s,\theta_2)\| \leq L_{\pi,log}\|\theta_1 - \theta_2\|$
\item $|r(s,a,x_1) - r(s,a,x_2)|\leq L_r\|x_1-x_2\|$
\end{enumerate}
\end{assumption}
Assumption \ref{as:lip} is a common regularity assumption for convergence in the bilevel optimization literature \citep{gaur2025sample,grazzi2023bilevel,chen2024finding}.

\begin{assumption}\label{as:log_reward}
There exist constants $R_{\max}>0$ and $C_{\log}>0$ such that, for all
$s\in\mathcal S$, $a\in\mathcal A$, $\theta\in\Theta$, and $x\in\mathcal X$ such that $|r(s,a,x)|\leq R_{\max}$ and $|\log \pi(a|s,\theta)|\leq C_{\log}$.
\end{assumption}
The boundedness of rewards and log-probabilities is imposed to ensure that the regularized return, value functions, and preference-based objectives appearing in the bilevel formulation are well-defined and uniformly bounded. The reward boundedness assumption is standard in discounted RL \citep{ganesh2025order,liu2020improved,chen2023finite} and prevents the cumulative reward from diverging. The log-probability boundedness assumption prevents degeneracy due to actions receiving zero or arbitrarily small probability, which would make log-policy terms infinite or unbounded. Together, these assumptions imply that the per-step regularized reward is bounded, and hence the discounted return is finite for every policy in the considered class.

\begin{assumption}\label{as:approx} Let $\theta$ be the policy parameter and $x$ be the reward function parameter. Let the estimated Q-value function be parameterized by $y$. Because of the limitation of the parameterization used, the following approximation error exists between the estimated and actual Q-value function
\begin{equation}
\min_{y} \mathbb{E}_{s,a}\Big(\hat{Q}^{\pi}_{y}(s,a,\theta,x) - Q^{\pi}(s,a, \theta, x)\Big)^2 \leq \epsilon_{approx}    
\end{equation}    
\end{assumption}
Assumption \ref{as:approx} is a standard assumption in RL literature \citep{fu2020single,Wang2020Neural,gaur2024closing}. It is not possible to ensure that the function approximator class used to learn the actual Q-value function will be able to learn the Q-value function with zero errors. 

\begin{assumption}\label{as:PL}
Let the Polyak-Lojasiewicz (PL) condition for a differentiable function f be defined as:
\begin{equation}
\frac{1}{2}\|\nabla f(x)\|^2\geq \mu(f(x) - f^{*})  \end{equation}
Here, $\mu>0$ is a constant and $f^{*}=\min_{x}f(x)$. We assume that $-J(\theta,x)$ satisfies PL condition w.r.t. parameter $\theta$. 
\end{assumption}
PL-type conditions are widely used in nonconvex bilevel optimization \citep{kwon2024on,chen2024finding}. In parameterized bilevel RL, \citet{gaur2025sample} also relies on a PL-type condition. More precisely, prior penalty-based methods impose PL on an auxiliary objective that combines the upper- and lower-level functions. In contrast, we remove any PL-type requirement involving the upper-level objective and assume PL only for the original lower-level loss. This condition relates the lower-level gradient norm to its global optimality gap and enables geometric convergence of the inner optimization; it does not require uniqueness of the lower-level solution. 

\begin{assumption}\label{as:selector}
For each $x_0\in\mathcal X$ visited by the outer-level algorithm, there exists
$\theta_0\in \arg\min_{\theta\in\Theta} -J(x_0,\theta)$ such that
$\nabla_\theta J(x_0,\theta_0)=0$ and
$\nabla^2_{\theta\theta} J(x_0,\theta_0)$ is nonsingular.
Consequently, by the implicit function theorem \citep{lorraine2020optimizing}, there exist neighborhoods
$U_{x_0}\subseteq \mathcal X$ of $x_0$ and
$V_{\theta_0}\subseteq \Theta$ of $\theta_0$, and a continuously differentiable
local selection $\theta^\star:U_{x_0}\to V_{\theta_0}$ satisfying
$\theta^\star(x_0)=\theta_0$ and
$\nabla_\theta J(x,\theta^\star(x))=0$ for all $x\in U_{x_0}$. Moreover, we assume that this local selection has uniformly bounded sensitivity
on the visited neighborhood, i.e., there exists $C_\theta>0$ such that
$\|\nabla_x \theta^\star(x)\|\leq C_\theta,
\qquad \forall x\in U_{x_0}$
We define the outer-level objective locally by $\Phi(x):=\phi(x,\theta^\star(x))$ for $x\in U_{x_0}$.
\end{assumption}
This assumption is substantially weaker than requiring the inner-level problem to have a globally unique minimizer \citep{chakraborty2024parl}. We do not assume that the response map $\Theta^{*}(x)=\arg\min_\theta -J(\theta,x)$ is single-valued on the whole domain. Instead, we only require that the minimizer selected by the algorithm is locally isolated and nondegenerate. Hence, the inner-level problem may still admit multiple global minimizers, possibly corresponding to different branches. The assumption merely ensures that the selected branch is stable under small perturbations of the upper-level variable and can therefore be locally represented as a differentiable function of x. This is a standard local regularity condition \citep{lorraine2020optimizing,mackay2018selftuning} needed for implicit-differentiation-based analysis; without such stability, the inner-level response may jump or bifurcate, and the outer objective along a selected solution would not admit a meaningful gradient characterization.
\begin{theorem}\label{thm:conv}
Suppose Assumption \ref{as:tv} and Assumption \ref{as:lip} to \ref{as:selector} hold true. Then, Algorithm \ref{alg:1} obtains the following convergence rate : 
\begin{equation}
\begin{split}
\frac{1}{T}\sum_{t=0}^{T-1}\|\nabla \Phi(x_t)\|^2 \leq& O\Big(\frac{1}{T}\Big) + O(\epsilon_{kl}^2) + O(\epsilon_{fd}^2) + O\Big(\frac{1}{B}\Big) +O(\epsilon_{approx}^2) + O(\exp^{-K})\\ 
&+O\Big(\frac{\gamma^{2H}}{B}\Big)
+O(\gamma^{2H}) + O(\epsilon_{approx})  
+O(\epsilon_{kl})\\ 
\end{split}
\end{equation}
By choosing, $T = \Theta(\epsilon^{-1})$, $B = \Theta(\epsilon^{-1})$, $K=\Theta(\log(\epsilon^{-1}))$, and $H=\Theta(\log(\epsilon^{-1})/\log(\gamma^{-1}))$, we obtain the iteration complexity of $T = O(\epsilon^{-1})$ and the sample complexity of $T.K.B.H = \tilde{O}(\epsilon^{-2})$. Therefore, we have:
\begin{equation}
\begin{split}
\frac{1}{T}\sum_{t=0}^{T-1}\|\nabla \Phi(x_t)\|^2 &\leq O(\epsilon) + O(\epsilon_{approx}) + O(\epsilon_{fd}) + O(\epsilon_{kl})
\end{split}    
\end{equation}
Here, $\epsilon_{approx}$ is the Q-value function approximation error defined in Assumption \ref{as:approx}. Further, $\epsilon_{fd}$ and $\epsilon_{kl}$ are errors introduced because of the limited policy class as defined in Assumption \ref{as:tv}. $B$ is the batch size used for empirical expectation, $H$ is the horizon length used for estimation of infinite horizon quantities, $K$ is the number of gradient updates for inner level objective optimization, $T$ is the number of gradient updates for outer level objective, and $\tau$ is the temperature coefficient defined in Boltzmann policy (Eq. \eqref{eq:boltz}).
\end{theorem}
\begin{remark}
We obtain a sample complexity of $\tilde{O}(\epsilon^{-2})$ for the Algorithm \ref{alg:1}, which is an improvement over the state-of-the-art sample complexity for the bilevel RL algorithms, whereas \citep{gaur2025sample}'s penalty-based bilevel RL algorithm obtains a sample complexity of $\tilde{O}(\epsilon^{-3})$. The algorithm proposed by \citep{gaur2025sample} used an outer level gradient with $O(\sigma)$ approximation error, and the gradient approximation error is controlled by an $O(\sigma^{2})$ term, and sampling error is controlled by an $O(1/(\sigma^{2}B))$. The tradeoff between gradient approximation error and sampling error drives up the sample complexity for their method. In our case, we use an approximate hypergradient with $O(\epsilon_{fd})$ error, which can be removed with a rich enough policy class as explained in Remark \ref{rem:conv}. Hence, our algorithm's sample complexity bound has two dominant error terms: $O(1/B)$ sampling error and $O(1/T)$ outer level optimization error. This allows us to obtain a better sample complexity. 
\end{remark}
\begin{remark}\label{rem:conv}
In Theorem \ref{thm:conv}, the norm of the gradient of the outer level objective function depends on term error terms like $O(\epsilon_{kl}^2)$, $O(\epsilon_{fd}^2)$, and $O(\epsilon_{approx})$. While $O(\epsilon_{approx})$ is a standard term in the literature of RL, the other two terms appear because the policy class cannot be guaranteed to be realizable and contain all Boltzmann policies. If the policy class is realizable, then error terms $O(\epsilon_{kl}^2)$, $O(\epsilon_{fd}^2)$ would be zero.   
\end{remark}

A detailed proof of the Theorem \ref{thm:conv} is provided in the Appendix \ref{thm:app-conv}. The rest of the section gives a proof sketch of Theorem \ref{thm:conv}. Using the $L$-smoothness of the outer level objective function $(\Phi(x)$, we obtain the following error decomposition:
\begin{equation}
\Phi(x_{t+1}) \leq  \Phi(x_{t}) - \Big(\frac{\beta}{2} -\beta^2L\Big)\|\nabla \Phi(x_t)\|^2 + \Big(\frac{\beta}{2}+\beta^2L\Big)\| \nabla \tilde{\phi}(x_t,\theta^{K}_t,B) - \nabla \Phi(x_t)\|^2
\end{equation}
Subsequently, we decompose the hypergradient estimation error $\| \nabla \tilde{\phi}(x_t,\theta^{K}_t,B) - \nabla \Phi(x_t)\|^2$ into gradient estimation error $\|\nabla \tilde{\phi}(x_t,\theta^{K}_t) - \nabla \Phi(x_t)\|$ and sampling error $\| \nabla \tilde{\phi}(x_t,\theta^{K}_t,B) - \nabla \tilde{\phi}(x_t,\theta^{K}_t)\|$ as given in Eq. \eqref{eq:decomp}.  
\begin{equation}\label{eq:decomp} 
\begin{split}
\| \nabla \tilde{\phi}(x_t,\theta^{K}_t,B) - \nabla \Phi(x_t)\|\leq\|\nabla \tilde{\phi}(x_t,\theta^{K}_t) - \nabla \Phi(x_t)\| + \| \nabla \tilde{\phi}(x_t,\theta^{K}_t,B) - \nabla \tilde{\phi}(x_t,\theta^{K}_t)\|\\    
\end{split}
\end{equation}
Here, the gradient estimation error expands into approximate gradient estimation term $\|\nabla \tilde{\phi}(x_t,\theta^{K}_t) - \nabla \tilde{\phi}(x_t,\theta^{*}(x_t))\|$ and hypergradient approximation error $\|\nabla \tilde{\phi}(x_t,\theta^{*}(x_t) - \nabla \Phi(x_t)\|$ as given in Eq. \eqref{eq:decomp2}. 
\begin{equation}\label{eq:decomp2}
\|\nabla \tilde{\phi}(x_t,\theta^{K}_t) - \nabla \Phi(x_t)\| \leq \|\nabla \tilde{\phi}(x_t,\theta^{K}_t) - \nabla \tilde{\phi}(x_t,\theta^{*}(x_t))\| + \|\nabla \tilde{\phi}(x_t,\theta^{*}(x_t) - \nabla \Phi(x_t)\|    
\end{equation}

\begin{lemma}\label{lm:inner}
Let $\theta^{K}_t$ be the value of the policy parameter after $t$ iterations of outer loop of Algorithm \ref{alg:1} and $\theta^{*}(x_t) = \arg\min_{\theta} -J(\theta,x_t)$. Then, the following bound holds for $\|\theta^{K}_t - \theta^*(x_t)\|$:  
\begin{equation}
\|\theta^{K}_t - \theta^{*}(x_t)\|^2 \leq O(\exp^{-K}) + O\Big(\frac{1}{B}\Big) + O\Big(\frac{\gamma^{2H}}{B}\Big) + O(\gamma^{2H}) + O(\epsilon_{approx})       
\end{equation}
Here, $\epsilon_{approx}$ is defined in Assumption \ref{as:approx}. $B$ is the batch size used for empirical expectation, $H$ is the horizon length used for estimation of infinite horizon quantities, $K$ is the number of gradient updates for inner level objective optimization, and $\gamma$ is the discount factor for the MDP.
\end{lemma}
Using Lemma \ref{lm:7}, we show that the approximate gradient estimation error depends on the inner level iterate sub-optimality gap $\|\theta^{K}_t - \theta^*(x_t)\|$. The inner level iterate sub-optimality gap is proven in Lemma \ref{lm:inner}. A detailed proof of the Lemma \ref{lm:inner} is given in the Appendix \ref{lm:app-inner}. Lemma \ref{lm:inner} uses the Quadratic Growth condition (see Definition \ref{def:qg}) to connect iterate suboptimality gap to value suboptimality gap $J(\theta^{K}_t, x_t) - J(\theta^{*}(x_t), x_t)$. Further, the value suboptimality gap expands into gradient estimation error of $O(\exp^{-K})$ because of the PL assumption (Assumption \ref{as:PL}) on the inner level objective and sampling error of $O(1/B)$. 

\begin{definition}\label{def:qg}
(\textbf{Quadratic Growth}) Let $f$ be proper, lower semicontinuous, and bounded below with unique minimizer $x^{*}$. $f$ is said to satisfy the Quadratic Growth condition if $f(x) - f(x^{*}) \geq \frac{\mu}{2}\|x - x^{*}\|^2$ holds true with $\mu>0$.

\end{definition}
One of departures from the analysis of \citep{gaur2025sample} comes from our hypergradient approximation error ($\|\nabla \tilde{\phi}(x_t,\theta^{*}(x_t) - \nabla \Phi(x_t)\|$) of $O(\epsilon_{kl})+O(\epsilon_{fd})$ (Theorem \ref{thm:gradient}). This error disappears in the case of sufficient rich policy class. In the analysis of \cite{gaur2025sample}, a similar term relies on the gradient approximation error of $O(\sigma^2)$, which, combined with their sampling error of $O(1/(\sigma^2B))$, leads to a tradeoff and drives up the sample complexity. Our algorithm doesn't suffer from such a tradeoff and thus obtains better sample complexity.

Finally, the samplping error term ($\| \nabla \tilde{\phi}(x_t,\theta^{K}_t,B) - \nabla \tilde{\phi}(x_t,\theta^{K}_t)\|$) from Eq. \eqref{eq:decomp} is present because of using the estimators $\widehat{U}^{\pi}_{off,z}(s,\theta^{K}_t,x_t)$ and $\widehat{W}^{\pi}_{off,z}(s, a, \theta^{K}_t,x_t)$ instead of using gradient shifted value functions $U^{\pi}(s,\theta^{K}_t,x_t)$ and $W^{\pi}(s, a, \theta^{K}_t,x_t)$. The rationale for using the estimators is explained in Remark \ref{rem:uoff}. Here, the error introduced by using estimator $\widehat{U}^{\pi}_{off,z}((s,\theta^{K}_t,x_t))$ instead of $U^{\pi}_{off,z}((s,\theta^{K}_t,x_t))$ is handled by Q-value convergence result from \citep{gaur2024closing} as parameter $z$ is updated using Bellman error. Further, we prove using Lemma \ref{lm:woff} that the error introduced because of using $U^{\pi}_{off}(s,\theta^{K}_t,x_t)$ instead of $U^{\pi}(s,\theta^{K}_t,x_t)$ is a function of $\|\theta^{K}_t - \theta^{*}(x_t)\|$ and is handled using Lemma \ref{lm:inner}. In a similar fashion, the error introduced by using $\widehat{W}^{\pi}_{off,z}(s, a, \theta^{K}_t,x_t)$ instead of $W^{\pi}(s, a, \theta^{K}_t,x_t)$ is handled. All the errors discussed above allow us to find a bound on the sampling error $\| \nabla \tilde{\phi}(x_t,\theta^{K}_t,B) - \nabla \tilde{\phi}(x_t,\theta^{K}_t)\|$. This completes the analysis.        
\section{Conclusion}\label{sc:6}
In this work, we provide a scalable bilevel RL algorithm using the optimality of the Boltzmann policy for the entropy regularized discounted reward objective function. We provide conditions under which the optimality of the Boltzmann policy could be used in the parameterized policy setting. Our proposed AHO algorithm uses an approximate hypergradient of the outer-level objective function. This allowed us to improve the previous best sample complexity of $\tilde{O}(\epsilon^{-3})$ \citep{gaur2025sample} to obtain a state-of-the-art sample complexity of $\tilde{O}(\epsilon^{-2})$. Further, by using a hypergradient-based approach instead of a penalty-based formulation, we were able to remove the assumption of the PL condition on the outer-level objective function as used by the prior state-of-the-art bilevel RL algorithm.

\bibliographystyle{apalike}
\bibliography{reference}
\newpage


\appendix

\section{Related Work}\label{sc:2}

\subsection{Bilevel optimization}
Bilevel optimization has been extensively studied for hyperparameter optimization, meta-learning, and hierarchical decision-making. Early work, such as \citep{ghadimi2018approximation}, established sample-complexity guarantees for hypergradient-based methods under convex inner-level assumptions, while \citep{chen2021closing, chen2022single} developed general stochastic and single-timescale bilevel frameworks. Iteration complexity was further improved using momentum-based methods \citep{yang2021provably, khanduri2021near}, though these approaches rely on Hessian-based updates and are computationally expensive. To address this, several Hessian-free and fully first-order methods have been proposed \citep{li2022fully, sow2022convergence, kwon2023fully, liu2022bome}, along with penalty-based formulations that simplify algorithm design \citep{shen2023penalty, kwon2024on} and value-function reformulations \citep{liu2021value}. These advances lay the foundation for the algorithm development in bilevel RL.

\subsection{Bilevel RL}
Bilevel RL has emerged as a natural framework for hierarchical decision-making in meta-RL, incentive design, and RL-HF. Early work established stochastic approximation and two-timescale bilevel frameworks for RL, with finite-time guarantees and applications to actor–critic methods \citep{hong2023two}. Penalty-based formulations tailored to RL and RL-HF were later proposed in \citep{shen2025principled}, providing convergence guarantees under smoothness and regularity assumptions. Subsequent studies analyzed statistical and modeling challenges in bilevel RL. \citep{gaur2025sample} derived state-of-the-art sample-complexity bounds for bilevel RL, while contextual and hypergradient-based extensions were explored in \citep{NEURIPS2024_e66309ea}. More recent works relax inner-level convexity assumptions \citep{yang2025bilevel, chakraborty2024parl}. \citep{chakraborty2024parl} provides a bilevel RL algorithm with Hessian-based hypergradient with iteration complexity of $O(\epsilon^{-1})$. Later, \citep{yang2025bilevel} provides a scalable Hessian-free hypergradient-based algorithm for finite state and action space and obtains a slightly worse iteration complexity of $O(\epsilon^{-1.5})$. The optimality of the Boltzmann policy for the regularized RL objective function helped \citep{yang2025bilevel} to make the algorithm Hessian-free for finite state and action space with a non-parameterized policy. In our work, we propose a Hessian-free algorithm for continuous state and action spaces with a parameterized policy. The usage of parameterized policy leads to issues, and we discuss the conditions under which we can resolve those issues and utilize the optimality of the Boltzmann policy for parameterized policy settings.

\section{Full AHO Algorithm Implementation}\label{sc:algo}

\subsection{Pseudo code for AHO}
Algorithm \ref{alg:4} provides a full description of the Approximate Hypergradient Optimzation (AHO) algorithm. 

\begin{algorithm}[H]
\caption{Approximate Hypergradient Optimization (AHO) Algorithm}\label{alg:4}
\begin{algorithmic}[1]
\State \textbf{Input:} State space $\mathcal{S}$, Action space $\mathcal{A}$, Reward function $r$, Number of outer level updates $T$, Number of inner level updates $K$, Batch size $B$, Horizon length $H$. Initial policy parameters $\theta_{0}^{0}$. Initial reward parameter $x_0$. 
\For{ $t \in \{0,\ldots,T-1\}$}
\For{ $k \in \{0,\ldots,K-1\}$}
\State Collect $B$ episodes $\{\xi(\theta^{k}_t)\}_{i=0}^{B-1}$ using policy parameter $\theta^{k}_t$
\State Obtain $\hat{Q}^{\pi}_{y}$ from Algorithm \ref{alg:2} with policy parameter $\theta_t^{k}$ 
\State $\theta^{k+1}_{t} = \theta^{k}_t + \alpha_t
\nabla_{\theta} J(\theta^{k}_t,x_t,B) $
\EndFor
\State Obtain $\widehat{W}^{\pi}_{off,w}$ and $\widehat{U}^{\pi}_{off,z}$ from Algorithm \ref{alg:3} with policy parameter $\theta_t^{K}$ and $\{\xi(\theta^{K}_t)\}_{i=0}^{B-1}$
\State Collect $B$ episodes $\{\xi'(\theta^{k}_t)\}_{i=0}^{B-1}$ using policy 
\State $x_{t+1} = x_{t} - \beta_t\nabla \tilde{\phi}(x_t, \theta^{K}_t,B)$ \qquad\qquad //Using $\xi'(\theta^{k}_t)\}_{i=0}^{B-1}$
\EndFor
\State \textbf{Output:} $x_{T}$
\end{algorithmic}    
\end{algorithm}

Algorithm \ref{alg:4} uses Algorithm \ref{alg:2} to obtain an estimator of the Q-value function. The Q-value estimator obtained is then used in Eq. \eqref{eq:app-inner-update} to update the policy parameter for the inner level optimization problem (Eq. \eqref{eq:obj}).

\begin{algorithm}[H]
\caption{Q-learning Algorithm}\label{alg:2}
\begin{algorithmic}[1]
\State \textbf{Input:} State space $\mathcal{S}$, Action space $\mathcal{A}$, Reward function $r$, Number of iterations $J$, Batch size $B$, Number of iterations $L$, $B$ episodes $\{\xi\}_{i=0}^{B-1}$ using policy parameter $\theta$. Initial Q-value function $Q_{11}$. 
\State Sample $B$ tuples $(s,a,r,s')$ from $\{\xi\}_{i=0}^{B-1}$ and store in buffer $\mathcal{B}$.
\For{ $j \in \{1,\ldots,J\}$}
\State Initialize Q-value function parameter $y_0$
\For{ $l \in \{1,\ldots,L\}$}
\State Sample $(s,a,r,s')$ from buffer $\mathcal{B}$
\State Sample $a' \sim \pi(\cdot|s',\theta)$
\State $Q_{tar} = r - \tau\log\pi(a|s,\theta) + \gamma Q_{jl}(s',a')$
\State $y'_l = y_{l-1} + \beta(Q_{tar} - Q_{y_{l-1}}(s,a))\nabla_{y}Q_{y_{l-1}}(s,a)$
\State $y_l = \Gamma_{y_0,(1-\gamma)^{-1}}(y'_l)$
\EndFor
\State Let $y=\frac{1}{L}\sum_{l=0}^{L-1}y_l$ and $Q_{jL} = Q_{y}$
\EndFor
\State \textbf{Output:} $Q_{JL}$
\end{algorithmic}    
\end{algorithm}

\begin{equation}\label{eq:app-inner-update}
 \nabla_{\theta}J(\theta^{k}_t, x_t, B) = \frac{1}{B}\sum_{j=0}^{B-1}\nabla_{\theta}\pi(a_j|s_j,\theta^{k}_t)\hat{Q}^{\pi}_{y}(s_j, a_j, \theta^{k}_t, x_t) 
\end{equation}

Further, Algorithm \ref{alg:4} uses Algorithm \ref{alg:3} to obtain estimators for the off-policy gradient shifted value function $ \widehat{U}^{\pi}_{off,z}(s,\theta^K_t,x_t)$ and off-policy gradient shifted Q-value function $\widehat{W}^{\pi}_{off,w}(s,a,\theta^K_t,x_t)$. Later, those estimates are used in Eq. \eqref{eq:app-phi2} to update the reward parameter for the outer level optimization problem (Eq. \eqref{eq:biobj}).

\begin{algorithm}[H]
\caption{Q-learning Algorithm for Gradient Shifted function $W^{\pi}$ and $U^{\pi}$}\label{alg:3}
\begin{algorithmic}[1]
\State \textbf{Input:} State space $\mathcal{S}$, Action space $\mathcal{A}$, Reward function $r$, Number of iterations $J$, Batch size $B$, Number of iterations $L$, $B$ episodes $\{\xi\}_{i=0}^{B-1}$ using policy parameter $\theta$ and reward parameter $x$, Initial gradient shifted value function $U_{11}$. 
\State Sample $B$ tuples $(s,a,s')$ from $\{\xi\}_{i=0}^{B-1}$ and store in buffer $\mathcal{B}$.
\For{ $j \in \{1,\ldots,J\}$}
\State Initialize gradient shifted value function parameter and Q-value function parameter $z_0$ and $w_0$ respectively.
\For{ $l \in \{1,\ldots,L\}$}
\State Sample $(s,a,s')$ from buffer $\mathcal{B}$
\State Sample $a' \sim \pi(\cdot|s',\theta)$
\State $U_{tar} = \nabla_x r(s,a,x) + \gamma U_{jl}(s')$
\State $z'_l = z_{l-1} + \beta(U_{tar} - U_{z_{l-1}}(s))\nabla_{z}U_{z_{l-1}}(s)$
\State $z_l = \Gamma_{z_0,(1-\gamma)^{-1}}(Z'_l)$
\State $w'_l = w_{l-1} + \beta(U_{tar} - W_{w_{l-1}}(s,a))\nabla_{w}W_{w_{l-1}}(s,a)$
\State $w_l = \Gamma_{w_0,(1-\gamma)^{-1}}(w'_l)$
\EndFor
\State Let $z=\frac{1}{L}\sum_{l=1}^{L}z_l$ and $U_{jL} = U_{z}$
\State Let $w=\frac{1}{L}\sum_{l=1}^{L}w_l$
\EndFor
\State \textbf{Output:} $W_{w}$ and $U_{JL}$
\end{algorithmic}    
\end{algorithm}

\begin{equation}\label{eq:app-outer-update}
\nabla \tilde{\phi}(x_t, \theta^{K}_{t},B) = \nabla \tilde{\phi}^{(i)}(x_t, \theta^{K}_{t},B) + \frac{1}{\tau}\nabla \tilde{\phi}^{(ii)}(x_t, \theta^{K}_{t},B)    
\end{equation}
\begin{equation}\label{eq:app-phi1}
\nabla \tilde{\phi}^{(i)}(x_t, \theta^{K}_{t},B) := \frac{1}{B}\sum_{j=0}^{B-1}\nabla_x l(\{d_{ij}\}_{0}^{I-1}, x)    
\end{equation}
\begin{equation}\label{eq:app-phi2}
\begin{split}
\nabla \tilde{\phi}^{(ii)}(x_t, \theta^{K}_{t},B):=\frac{1}{B}\sum_{j=0}^{B-1}\Big(l(\{d_{i,j}\}_{0}^{I-1}, x_t) \sum_{i}\sum_{t}\big(&\widehat{W}^{\pi}_{off,w}(s_t^{i,j},a_t^{i,j},\theta^K_t,x_t) \\ 
&- \widehat{U}^{\pi}_{off,z}(s_t^{i,j},\theta^K_t,x_t)\big)\Big)  \end{split} 
\end{equation}

Please note that Algorithm \ref{alg:2} and \ref{alg:3} are based on the Algorithm 1 from \citep{gaur2024closing}.

\subsection{Empirical Results for AHO}\label{sc:exp}

\subsubsection{Bilevel Formulation of the Preference-Based RL Tasks}
\label{sc:exp-formulation}

We evaluate our method on walker-walk and cheetah-run from the DM Control
suite \citep{tassa2018deepmind}. Although these benchmarks are commonly
studied as single-level RL tasks, we instantiate them as preference-based
bilevel RL problems using the formulation in Section~\ref{sc:3}.

\textbf{Outer-level parameter.} The outer-level parameter $x$ denotes the
parameters of the learned reward function $r(s,a,x)$. We represent
$r(s,a,x)$ using an ensemble of $3$ fully-connected MLPs, where each MLP has
$3$ hidden layers of width $256$ and a $\tanh$ output activation. Each
ensemble member contains approximately $1.3\times 10^{5}$ parameters.

\textbf{Inner-level policy optimization.} For a fixed reward parameter $x$,
the inner-level problem trains an entropy-regularized RL policy by optimizing
$J(\theta,x)$, as defined in Eq.~\eqref{eq:obj}. The resulting policy
parameter is denoted by the single-valued localized selection
$\theta^{*}(x)$, as described in Assumption~\ref{as:selector}. Thus, the
policy used for collecting trajectory segments and evaluating preferences is
$\pi_{\theta^{*}(x)}$.

\textbf{Outer-level preference objective.} The outer level optimizes the
reward parameter $x$ using pairwise preference feedback over trajectory
segments sampled from $P(d,\theta^{*}(x))$. In our experiments, we use the
pairwise setting $I=2$. For a labeled preference pair
$(d_0,l_0,d_1,l_1)$, where $l_1=1,l_0=0$ indicates that $d_1$ is preferred
over $d_0$, and $l_0=1,l_1=0$ indicates that $d_0$ is preferred over $d_1$,
the outer-level loss is the negative Bradley--Terry log-likelihood:
\begin{equation}
\begin{split}
\phi(x,\theta^{*}(x))
=
-\mathbb{E}_{d_0,d_1,l_0,l_1 \sim P(\cdot,\theta^{*}(x))}
\Big[
    l_0 \log P(d_0 \succ d_1 \mid x)
    +
    l_1 \log P(d_1 \succ d_0 \mid x)
\Big],
\end{split}
\end{equation}
where
\begin{equation}
    P(d_a \succ d_b \mid x)
    =
    \frac{\exp(R_x(d_a))}
    {\exp(R_x(d_a))+\exp(R_x(d_b))},
    \qquad
    R_x(d)=\sum_{(s_t,a_t)\in d} r(s_t,a_t,x).
\end{equation}
Thus, the reward network is trained so that trajectory segments with larger
preference labels receive larger learned cumulative reward.

\textbf{Preference feedback.} Preference labels are generated by a scripted
teacher that selects the trajectory segment with larger ground-truth return.
Segment pairs are selected using a disagreement-based active-query strategy \citep{christiano2017deep},
which chooses pairs with high variance across the reward ensemble. Each run
uses a total feedback budget of $500$ preference queries, with new
feedback solicited every $20{,}000$ environment steps after an initial
unsupervised exploration phase.

\subsubsection{Experimental Setup}

We compare our proposed algorithm (AHO) against the bilevel RL algorithm
of \citet{gaur2025sample} on two continuous-control tasks from the DM Control
suite \citep{tassa2018deepmind}: walker-walk and cheetah-run. Walker-walk
trains a planar bipedal robot to walk via joint torques; cheetah-run trains
a planar cheetah to run via joint torques. Both algorithms are trained for
$500{,}000$ environment steps and evaluated every $10{,}000$ steps over
$5$ random seeds. Hyperparameters that differ between the two methods are
held identical wherever applicable, and shared values follow the implementation of \citep{gaur2025sample}; the full list is given in
Table~\ref{tab:hparams}.

\begin{table}[H]
\centering
\small
\caption{Hyperparameters used in both walker-walk and cheetah-run.
Values not listed in the upper block are AHO-specific; the upper block is
shared with the baseline of \citet{gaur2025sample}.}
\label{tab:hparams}
\begin{tabular}{lll}
\toprule
\textbf{Group} & \textbf{Hyperparameter} & \textbf{Value} \\
\midrule
\multirow{4}{*}{Inner level}
 & Actor / critic learning rate $\alpha_t$ & $5 \times 10^{-4}$ \\
 & Hidden layers (actor / critic)          & $2 \times 1024$ \\
 & Inner batch size                        & $1024$ \\
 & Inner updates per env step ($K$ per outer iter implicit) & $1$ \\
\midrule
\multirow{6}{*}{Outer level (reward)}
 & Reward learning rate $\beta_t$          & $5 \times 10^{-4}$ \\
 & Reward optimizer                        & Adam ($\beta_1{=}0.9, \beta_2{=}0.999$) \\
 & Reward batch size $B$                   & $128$ pairs \\
 & Reward update steps per session         & \makecell[l]{up to $200$ \\(early stop at $0.97$ acc)} \\
 & Reward ensemble size                    & $3$ \\
 & Reward MLP                              & $3 \times 256$, $\tanh$ output \\
\midrule
\multirow{4}{*}{Preference data}
 & Segment / horizon length $H$            & $50$ \\
 & Total feedback budget                   & $500$ pairs \\
 & Feedback interval                       & every $20{,}000$ env steps \\
 & Active sampling                         & \makecell[l]{disagreement \\(\texttt{feed\_type=1})} \\
\midrule
\multirow{4}{*}{Schedule}
 & Total env steps                         & $5 \times 10^{5}$ \\
 & Seed steps (random actions)             & $1{,}000$ \\
 & Unsupervised exploration steps          & $9{,}000$ \\
 & Discount factor (RL)                    & $0.99$ \\
\midrule
\multirow{2}{*}{AHO-specific}
 & Temperature $\tau$                      & $100.0$ \\
 & Advantage discount $\gamma_{\text{AHO}}$ & $0.99$ \\
\bottomrule
\end{tabular}
\end{table}

All experiments were run on a single workstation equipped with an
NVIDIA GeForce RTX 2080 Ti GPU (11~GB VRAM) and 64~GB of system memory. All seeds of each algorithm completes in approximately $7$ hours of wall-clock time. Our implementation can be found at the URL: \url{https://anonymous.4open.science/r/AHO-Bilevel-RL-DBE2/}

\begin{figure}[H]
    \centering
    \includegraphics[width=0.47\textwidth]{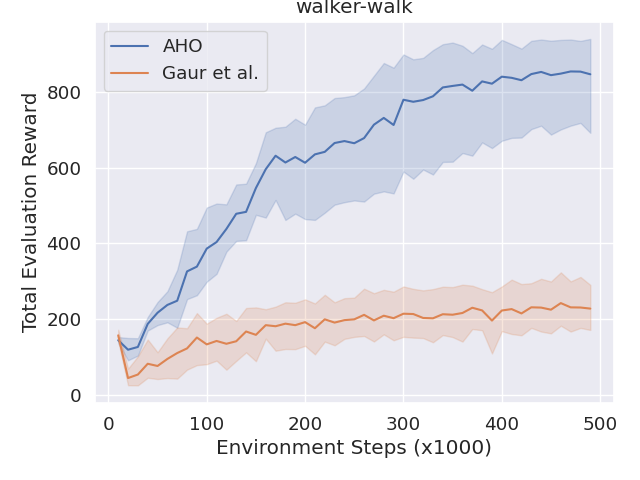}
    \includegraphics[width=0.47\textwidth]{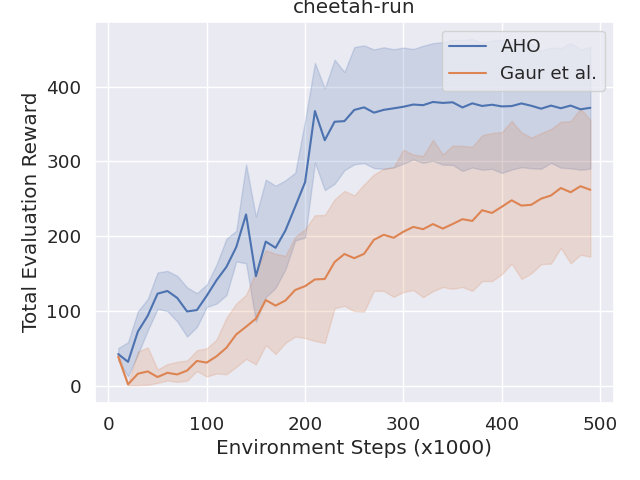}
    \caption{Total episode reward during evaluation for walker-walk (left) and cheetah-run (right). Solid lines show the mean across $5$ random seeds and shaded regions denote 95\% confidence interval around mean.}
    \label{fig:1}
\end{figure}

\subsubsection{Results}

AHO obtains higher total evaluation reward than the bilevel baseline of \citet{gaur2025sample} on both walker-walk and cheetah-run (Figure~\ref{fig:1}). This empirically supports the
sample-complexity improvement predicted by our theoretical analysis.

\textbf{On the choice of baseline.} We do not include PARL
\citep{chakraborty2024parl} as a learning-curve baseline because PARL's hypergradient computation requires forming and inverting the inner-loop Hessian of the policy objective. For the inner level policy used here (two MLPs of
$2 \times 1024$ hidden units, totaling $\sim 2.6 \times 10^{6}$ parameters per network), the Hessian has $\sim 7 \times 10^{12}$ entries and cannot be stored in the $11$~GB of VRAM available on the RTX 2080 Ti, let alone inverted. Even with truncated Neumann or conjugate-gradient approximations, a single PARL outer step requires multiple Hessian-vector products through
the full inner level actor--critic graph, whereas AHO's outer step only requires a backward pass through the reward network ($\sim 1.3 \times 10^{5}$
parameters). We therefore consider the
Hessian-free design of AHO an empirical as well as a theoretical
prerequisite for the network sizes used in modern preference-based RL, and report the comparison against the next-most-relevant baseline \citep{gaur2025sample}, which is also Hessian-free.


\section{Proof of Theorem \ref{thm:gradient}}

\subsection{Lemmas for the Proof of Theorem \ref{thm:gradient}}

\begin{lemma}[Smoothness of the outer-level objective]
\label{lem:outer-objective-smoothness}
Let
\begin{equation}
\Phi(x)
=
\phi\bigl(x,\theta^\star(x)\bigr),
\end{equation}
where $\theta^\star(x)$ is the localized selection map from
Assumption~\ref{as:selector}, defined on a convex neighborhood
$\mathcal U$. Suppose that, for all relevant $x_i$ and $\theta_i$,
\begin{equation}
||\nabla_x\phi(x_1,\theta_1)
-
\nabla_x\phi(x_2,\theta_2)||
\leq
L_{\phi,xx}||x_1-x_2||
+
L_{\phi,x\theta}||\theta_1-\theta_2||,
\end{equation}
and
\begin{equation}
||\nabla_\theta\phi(x_1,\theta_1)
-
\nabla_\theta\phi(x_2,\theta_2)||
\leq
L_{\phi,\theta x}||x_1-x_2||
+
L_{\phi,\theta\theta}||\theta_1-\theta_2||.
\end{equation}
Suppose also that
\begin{equation}
||\nabla_\theta\phi(x,\theta)||
\leq
C_{\phi,\theta}.
\end{equation}

Let $F$ denote the lower-level loss and define
\begin{equation}
H(x)
:=
\nabla_{\theta\theta}^2F
\bigl(x,\theta^\star(x)\bigr),
\qquad
B(x)
:=
\nabla_{\theta x}^2F
\bigl(x,\theta^\star(x)\bigr).
\end{equation}
Suppose that there exist constants
$\mu_H,C_B,L_H,L_B>0$ such that
\begin{equation}
\sigma_{\min}\bigl(H(x)\bigr)
\geq
\mu_H,
\qquad
||B(x)||
\leq
C_B,
\end{equation}
and
\begin{equation}
||H(x_1)-H(x_2)||
\leq
L_H||x_1-x_2||,
\qquad
||B(x_1)-B(x_2)||
\leq
L_B||x_1-x_2||.
\end{equation}
If
\begin{equation}
||\nabla\theta^\star(x)||
\leq
C_\theta,
\end{equation}
then $\Phi$ is $L_\Phi$-smooth on $\mathcal U$, where
\begin{equation}
L_\Phi
=
L_{\phi,xx}
+
L_{\phi,x\theta}C_\theta
+
C_{\phi,\theta}L_{\nabla\theta^\star}
+
C_\theta
\left(
L_{\phi,\theta x}
+
L_{\phi,\theta\theta}C_\theta
\right),
\end{equation}
with
\begin{equation}
L_{\nabla\theta^\star}
=
\frac{C_BL_H}{\mu_H^2}
+
\frac{L_B}{\mu_H}.
\end{equation}
\end{lemma}

\begin{proof}
Since $\mathcal U$ is convex and
$||\nabla\theta^\star(x)||\leq C_\theta$, the mean-value inequality
gives
\begin{equation}
||\theta^\star(x_1)-\theta^\star(x_2)||
\leq
C_\theta||x_1-x_2||.
\end{equation}

The lower-level stationarity condition satisfies
\begin{equation}
\nabla_\theta F\bigl(x,\theta^\star(x)\bigr)
=
0.
\end{equation}
Implicit differentiation therefore gives
\begin{equation}
\nabla\theta^\star(x)
=
-H(x)^{-1}B(x).
\end{equation}
Adding and subtracting $H(x_2)^{-1}B(x_1)$ yields
\begin{equation}
||\nabla\theta^\star(x_1)-\nabla\theta^\star(x_2)||
\leq
||H(x_1)^{-1}-H(x_2)^{-1}||\,||B(x_1)||
+
||H(x_2)^{-1}||\,||B(x_1)-B(x_2)||.
\end{equation}
Using
\begin{equation}
H(x_1)^{-1}-H(x_2)^{-1}
=
H(x_1)^{-1}
\bigl(H(x_2)-H(x_1)\bigr)
H(x_2)^{-1}
\end{equation}
and $||H(x)^{-1}||\leq 1/\mu_H$, we obtain
\begin{equation}
||\nabla\theta^\star(x_1)-\nabla\theta^\star(x_2)||
\leq
L_{\nabla\theta^\star}||x_1-x_2||.
\end{equation}

By the chain rule,
\begin{equation}
\nabla\Phi(x)
=
\nabla_x\phi\bigl(x,\theta^\star(x)\bigr)
+
\nabla\theta^\star(x)^\top
\nabla_\theta\phi\bigl(x,\theta^\star(x)\bigr).
\end{equation}
Adding and subtracting the corresponding cross terms and applying the
preceding bounds gives
\begin{equation}
||\nabla\Phi(x_1)-\nabla\Phi(x_2)||
\leq
L_\Phi||x_1-x_2||.
\end{equation}
Hence, $\Phi$ is $L_\Phi$-smooth on $\mathcal U$.
\end{proof}

\begin{lemma}\label{lem:l-bound}
Let $\{d_i\}_{i=0}^{I-1}$ be a collection of $I$ trajectories, where each trajectory is denoted by
$d=\{s_t,a_t\}_{t=0}^{H-1}$. Let $l_i\in\{0,1\}$ be a one-hot preference label satisfying
$\sum_{i=0}^{I-1}l_i=1$. Consider the multinomial Bradley--Terry preference model
\[
P(d_i \text{ is preferred among } \{d_j\}_{j=0}^{I-1}\mid x)
=
\frac{\exp(R_x(d_i))}
{\sum_{j=0}^{I-1}\exp(R_x(d_j))},
\]
where
$R_x(d)=\sum_{s_t,a_t\in d}r(s_t,a_t,x)$. Define the preference log-likelihood
\[
l(\{d_i\}_{i=0}^{I-1},x)
:=
\sum_{i=0}^{I-1}
l_i
\log
P(d_i \text{ is preferred among } \{d_j\}_{j=0}^{I-1}\mid x).
\]
Then, for every collection $\{d_i\}_{i=0}^{I-1}$ and every $x$,
\[
|l(\{d_i\}_{i=0}^{I-1},x)|
\le
2HR_{\max}+\log I = C_{l}
\]
Here, $R_{max}$ is the bound on the reward given in Assumption \ref{as:log_reward}.
\end{lemma}

\begin{proof}
For notational simplicity, define
\[
p_i(x)
:=
P(d_i \text{ is preferred among } \{d_j\}_{j=0}^{I-1}\mid x)
=
\frac{\exp(R_x(d_i))}
{\sum_{j=0}^{I-1}\exp(R_x(d_j))}.
\]
Then
\[
l(\{d_i\}_{i=0}^{I-1},x)
=
\sum_{i=0}^{I-1}l_i\log p_i(x).
\]
Since $l_i\in\{0,1\}$ and $\sum_{i=0}^{I-1}l_i=1$, there exists an index
$i^\star\in\{0,\dots,I-1\}$ such that $l_{i^\star}=1$ and $l_i=0$ for all
$i\neq i^\star$. Hence,
\[
l(\{d_i\}_{i=0}^{I-1},x)
=
\log p_{i^\star}(x).
\]
Using the definition of $p_{i^\star}(x)$,
\[
\log p_{i^\star}(x)
=
R_x(d_{i^\star})
-
\log\left(
\sum_{j=0}^{I-1}\exp(R_x(d_j))
\right).
\]
Using Assumption \ref{as:log_reward} for the boundedness of the reward function, we have $|R_x(d)|\le HR_{\max}$ for all trajectories $d$, we have
\[
-HR_{\max}\le R_x(d_j)\le HR_{\max}
\]
for all $j$. Therefore,
\[
\sum_{j=0}^{I-1}\exp(R_x(d_j))
\le
I\exp(HR_{\max}),
\]
which implies
\[
\log\left(
\sum_{j=0}^{I-1}\exp(R_x(d_j))
\right)
\le
\log I+HR_{\max}.
\]
Also,
\[
\sum_{j=0}^{I-1}\exp(R_x(d_j))
\ge
\exp(R_x(d_{i^\star})),
\]
and therefore
\[
\log\left(
\sum_{j=0}^{I-1}\exp(R_x(d_j))
\right)
\ge
R_x(d_{i^\star}).
\]
Thus,
\[
\log p_{i^\star}(x)
=
R_x(d_{i^\star})
-
\log\left(
\sum_{j=0}^{I-1}\exp(R_x(d_j))
\right)
\le 0.
\]
Hence,
\[
|l(\{d_i\}_{i=0}^{I-1},x)|
=
|\log p_{i^\star}(x)|
=
-\log p_{i^\star}(x).
\]
Using the previous upper bound on the log-sum-exp term,
\[
\begin{aligned}
|l(\{d_i\}_{i=0}^{I-1},x)|
&=
\log\left(
\sum_{j=0}^{I-1}\exp(R_x(d_j))
\right)
-
R_x(d_{i^\star}) \\
&\le
(\log I+HR_{\max}) + HR_{\max} \\
&=
2HR_{\max}+\log I.
\end{aligned}
\]
The proof is complete.
\end{proof}

\begin{lemma}\label{lm:app-psibound}
Let $\Psi(x,\theta^{*}(x))$ be the difference between the exact hypergradient  ($\nabla \phi(x,\theta^{*}(x))$) and approximate hypergradient ($\nabla \tilde{\phi}(x,\theta^{*}(x))$) of the outer level objective function ($\phi(x,\theta)$) as given in Eq. \eqref{eq:psi}. The difference is bounded as $\|\Psi(x,\theta^{*}(x))\| \leq C_lIH\epsilon_{fd}$. Here, $\|l(\{d_i\}_{0}^{I-1},x)\| \leq C_l$ (Lemma \ref{lem:l-bound}), $I$ is the number of trajectory, $H$ is the length of each trajectory, $\epsilon_{fd}$ is the approximation error from Assumption \ref{as:tv}.  
\begin{equation}\label{eq:psi}
\begin{split}
&\Psi_1(x,\theta^{*}(x)):=\\
& \mathbb{E}_{d_i \sim P(d_i, \theta^*(x))} \Big[l(\{d_i\}_{0}^{I-1}, x) \sum_{i}\sum_{t}\big(\nabla \log\pi(a^{i}_t|s^{i}_t,\theta^{*}(x)) - \nabla \log\pi^{B}(a^{i}_t|s^{i}_t,\theta^{*}(x),x)\big)\Big]
\end{split}
\end{equation}
\begin{equation}
\begin{split}
&\Psi_2(x,\theta^{*}(x)):=\\
&\tau^{-1}\mathbb{E}_{d_i \sim P(d_i, \theta^*(x))} \Big[l(\{d_i\}_{0}^{I-1}, x) \sum_{i}\sum_{t}(\nabla Q^{\pi}(s_t^{i},a_t^{i},\theta^{*}(x),x) - W^{\pi}(s_t^{i},a_t^{i},\theta^{*}(x),x))\Big]\\
&+ \tau^{-1}\mathbb{E}_{d_i \sim P(d_i, \theta^*(x))} \Big[l(\{d_i\}_{0}^{I-1}, x) \sum_{i}\sum_{t}(U^{\pi}(s_t^{i},\theta^{*}(x),x) - \nabla V^{\pi,B}(s_t^{i},\theta^{*}(x),x))\Big]\\
\end{split}    
\end{equation}
\end{lemma}
\begin{proof}
First we will first bound on $\|\Psi_1(x,\theta^{*}(x))\|$:

\begin{equation}\label{eq:app-psibound1}
\begin{split}
&\|\Psi_1(x,\theta^{*}(x))\|\\
\leq &\mathbb{E}_{d_i \sim P(d_i, \theta^*(x))} \Big[\|l(\{d_i\}_{0}^{I-1}, x) \sum_{i}\sum_{t}\big(\nabla \log\pi(a^{i}_t|s^{i}_t,\theta^{*}(x)) - \nabla \log\pi^{B}(a^{i}_t|s^{i}_t,\theta^{*}(x),x)\big)\|\Big]\\     
\leq& C_l\sum_{i}\sum_{t}\mathbb{E}_{d_i \sim P(d_i, \theta^*(x))} \Big[\|\nabla \log\pi(a^{i}_t|s^{i}_t,\theta^{*}(x)) - \nabla \log\pi^{B}(a^{i}_t|s^{i}_t,\theta^{*}(x),x)\|\Big]\\
\leq& C_l\sum_{i}\sum_{t}\mathbb{E}_{s_t \sim P(d_i, \theta^*(x))} \Big[\int \pi(a^{i}_t|s^{i}_t,\theta^{*}(x))\|\nabla \log\pi(a^{i}_t|s^{i}_t,\theta^{*}(x))\\
&- \nabla \log\pi^{B}(a^{i}_t|s^{i}_t,\theta^{*}(x),x)\|\;d a^{i}_t\Big] \\
\leq & C_lIH\epsilon_{fd} \quad (\text{Using Assumption \ref{as:tv}})
\end{split}    
\end{equation}

We will now find bound on $\|\Psi_2(x,\theta^{*}(x))\|$:

\begin{equation}\label{eq:app-psibound2}
\begin{split}
&\|\Psi_2(x,\theta^{*}(x))\|\\
\leq&\tau^{-1}\mathbb{E}_{d_i \sim P(d_i, \theta^*(x))} \Big[l(\{d_i\}_{0}^{I-1}, x) \sum_{i}\sum_{t}\|\nabla Q^{\pi}(s_t^{i},a_t^{i},\theta^{*}(x),x) - W^{\pi}(s_t^{i},a_t^{i},\theta^{*}(x),x)\|\Big]\\
&+ \tau^{-1}\mathbb{E}_{d_i \sim P(d_i, \theta^*(x))} \Big[l(\{d_i\}_{0}^{I-1}, x) \sum_{i}\sum_{t}\|U^{\pi}(s_t^{i},\theta^{*}(x),x) - \nabla V^{\pi,B}(s_t^{i},\theta^{*}(x),x)\|\Big]\\
\leq& \frac{1+\gamma}{1-\gamma}C_lHI(L_{\pi,log}C_\theta \epsilon_{kl} + \epsilon_{fd})\quad (\text{Using Lemma \ref{lm:app-nablaqv}}) 
\end{split}    
\end{equation}

Therefore, using Eq. \eqref{eq:app-psibound1} and Eq. \eqref{eq:app-psibound2}:
\begin{equation}
\begin{split}
\|\Psi(x,\theta^{*}(x))\|
\leq& \|\Psi_1(x,\theta^{*}(x))\| + \|\Psi_2(x,\theta^{*}(x))\|\\
\leq& \frac{1+\gamma}{1-\gamma}C_lHI(L_{\pi,log}C_\theta \epsilon_{kl} + \epsilon_{fd}) + C_lIH\epsilon_{fd} \\
\|\Psi(x,\theta^{*}(x))\|=& O(\epsilon_{kl}) + O(\epsilon_{fd})
\end{split}
\end{equation}
\end{proof}

\begin{lemma}\label{lm:valueboltz}
Let $V^{\pi}(s, \theta, x)$ be the value function and $V^{\pi,B}(s, \theta, x)$ be the Boltzmann value function (Eq. \eqref{eq:boltz}) for the policy parameter $\theta$ and reward parameter $x$. The following relation is satisfied by these functions:
\begin{equation}
V^{\pi}(s, \theta, x)
= V^{\pi,B}(s,\theta, x) - \tau D_{KL} (\pi(\cdot|s,\theta)||\pi^{B}(\cdot|s,\theta, x))  
\end{equation}
Here, $\pi^{B}(\cdot|s,\theta, x)$ is the Boltzmann policy.
\end{lemma}
\begin{proof}
Using definition of $V^{\pi}(s,\theta,x)$ and $\pi^{B}(a|s,\theta, x)$:
\begin{equation}
\begin{split}
V^{\pi}(s, \theta, x) =& \mathbb{E}_{a \sim \pi(\cdot|s, \theta)}[-\tau\log\pi(a|s,\theta) + Q^{\pi}(s,a, \theta, x)] \\    
=& \mathbb{E}_{a \sim \pi(\cdot|s, \theta)}[-\tau\log\pi(a|s,\theta) + V^{\pi,B}(s,\theta, x) + \tau\log\pi^{B}(a|s,\theta, x)] \\    
=& V^{\pi,B}(s,\theta, x) - \tau D_{KL} (\pi(\cdot|s,\theta)||\pi^{B}(\cdot|s,\theta, x)) \\    
\end{split}
\end{equation}
    
\end{proof}

\begin{lemma}\label{lm:app-nablaqv}
Let $U^{\pi}(s,\theta, x)$ be the gradient shifted value function and $W^{\pi}(s, a, \theta, x)$ be the gradient shifted Q-value function. Let $\theta^{*}(x)$ be the optimal policy parameter for the entropy regularized RL objective function (Eq. \eqref{eq:obj}). \\
\begin{enumerate}
\item The following recurrence relation is satisfied by the functions $U^{\pi}$ and $W^{\pi}$:
\begin{equation}
\begin{split}
&U^{\pi}(s,\theta^*(x),x) = \int \pi^{B}(a|s,\theta^{*}(x),x)\big(\nabla  r(s,a,x) + \gamma\mathbb{E}_{s' \sim P(\cdot|s,a)}[U^{\pi}(s',\theta^*(x),x)]\big) \;da\\
&W^{\pi}(s, a, \theta^*(x),x) 
 = \nabla  r(s,a, x) + \gamma\mathbb{E}_{s' \sim P(\cdot|s,a)}[U^{\pi}(s',\theta^*(x),x)]\\
\end{split}    
\end{equation}
\item $U^{\pi}(s,\theta^*(x), x)$ and $W^{\pi}(s, a,\theta^*(x), x)$ act as surrogate for $\nabla V^{\pi,B}(s,\theta^*(x), x)$ and $\nabla Q^{\pi}(s, a,\theta^*(x), x)$ with the following error bound: 
\begin{equation}
\begin{split}
&\|U^{\pi}(s,\theta^*(x),x) - \nabla  V^{\pi,B}(s,\theta^*(x),x)\| \leq  O(\tau (\epsilon_{kl} + \epsilon_{fd})  \\
&\|W^{\pi}(s,a,\theta^*(x),x) - \nabla  Q^{\pi}(s, a, \theta^*(x),x)\|
\leq O(\tau(\epsilon_{kl} + \epsilon_{fd})) 
\end{split}    
\end{equation}
\end{enumerate}
Here, $\epsilon_{kl}$ and $\epsilon_{fd}$ are errors due to unrealizable policy class as defined in Assumption $\ref{as:tv}$ and $\tau$ is the temperature parameter defined for the Boltzmann policy.
\end{lemma}
\begin{proof}
We have, 
\begin{equation}\label{eq:app-nablaqv1}
 \exp{\Big(\frac{V^{\pi,B}(s,\theta^*(x),x)}{\tau}\Big)} = \int \exp{\Big(\frac{Q^{\pi}(s,a,\theta^*(x),x)}{\tau}\Big)} da   
\end{equation}

Taking the derivative of the equation:

\begin{equation}\label{eq:app-nablaqv2}
\begin{split}    
&\nabla \exp{\Big(\frac{V^{\pi,B}(s,\theta^*(x),x)}{\tau}\Big)} = \int \nabla \exp{\Big(\frac{Q^{\pi}(s,a,\theta^*(x),x)}{\tau}\Big)} \;da \\
\implies&\exp{\Big(\frac{V^{\pi,B}(s,\theta^*(x),x)}{\tau}\Big)}\frac{\nabla  V^{\pi,B}(s,\theta^*(x),x)}{\tau}\\ 
&\qquad\qquad\qquad\quad=
\int \exp{\Big(\frac{Q^{\pi}(s,a,\theta^*(x),x)}{\tau}\Big)}\frac{\nabla Q^{\pi}(s,a,\theta^*(x),x)}{\tau} \;da \\
\implies&\frac{\nabla  V^{\pi,B}(s,\theta^*(x),x)}{\tau} \\ 
&\qquad\qquad\qquad= \int \exp{\Big(\frac{Q^{\pi}(s,a,\theta^*(x),x)}{\tau} - \frac{V^{\pi,B}(s,\theta^*(x),x)}{\tau}\Big)}\frac{\nabla Q^{\pi}(s,a,\theta^*(x),x)}{\tau} \;da \\
\implies&\nabla  V^{\pi,B}(s,\theta^*(x),x) = \int \pi^{B}(a|s,\theta^{*}(x),x)\nabla Q^{\pi}(s,a,\theta^*(x),x) \;da \\
\end{split}
\end{equation}

Using definition of $Q^{\pi}(s,a,\theta^*(x),x)$:
\begin{equation}\label{eq:app-nablaqv3}
\begin{split}
&\nabla  V^{\pi,B}(s,\theta^*(x),x) = \int \pi^{B}(a|s,\theta^{*}(x),x)\big(\nabla  r(s,a,x) + \gamma\mathbb{E}_{s' \sim P(\cdot|s,a)}[\nabla V^{\pi}(s',\theta^*(x),x)]\big) \;da \\   
\end{split}
\end{equation}

Using definition of $V^{\pi}(s,\theta^*(x),x)$ and $\pi^{B}(a|s,\theta^*(x), x)$:
\begin{equation}\label{eq:app-nablaqv4}
\begin{split}
&V^{\pi}(s, \theta^*(x), x) \\
=& \mathbb{E}_{a \sim \pi(\cdot|s, \theta^*(x))}[-\tau\log\pi(a|s,\theta^*(x)) + Q^{\pi}(s,a, \theta^*(x), x)] \\    
=& \mathbb{E}_{a \sim \pi(\cdot|s, \theta^*(x))}[-\tau\log\pi(a|s,\theta^*(x)) + V^{\pi,B}(s,\theta^*(x), x) + \tau\log\pi^{B}(a|s,\theta^*(x), x)] \\    
=& V^{\pi,B}(s,\theta^*(x), x) - \tau D_{KL} (\pi(\cdot|s,\theta^*(x))||\pi^{B}(\cdot|s,\theta^*(x), x)) \\    
\implies& \nabla V^{\pi}(s, \theta^*(x), x)
= \nabla V^{\pi,B}(s,\theta^*(x), x) - \tau \nabla D_{KL} (\pi(\cdot|s,\theta^*(x))||\pi^{B}(\cdot|s,\theta^*(x), x)) \\   
\end{split}
\end{equation}

Let $\Delta(s,\theta^{*}(x)) = \tau \nabla D_{KL} (\pi(\cdot|s,\theta^*(x))||\pi^{B}(\cdot|s,\theta^*(x), x))$. Using Eq. \eqref{eq:app-nablaqv4} in Eq. \eqref{eq:app-nablaqv3} we get:

\begin{equation}\label{eq:app-nablaqv5}
\begin{split}
\nabla  V^{\pi,B}(s,\theta^*(x),x) = \int \pi^{B}(a|s,\theta^{*}(x),x)\big(&\nabla  r(s,a,x) - \gamma\mathbb{E}_{s' \sim P(\cdot|s,a)}[\Delta(s',\theta^{*}(x)]\\
&+\gamma\mathbb{E}_{s' \sim P(\cdot|s,a)}[\nabla V^{\pi,B}(s',\theta^*(x),x)]\big) \;da \\   
\end{split}
\end{equation}

\begin{equation}\label{eq:app-nablaqv6}
\begin{split}
&\nabla  V^{\pi,B}(s,\theta^*(x),x) \\
=&\mathbb{E}^{\pi}\big[\sum_{t=0}^{\infty}\gamma^t(\nabla r(s_t,a_t,x)-\gamma\mathbb{E}_{s' \sim P(\cdot|s,a)}[\Delta(s',\theta^{*}(x)])\big | s_0 = s, a_t \sim \pi^{B}(\cdot|s_t,\theta^{*}(x))] 
\end{split}    
\end{equation}

Let us define a new recursion based on Eq. \eqref{eq:app-nablaqv5} and introduce a new quantity $U^{\pi}(s,\theta^{*}(x),x)$:

\begin{equation}\label{eq:app-nablaqv7}
\begin{split}
U^{\pi}(s,\theta^*(x),x) = \int \pi^{B}(a|s,\theta^{*}(x),x)\big(&\nabla  r(s,a,x) 
+\gamma\mathbb{E}_{s' \sim P(\cdot|s,a)}[U^{\pi}(s',\theta^*(x),x)]\big) \;da \\   
\end{split}
\end{equation}

\begin{equation}\label{eq:app-nablaqv8}
\begin{split}
U^{\pi}(s,\theta^*(x),x) 
=\mathbb{E}^{\pi}\big[\sum_{t=0}^{\infty}\gamma^t\nabla  r(s_t,a_t,x) \big | s_0 = s, a_t \sim \pi^{B}(\cdot|s_t,\theta^{*}(x))] 
\end{split}    
\end{equation}

Using Eq. \eqref{eq:app-nablaqv6} and Eq. \eqref{eq:app-nablaqv8}, let us find an upper bound on $\|U^{\pi}(s,\theta^*(x),x) - \nabla  V^{\pi}(s,\theta^*(x),x)\|$:

\begin{equation}\label{eq:app-nablaqv9}
\begin{split}
&\|U^{\pi}(s,\theta^*(x),x) - \nabla  V^{\pi,B}(s,\theta^*(x),x)\| \\
\leq&  \mathbb{E}^{\pi}\Big[\sum_{t=0}^{\infty}\gamma^t(\mathbb{E}_{s' \sim P(\cdot|s,a)}[\|\Delta(s',\theta^{*}(x))\|]\big | s_0 = s, a_t \sim \pi^{B}(\cdot|s_t,\theta^{*}(x))\Big]
\end{split}    
\end{equation}

Let us find a bound on $\|\Delta(s',\theta^{*}(x))\|$:

\begin{equation}\label{eq:app-nablaqv10}
\begin{split}
&\|\Delta(s',\theta^{*}(x))\|\\
\stackrel{(1)}{\leq}& \tau \|\nabla D_{KL} (\pi(\cdot|s,\theta^*(x))||\pi^{B}(\cdot|s,\theta^*(x), x))\| \\
\stackrel{(2)}{\leq}& \tau \int \|\nabla\pi(a|s,\theta^*(x))\||\log\pi(a|s,\theta^*(x)) - \log\pi^{B}(a|s,\theta^*(x), x)|\;da \\
&+ \tau\int \pi(a|s,\theta^*(x))\|\nabla\log\pi(a|s,\theta^*(x)) - \nabla\log\pi^{B}(a|s,\theta^*(x), x)\|\;da\\
\stackrel{(3)}{\leq}& \tau \int \pi(a|s,\theta^*(x)) \|\nabla\log\pi(a|s,\theta^*(x))\||\log\pi(a|s,\theta^*(x)) - \log\pi^{B}(a|s,\theta^*(x), x)|\;da + \tau\epsilon_{fd}\\
\stackrel{(4)}{\leq}& \tau L_{\pi,log}C_\theta \int \pi(a|s,\theta^*(x)) |\log\pi(a|s,\theta^*(x)) - \log\pi^{B}(a|s,\theta^*(x), x)|\;da + \tau\epsilon_{fd}\\
\stackrel{(5)}{\leq}& \tau L_{\pi,log}C_\theta \epsilon_{kl} + \tau\epsilon_{fd}\\
\end{split}
\end{equation}

Here, in (3) we obtain $\epsilon_{fd}$ in the expression by using Assumption \ref{as:tv}.2. In (4), we use Assumption \ref{as:lip}.4 and Assumption \ref{as:selector} to bound $\|\nabla\log\pi(a|s,\theta^*(x))\|$ as $ \|\nabla\log\pi(a|s,\theta^*(x))\| \leq \|\nabla_{\theta}\log\pi(a|s,\theta^*(x))\|\|\nabla\theta^{*}(x)\| \leq  L_{\pi,log}C_\theta$. Further, we obtain (5) by using Assumption \ref{as:tv}.1 for $\epsilon^{kl}$.

Using Eq. \eqref{eq:app-nablaqv9} and Eq. \eqref{eq:app-nablaqv10} we obtain:
\begin{equation}\label{eq:app-nablaqv11}
\begin{split}
\|U^{\pi}(s,\theta^*(x),x) - \nabla  V^{\pi,B}(s,\theta^*(x),x)\| \leq  \frac{\tau L_{\pi,log}C_\theta \epsilon_{kl} + \tau\epsilon_{fd}}{1-\gamma}   
\end{split}    
\end{equation}

Let us obtain an expression for $\nabla  Q^{\pi}(s, a, \theta^*(x),x)$:

\begin{equation}\label{eq:app-nablaqv12}
\begin{split}
\nabla  Q^{\pi}(s, a, \theta^*(x),x) 
& = \nabla  r(s,a, x) + \gamma\mathbb{E}_{s' \sim P(\cdot|s,a)}[\nabla V^{\pi,B}(s',\theta^*(x),x)-\Delta(s',\theta^*(x))]\\
\end{split}    
\end{equation}

Let us obtain a new recursion using the above equation by substituting $\nabla V^{\pi,B}(s',\theta^*(x),x)$ with $U(s,\theta^{*}(x),x)$ and introducing a new quantity $W^{\pi}(s,a,\theta^{*}(x),x)$ as:
\begin{equation}\label{eq:app-nablaqv13}
\begin{split}
W^{\pi}(s,a,\theta^{*}(x),x) 
& = \nabla  r(s,a, x) + \gamma\mathbb{E}_{s' \sim P(\cdot|s,a)}[U^{\pi}(s',\theta^*(x),x)]\\
\end{split}    
\end{equation}

Using Eq. \eqref{eq:app-nablaqv13} and Eq. \eqref{eq:app-nablaqv11} we obtain bound for $\|\nabla  Q^{\pi}(s, a, \theta^*(x),x)-W^{\pi}(s,a,\theta^*(x),x)\|$ as:

\begin{equation}
\begin{split}
&\|\nabla  Q^{\pi}(s, a, \theta^*(x),x)-W^{\pi}(s,a,\theta^*(x),x)\|\\
\leq& \gamma\mathbb{E}_{s' \sim P(\cdot|s,a)}[\|U^{\pi}(s',\theta^*(x),x) - \nabla  V^{\pi,B}(s,\theta^*(x),x)\|+\|\Delta(s',\theta^*(x))\|] \\
\leq& \frac{\tau\gamma(2-\gamma)(L_{\pi,log}C_\theta \epsilon_{kl} + \epsilon_{fd})}{1-\gamma} \quad(\text{Using Eq. \eqref{eq:app-nablaqv10}}) 
\end{split}    
\end{equation}

Let us define $U^{\pi}_{off}$ and $W^{\pi}_{off}$ as the following:
\begin{equation}
\begin{split}
 U^{\pi}_{off}(s,\theta,x) = \int \pi(a|s,\theta)\big(\nabla  r(s,a,x) + \gamma\mathbb{E}_{s' \sim P(\cdot|s,a)}[U^{\pi}_{off}(s',\theta, x)]\big) \;da \\   
\end{split}
\end{equation}
\begin{equation}
\begin{split}
W^{\pi}_{off}(s, a, \theta,x) 
& = \nabla  r(s,a, x) + \gamma\mathbb{E}_{s' \sim P(\cdot|s,a)}[U^{\pi}_{off}(s',\theta, x)]\\
\end{split}
\end{equation}

\end{proof} 

\begin{lemma}\label{lm:app-boltz}
Let $J(\theta,x)$ (Eq. \eqref{eq:obj}) be the objective function for RL defined for policy class $\Theta$ where $\theta \in \Theta$ is the policy parameter, and $x$ is the reward parameter. If policy class $\Theta$ is realizable and there exist $\theta'(x) \in \Theta$ such that $\pi(a|s,\theta'(x)) = \pi^{B}(a|s, \theta'(x),x)$ then $\theta'(x)$ is the optimal policy parameter for $J(\theta,x)$.
\end{lemma}
\begin{proof}
Let us fix $x$ and write
\begin{equation}
\label{eq:app-boltz-pi-prime-def}
\pi'(a|s) := \pi(a|s,\theta'(x)).
\end{equation}
Let $Q^{\pi'}$ and $V^{\pi'}$ denote the entropy-regularized state-action
value function and value function corresponding to $\pi'$. Since
$\theta'(x)$ satisfies
$\pi(a|s,\theta'(x))=\pi^{B}(a|s,\theta'(x),x)$, we have
\begin{equation}
\label{eq:app-boltz-fixed-policy}
\pi'(a|s)
=
\frac{
\exp(Q^{\pi'}(s,a,x)/\tau)
}{
\int_{\mathcal A}
\exp(Q^{\pi'}(s,\bar a,x)/\tau)
d\bar a
}.
\end{equation}
Define the partition function
\begin{equation}
\label{eq:app-boltz-partition}
Z^{\pi'}(s)
:=
\int_{\mathcal A}
\exp(Q^{\pi'}(s,\bar a,x)/\tau)
d\bar a .
\end{equation}
Then Eq. \eqref{eq:app-boltz-fixed-policy} implies
\begin{equation}
\label{eq:app-boltz-log-pi}
\log \pi'(a|s)
=
\frac{1}{\tau}Q^{\pi'}(s,a,x)
-
\log Z^{\pi'}(s).
\end{equation}

By the definition of the entropy-regularized value function,
\begin{equation}
\label{eq:app-boltz-v-def}
V^{\pi'}(s,x)
=
\int_{\mathcal A}
\pi'(a|s)
\left[
Q^{\pi'}(s,a,x)
-
\tau \log \pi'(a|s)
\right]
da .
\end{equation}
Substituting Eq. \eqref{eq:app-boltz-log-pi} into
Eq. \eqref{eq:app-boltz-v-def}, we obtain
\begin{equation}
\label{eq:app-boltz-v-log-partition}
\begin{aligned}
V^{\pi'}(s,x)
&=
\int_{\mathcal A}
\pi'(a|s)
\left[
Q^{\pi'}(s,a,x)
-
\tau
\left(
\frac{1}{\tau}Q^{\pi'}(s,a,x)
-
\log Z^{\pi'}(s)
\right)
\right]
da                                                \\
&=
\int_{\mathcal A}
\pi'(a|s)
\left[
\tau \log Z^{\pi'}(s)
\right]
da                                                \\
&=
\tau \log Z^{\pi'}(s).
\end{aligned}
\end{equation}
Using the definition of $Z^{\pi'}(s)$ from Eq.
\eqref{eq:app-boltz-partition}, this gives
\begin{equation}
\label{eq:app-boltz-v-softmax-q}
V^{\pi'}(s,x)
=
\tau
\log
\int_{\mathcal A}
\exp(Q^{\pi'}(s,a,x)/\tau)
da .
\end{equation}

Now, by the soft Bellman equation for the fixed policy $\pi'$,
\begin{equation}
\label{eq:app-boltz-q-policy-bellman}
Q^{\pi'}(s,a,x)
=
r(s,a,x)
+
\gamma
\mathbb{E}_{s'\sim P(\cdot|s,a)}
\left[
V^{\pi'}(s',x)
\right].
\end{equation}
Substituting Eq. \eqref{eq:app-boltz-q-policy-bellman} into
Eq. \eqref{eq:app-boltz-v-softmax-q}, we obtain
\begin{equation}
\label{eq:app-boltz-soft-optimal-fixed-point}
V^{\pi'}(s,x)
=
\tau
\log
\int_{\mathcal A}
\exp
\left(
\frac{
r(s,a,x)
+
\gamma
\mathbb{E}_{s'\sim P(\cdot|s,a)}
[
V^{\pi'}(s',x)
]
}{\tau}
\right)
da .
\end{equation}
The right-hand side is exactly the soft Bellman optimality operator
applied to $V^{\pi'}$. Therefore,
\begin{equation}
\label{eq:app-boltz-v-fixed-point}
V^{\pi'} = \mathcal{T}V^{\pi'}.
\end{equation}

By Lemma \ref{lem:soft-bellman-contraction}, the soft Bellman optimality
operator $\mathcal{T}$ is a $\gamma$-contraction and has a unique fixed
point, denoted by $V^*$. Since $V^{\pi'}$ is a fixed point of
$\mathcal{T}$, uniqueness implies
\begin{equation}
\label{eq:app-boltz-v-equals-optimal}
V^{\pi'}(s,x)=V^*(s,x),
\qquad \forall s\in\mathcal S .
\end{equation}
Thus $\pi'$ is an optimal entropy-regularized policy.

Finally, for any $\theta\in\Theta$, let
$\pi_\theta(a|s)=\pi(a|s,\theta)$. Since $V^*$ is the optimal soft value
function,
\begin{equation}
\label{eq:app-boltz-value-domination}
V^{\pi_\theta}(s,x)
\leq
V^*(s,x)
=
V^{\pi'}(s,x),
\qquad \forall s\in\mathcal S .
\end{equation}
Taking expectation with respect to the initial state distribution used in
the definition of $J$, we obtain
\begin{equation}
\label{eq:app-boltz-objective-domination}
J(\theta,x)
\leq
J(\theta'(x),x),
\qquad \forall \theta\in\Theta .
\end{equation}
Therefore,
\begin{equation}
\label{eq:app-boltz-optimal-theta}
\theta'(x)\in \arg\max_{\theta\in\Theta} J(\theta,x).
\end{equation}
Hence, $\theta'(x)$ is an optimal policy parameter for $J(\theta,x)$.
\end{proof}
\begin{lemma}\label{lem:soft-bellman-contraction}
Fix a reward parameter $x$ and let $0<\gamma<1$ and $\tau>0$.
Let $\mathcal{B}(\mathcal{S})$ denote the space of bounded real-valued
functions on $\mathcal{S}$ equipped with the sup norm
\begin{equation}
\label{eq:sup-norm}
\|V\|_\infty := \sup_{s\in\mathcal{S}} |V(s)|.
\end{equation}
Define the soft Bellman optimality operator $\mathcal{T}$ by
\begin{equation}
\label{eq:soft-bellman-operator}
(\mathcal{T}V)(s)
=
\tau
\log
\int_{\mathcal{A}}
\exp
\left(
\frac{
r(s,a,x)
+
\gamma
\mathbb{E}_{s'\sim P(\cdot|s,a)}[V(s')]
}{\tau}
\right)
da.
\end{equation}
Assume that $\mathcal{T}V$ is well-defined and bounded for every
$V\in\mathcal{B}(\mathcal{S})$. Then $\mathcal{T}$ is a
$\gamma$-contraction in the sup norm; that is, for all
$V,W\in\mathcal{B}(\mathcal{S})$,
\begin{equation}
\label{eq:soft-bellman-contraction}
\|\mathcal{T}V-\mathcal{T}W\|_\infty
\leq
\gamma \|V-W\|_\infty.
\end{equation}
Consequently, $\mathcal{T}$ has a unique fixed point
$V^*\in\mathcal{B}(\mathcal{S})$.
\end{lemma}

\begin{proof}
Let $V,W\in\mathcal{B}(\mathcal{S})$ and define
\begin{equation}
\label{eq:delta-def}
\delta := \|V-W\|_\infty.
\end{equation}
Then for every $s'\in\mathcal{S}$,
\begin{equation}
\label{eq:pointwise-vw-bound}
W(s')-\delta \leq V(s') \leq W(s')+\delta.
\end{equation}
Taking expectation with respect to $s'\sim P(\cdot|s,a)$ gives, for every
$(s,a)\in\mathcal{S}\times\mathcal{A}$,
\begin{equation}
\label{eq:expected-vw-bound}
\mathbb{E}_{s'\sim P(\cdot|s,a)}[W(s')]
-
\delta
\leq
\mathbb{E}_{s'\sim P(\cdot|s,a)}[V(s')]
\leq
\mathbb{E}_{s'\sim P(\cdot|s,a)}[W(s')]
+
\delta.
\end{equation}
Multiplying by $\gamma$ and adding $r(s,a,x)$, we obtain
\begin{equation}
\label{eq:bellman-argument-bound}
r(s,a,x)
+
\gamma
\mathbb{E}_{s'\sim P(\cdot|s,a)}[V(s')]
\leq
r(s,a,x)
+
\gamma
\mathbb{E}_{s'\sim P(\cdot|s,a)}[W(s')]
+
\gamma \delta.
\end{equation}
Therefore,
\begin{equation}
\label{eq:exponential-bound}
\begin{aligned}
&\exp
\left(
\frac{
r(s,a,x)
+
\gamma
\mathbb{E}_{s'\sim P(\cdot|s,a)}[V(s')]
}{\tau}
\right) \\
&\qquad\leq
\exp\left(\frac{\gamma\delta}{\tau}\right)
\exp
\left(
\frac{
r(s,a,x)
+
\gamma
\mathbb{E}_{s'\sim P(\cdot|s,a)}[W(s')]
}{\tau}
\right).
\end{aligned}
\end{equation}
Integrating both sides over $\mathcal{A}$ gives
\begin{equation}
\label{eq:integral-bound}
\begin{split}
&\int_{\mathcal{A}}
\exp
\left(
\frac{
r(s,a,x)
+
\gamma
\mathbb{E}_{s'\sim P(\cdot|s,a)}[V(s')]
}{\tau}
\right)
da \\
\leq&
\exp\left(\frac{\gamma\delta}{\tau}\right)
\int_{\mathcal{A}}
\exp
\left(
\frac{
r(s,a,x)
+
\gamma
\mathbb{E}_{s'\sim P(\cdot|s,a)}[W(s')]
}{\tau}
\right)
da.
\end{split}
\end{equation}
Applying $\tau\log(\cdot)$ to both sides yields
\begin{equation}
\label{eq:tv-upper-tw}
(\mathcal{T}V)(s)
\leq
(\mathcal{T}W)(s)
+
\gamma\delta.
\end{equation}
By the same argument with $V$ and $W$ interchanged,
\begin{equation}
\label{eq:tw-upper-tv}
(\mathcal{T}W)(s)
\leq
(\mathcal{T}V)(s)
+
\gamma\delta.
\end{equation}
Hence, for every $s\in\mathcal{S}$,
\begin{equation}
\label{eq:pointwise-contraction}
|(\mathcal{T}V)(s)-(\mathcal{T}W)(s)|
\leq
\gamma\delta.
\end{equation}
Taking the supremum over $s$ gives
\begin{equation}
\label{eq:sup-contraction-proof}
\|\mathcal{T}V-\mathcal{T}W\|_\infty
\leq
\gamma\delta
=
\gamma\|V-W\|_\infty.
\end{equation}
Thus $\mathcal{T}$ is a $\gamma$-contraction.

Since $\mathcal{B}(\mathcal{S})$ equipped with the sup norm is a complete
metric space and $0<\gamma<1$, Banach's fixed-point theorem implies that
$\mathcal{T}$ has a unique fixed point $V^*\in\mathcal{B}(\mathcal{S})$.
\end{proof}

\subsection{Proof of Theorem \ref{thm:gradient}}
\begin{theorem}\label{thm:app-gradient}
The hypergradient ($\nabla \phi(x,\theta^{*}(x))$) for the outer objective function $\phi(x,\theta^{*}(x))$ is expressed as
\begin{equation}
\begin{split}
\nabla \phi(x,\theta^{*}(x))=&\mathbb{E}_{d_i \sim P(d_i, \theta^*(x))} [\nabla_x l(\{d_i\}_{0}^{I-1}, x)] \\
&+ \mathbb{E}_{d_i \sim P(d_i, \theta^*(x))} \Big[l(\{d_i\}_{0}^{I-1}, x) \sum_{i}\sum_{h}\nabla \log\pi(a_h^{i}|s_h^{i},\theta^{*}(x))\Big]
\end{split}    
\end{equation}
Let the approximate hypergradient be defined in Eq. \ref{eq:approx-grad}, then the approximation error is given by $\Psi(x,\theta^{*}(x))) = \nabla \phi(x,\theta^{*}(x)) - \nabla \tilde{\phi}(x,\theta^{*}(x)) $ and $\|\Psi(x,\theta^{*}(x)))\| = O(\epsilon_{fd}) + O(\epsilon_{kl})$. Further, $I$ is the number of trajectories, $H$ is the length of each trajectories, $\epsilon_{fd}$ and $\epsilon_{kl}$ are errors due to unrealizable policy class as defined in Assumption \ref{as:tv}, $\tau$ is the temperature coefficient from Eq. \eqref{eq:boltz} and $|l(\{d_i\}_{0}^{I-1}, x)|\leq C_l$ (Lemma \ref{lem:l-bound}). Further, $W^{\pi}(s,a,\theta^{*}(x),x)$ and $U^{\pi}(s,\theta^{*}(x),x)$ are the gradient-shifted Q-value function and the gradient-shifted value function defined in Eq. \eqref{eq:uw}.   
\end{theorem}

\begin{proof}
 We have:
\begin{equation}\label{eq:app-gradient1}
\begin{split}
&\nabla \phi(x,\theta^{*}(x)) \\
=&  \nabla  \mathbb{E}_{d_i \sim P(d_i,\theta^*(x))} [l(\{d_i\}_{0}^{I-1}, x)]\\
=& \mathbb{E}_{d_i \sim P(d_i,\theta^*(x))} [\nabla l(\{d_i\}_{0}^{I-1}, x)] + \int \nabla \Pi_{i=0}^{I-1}P(d_i,\theta^*(x))[l(\{d_i\}_{0}^{I-1}, x)]\Pi_{i=0}^{I-1}d d_i\\
=& \mathbb{E}_{d_i \sim P(d_i,\theta^*(x))} [\nabla l(\{d_i\}_{0}^{I-1}, x)] \\
&+ \int \Pi_{i=0}^{I-1}P(d_i,\theta^*(x)) \nabla \log\Big(\Pi_{i=0}^{I-1}P(d_i,\theta^*(x))\Big)[l(\{d_i\}_{0}^{I-1}, x)]\Pi_{i=0}^{I-1}d d_i
\end{split}
\end{equation}

We know that for a trajectory $d_i$, $\log P(d_i, \theta)$ is given as:
\begin{equation}\label{eq:app-gradient2}
\log P(d_i, \theta) = \log\rho_{0}(s_0^{i}) + \sum_{h=0}^{H-2}(\log\pi(a_{h}^{i}|s_{h}^{i},\theta) + \log P(s_{h+1}^{i}|s_{h}^{i}, a_{h}^{i})) + \log\pi(a_{H-1}^{i}|s_{H-1}^{i},\theta)    
\end{equation}

Using Eq. \eqref{eq:app-gradient1} and Eq. \eqref{eq:app-gradient2}, we obtain the following hypergradient:
\begin{equation}
\begin{split}
\nabla  \mathbb{E}_{d_i \sim P(d_i,\theta^*(x))} [l(\{d_i\}_{0}^{I-1}, x)] =\;&  \mathbb{E}_{d_i \sim P(d_i, \theta^*(x))} [\nabla_x l(\{d_i\}_{0}^{I-1}, x)] \\
&+ \mathbb{E}_{d_i \sim P(d_i, \theta^*(x))} \Big[l(\{d_i\}_{0}^{I-1}, x) \sum_{i}\sum_{h}\nabla \log\pi(a_h^{i}|s_h^{i},\theta^{*}(x))\Big]      
\end{split}
\end{equation}

By adding and subtracting $\nabla\log\pi^{B}(a^{i}_t|s^{i}_t,\theta^{*}(x),x)$ we get:

\begin{equation}
\begin{split}
&\nabla  \mathbb{E}_{d_i \sim P(d_i, \theta^*(x))} [l(\{d_i\}_{0}^{I-1}, x)] \\
=\;&  \mathbb{E}_{d_i \sim P(d_i, \theta^*(x))} [\nabla_x l(\{d_i\}_{0}^{I-1}, x)] \\
&+ \mathbb{E}_{d_i \sim P(d_i, \theta^*(x))} \Big[l(\{d_i\}_{0}^{I-1}, x) \sum_{i}\sum_{t}\nabla \log\pi^{B}(a^{i}_t|s^{i}_t,\theta^{*}(x))\Big]\\        
&+ \mathbb{E}_{d_i \sim P(d_i, \theta^*(x))} \Big[l(\{d_i\}_{0}^{I-1}, x) \sum_{i}\sum_{t}\big(\nabla \log\pi(a^{i}_t|s^{i}_t,\theta^{*}(x)) - \nabla \log\pi^{B}(a^{i}_t|s^{i}_t,\theta^{*}(x),x)\big)\Big]        
\end{split}
\end{equation}

Let us define:
\begin{equation}
\begin{split}
&\Psi_1(x,\theta^{*}(x)):= \\
&\mathbb{E}_{d_i \sim P(d_i, \theta^*(x))} \Big[l(\{d_i\}_{0}^{I-1}, x) \sum_{i}\sum_{t}\big(\nabla \log\pi(a^{i}_t|s^{i}_t,\theta^{*}(x)) - \nabla \log\pi^{B}(a^{i}_t|s^{i}_t,\theta^{*}(x),x)\big)\Big]
\end{split}
\end{equation}

Using the definition of $\pi^{B}(a^{i}_t|s^{i}_t,\theta^{*}(x),x)$ we get:
\begin{equation}
\begin{split}
&\nabla  \mathbb{E}_{d_i \sim P(d_i,\theta^*(x))} [l(\{d_i\}_{0}^{I-1}, x)]\\ 
=\;&  \mathbb{E}_{d_i \sim P(d_i, \theta^*(x))} [\nabla_x l(\{d_i\}_{0}^{I-1}, x)] \\
&+ \tau^{-1}\mathbb{E}_{d_i \sim P(d_i, \theta^*(x))} \Big[l(\{d_i\}_{0}^{I-1}, x) \sum_{i}\sum_{t}(\nabla Q^{\pi}(s_t^{i},a_t^{i},\theta^{*}(x),x) - \nabla V^{\pi,B}(s_t^{i},\theta^{*}(x),x))\Big]\\
&+ \Psi_1(x,\theta^{*}(x))  \\      
=\;&  \mathbb{E}_{d_i \sim P(d_i, \theta^*(x))} [\nabla_x l(\{d_i\}_{0}^{I-1}, x)] \\
&+ \tau^{-1}\mathbb{E}_{d_i \sim P(d_i, \theta^*(x))} \Big[l(\{d_i\}_{0}^{I-1}, x) \sum_{i}\sum_{t}(W^{\pi}(s_t^{i},a_t^{i},\theta^{*}(x),x) - U^{\pi}(s_t^{i},\theta^{*}(x),x))\Big]\\
&+ \tau^{-1}\mathbb{E}_{d_i \sim P(d_i, \theta^*(x))} \Big[l(\{d_i\}_{0}^{I-1}, x) \sum_{i}\sum_{t}(\nabla Q^{\pi}(s_t^{i},a_t^{i},\theta^{*}(x),x) - W^{\pi}(s_t^{i},a_t^{i},\theta^{*}(x),x))\Big]\\
&+ \tau^{-1}\mathbb{E}_{d_i \sim P(d_i, \theta^*(x))} \Big[l(\{d_i\}_{0}^{I-1}, x) \sum_{i}\sum_{t}(U^{\pi}(s_t^{i},\theta^{*}(x),x) - \nabla V^{\pi,B}(s_t^{i},\theta^{*}(x),x))\Big]\\
&+ \Psi_1(x,\theta^{*}(x))        
\end{split}
\end{equation}

Let us define $\Psi_2(x,\theta^{*}(x))$ as:

\begin{equation}
\begin{split}
&\Psi_2(x,\theta^{*}(x))\\
=&\tau^{-1}\mathbb{E}_{d_i \sim P(d_i, \theta^*(x))} \Big[l(\{d_i\}_{0}^{I-1}, x) \sum_{i}\sum_{t}(\nabla Q^{\pi}(s_t^{i},a_t^{i},\theta^{*}(x),x) - W^{\pi}(s_t^{i},a_t^{i},\theta^{*}(x),x))\Big]\\
&+ \tau^{-1}\mathbb{E}_{d_i \sim P(d_i, \theta^*(x))} \Big[l(\{d_i\}_{0}^{I-1}, x) \sum_{i}\sum_{t}(U^{\pi}(s_t^{i},\theta^{*}(x),x) - \nabla V^{\pi,B}(s_t^{i},\theta^{*}(x),x))\Big]\\
\end{split}    
\end{equation}

Therefore, we have:
\begin{equation}
\begin{split}
&\nabla  \mathbb{E}_{d_i \sim P(d_i,\theta^*(x))} [l(\{d_i\}_{0}^{I-1}, x)]\\ 
=\;&  \mathbb{E}_{d_i \sim P(d_i, \theta^*(x))} [\nabla_x l(\{d_i\}_{0}^{I-1}, x)] \\
&+ \tau^{-1}\mathbb{E}_{d_i \sim P(d_i, \theta^*(x))} \Big[l(\{d_i\}_{0}^{I-1}, x) \sum_{i}\sum_{t}(W^{\pi}(s_t^{i},a_t^{i},\theta^{*}(x),x) - U^{\pi}(s_t^{i},\theta^{*}(x),x))\Big]\\
&+ \Psi_1(x,\theta^{*}(x)) +\Psi_2(x,\theta^{*}(x))\\        
\end{split}
\end{equation}

We define $\nabla \tilde{\phi}(x,\theta^{*}(x))$ as:
\begin{equation}
\begin{split}
&\nabla \tilde{\phi}(x,\theta^{*}(x))\\
=\;&  \mathbb{E}_{d_i \sim P(d_i, \theta^*(x))} [\nabla_x l(\{d_i\}_{0}^{I-1}, x)] \\
&+ \tau^{-1}\mathbb{E}_{d_i \sim P(d_i, \theta^*(x))} \Big[l(\{d_i\}_{0}^{I-1}, x) \sum_{i}\sum_{t}(W^{\pi}(s_t^{i},a_t^{i},\theta^{*}(x),x) - U^{\pi}(s_t^{i},\theta^{*}(x),x))\Big]\\
\end{split}    
\end{equation}

Therefore, we have:
\begin{equation}
\nabla \phi(x,\theta^{*}(x)) - \nabla \tilde{\phi}(x,\theta^{*}(x)) = \Psi_1(x,\theta^{*}(x)) + \Psi_2(x,\theta^{*}(x)) = \Psi(x,\theta^{*}(x))    
\end{equation}
Further, using Lemma \ref{lm:app-psibound} we obtain that $\|\Psi(x,\theta^{*}(x)))\| = O(\epsilon_{kl}) + O(\epsilon_{fd})$.
\end{proof}

\section{Proof of Theorem \ref{thm:conv}}

\subsection{Lemmas for the Proof of Theorem \ref{thm:conv}}

\begin{lemma}\label{lem:grad-l-bound}
Let $\{d_i\}_{i=0}^{I-1}$ be a collection of $I$ trajectories, where each trajectory is denoted by
$d=\{s_t,a_t\}_{t=0}^{H-1}$. Let $l_i\in\{0,1\}$ be a one-hot preference label satisfying
$\sum_{i=0}^{I-1}l_i=1$. Consider the multinomial Bradley--Terry preference model
\begin{equation}
\label{eq:mult-bt}
P(d_i \text{ is preferred among } \{d_j\}_{j=0}^{I-1}\mid x)
=
\frac{\exp(R_x(d_i))}
{\sum_{j=0}^{I-1}\exp(R_x(d_j))},
\end{equation}
where
$R_x(d)=\sum_{s_t,a_t\in d} r(s_t,a_t,x)$. Define the preference log-likelihood
\begin{equation}
\label{eq:mult-bt-loglik}
l(\{d_i\}_{i=0}^{I-1},x)
:=
\sum_{i=0}^{I-1}
l_i
\log
P(d_i \text{ is preferred among } \{d_j\}_{j=0}^{I-1}\mid x).
\end{equation}
Then, for every collection $\{d_i\}_{i=0}^{I-1}$ and every $x$,
\begin{equation}
\label{eq:l-grad-bound}
\|\nabla_x l(\{d_i\}_{i=0}^{I-1},x)\|
\le 2HL_r = C_{lx}
\end{equation}
Here, $L_r$ is Lipchitz coefficient from Assumption \ref{as:lip}.5.
\end{lemma}

\begin{proof}
For notational simplicity, define
\begin{equation}
p_i(x)
:=
P(d_i \text{ is preferred among } \{d_j\}_{j=0}^{I-1}\mid x)
=
\frac{\exp(R_x(d_i))}
{\sum_{j=0}^{I-1}\exp(R_x(d_j))}.
\end{equation}
Then
\begin{equation}
l(\{d_i\}_{i=0}^{I-1},x)
=
\sum_{i=0}^{I-1} l_i \log p_i(x).
\end{equation}
Since $\sum_{i=0}^{I-1}l_i=1$ and $l_i\in\{0,1\}$, there exists an index $i^\star\in\{0,\dots,I-1\}$ such that $l_{i^\star}=1$ and $l_i=0$ for all $i\neq i^\star$. Therefore,
\begin{equation}
l(\{d_i\}_{i=0}^{I-1},x)
=
\log p_{i^\star}(x).
\end{equation}
Using the definition of $p_{i^\star}(x)$, we get
\begin{equation}
\log p_{i^\star}(x)
=
R_x(d_{i^\star})
-
\log
\left(
\sum_{j=0}^{I-1}\exp(R_x(d_j))
\right).
\end{equation}
Differentiating with respect to $x$ gives
\begin{equation}
\nabla_x l(\{d_i\}_{i=0}^{I-1},x)
=
\nabla_x R_x(d_{i^\star})
-
\nabla_x
\log
\left(
\sum_{j=0}^{I-1}\exp(R_x(d_j))
\right).
\end{equation}
By the chain rule,
\begin{equation}
\nabla_x
\log
\left(
\sum_{j=0}^{I-1}\exp(R_x(d_j))
\right)
=
\frac{
\sum_{j=0}^{I-1}
\exp(R_x(d_j))\nabla_x R_x(d_j)
}{
\sum_{j=0}^{I-1}\exp(R_x(d_j))
}.
\end{equation}
Using the definition of $p_j(x)$, this can be written as
\begin{equation}
\nabla_x
\log
\left(
\sum_{j=0}^{I-1}\exp(R_x(d_j))
\right)
=
\sum_{j=0}^{I-1}p_j(x)\nabla_x R_x(d_j).
\end{equation}
Hence,
\begin{equation}
\nabla_x l(\{d_i\}_{i=0}^{I-1},x)
=
\nabla_x R_x(d_{i^\star})
-
\sum_{j=0}^{I-1}p_j(x)\nabla_x R_x(d_j).
\end{equation}
Taking norms and using the triangle inequality,
\begin{align}
\|\nabla_x l(\{d_i\}_{i=0}^{I-1},x)\|
&\le
\|\nabla_x R_x(d_{i^\star})\|
+
\left\|
\sum_{j=0}^{I-1}p_j(x)\nabla_x R_x(d_j)
\right\| \nonumber \\
&\le
\|\nabla_x R_x(d_{i^\star})\|
+
\sum_{j=0}^{I-1}p_j(x)\|\nabla_x R_x(d_j)\|.
\end{align}
By Assumption \ref{as:lip}.5, we have
$\|\nabla_x R_x(d_{i^\star})\|\le HL_r$ and
$\|\nabla_x R_x(d_j)\|\le HL_r$ for all $j$. Therefore,
\begin{align}
\|\nabla_x l(\{d_i\}_{i=0}^{I-1},x)\|
&\le
HL_r
+
\sum_{j=0}^{I-1}p_j(x)HL_r \nonumber \\
&=
HL_r
+
G_R\sum_{j=0}^{I-1}p_j(x).
\end{align}
Since $\{p_j(x)\}_{j=0}^{I-1}$ is a probability distribution over the $I$ trajectories,
$\sum_{j=0}^{I-1}p_j(x)=1$. Thus,
\begin{equation}
\|\nabla_x l(\{d_i\}_{i=0}^{I-1},x)\|
\le
2HL_r = C_{lx}.
\end{equation}
This proves the pointwise gradient bound.

\end{proof}

\begin{lemma}\label{lm:lipzU}
Let $U^{\pi}(s,\theta,x)$ be the gradient shifted value function defined in Lemma \ref{lm:nablaqv}. $U^{\pi}(s,\theta,x)$ satisfies Lipchitz continuity i.e., $\|U^{\pi}(s,\theta_1,x) - U^{\pi}(s,\theta_2,x)\| \leq Vol(\mathcal{A})\frac{L_{B}L_r}{\tau(1-\gamma)^2}\|\theta_1 - \theta_2\|$. Here, $Vol(\mathcal{A})$ is the volume of the action space $\mathcal{A}$, $L_{B}$ is the Lipchitz coefficient of the Boltzmann policy $\pi^{B}(a|s,\theta,x)$ defined in Lemma \ref{lm:lippiB}. $L_r$ is defined in Assumptions \ref{as:lip}, $\tau$ is the temperature coefficient of the Boltzmann policy, and $\gamma$ is the discount factor of the MDP.
\end{lemma}
\begin{proof}
Using Lemma \ref{lm:nablaqv} for the definition of $U^{\pi}(s,\theta,x)$ we have:
\begin{equation}
\begin{split}
U^{\pi}(s,\theta,x) = \int \pi^{B}(a|s,\theta)\big(\nabla  r(s,a,x) + \gamma\mathbb{E}_{s' \sim P(\cdot|s,a)}[U^{\pi}(s',\theta, x)]\big) \;da    
\end{split}    
\end{equation}
\begin{equation}
\begin{split}
&\|U^{\pi}(s,\theta_1,x) - U^{\pi}(s,\theta_2,x)\|\\
\stackrel{(1)}{\leq}& \|\int \pi^{B}(a|s,\theta_1,x)\big(\nabla  r(s,a,x) + \gamma\mathbb{E}_{s' \sim P(\cdot|s,a)}[U^{\pi}(s', \theta_1, x)]\big) \;da \\
&- \int \pi^{B}(a|s,\theta_2,x)\big(\nabla  r(s,a,x) + \gamma\mathbb{E}_{s' \sim P(\cdot|s,a)}[U^{\pi}(s', \theta_2, x)]\big) \;da\|\\
\stackrel{(2)}{\leq}&\|\int (\pi^{B}(a|s,\theta_1,x) - \pi^{B}(a|s,\theta_2,x))\nabla  r(s,a,x) \; da\| \\
&+ \gamma\|\int \pi^{B}(a|s,\theta_1,x)\mathbb{E}_{s' \sim P(\cdot|s,a)}[U^{\pi}(s', \theta_1, x) - U^{\pi}(s', \theta_2, x)] \;da\| \\
&+\gamma\|\int (\pi^{B}(a|s,\theta_1,x) - \pi^{B}(a|s,\theta_2,x))\mathbb{E}_{s' \sim P(\cdot|s,a)}[U^{\pi}(s', \theta_2, x)] \;da\| \\
\stackrel{(3)}{\leq}&\; Vol(\mathcal{A})\frac{L_{B}L_r}{\tau(1-\gamma)}\|\theta_1 - \theta_2\| \\
&+ \gamma\|\int \pi^{B}(a|s,\theta_1,x)\mathbb{E}_{s' \sim P(\cdot|s,a)}[U^{\pi}(s', \theta_1, x) - U^{\pi}(s', \theta_2, x)] \; da\| \\
\stackrel{(4)}{\leq}&\; Vol(\mathcal{A})\frac{L_{B}L_r}{\tau(1-\gamma)}\|\theta_1 - \theta_2\| + \gamma Vol(\mathcal{A})\frac{L_{B}L_r}{\tau(1-\gamma)}\|\theta_1 - \theta_2\| \\
&+ \gamma^2\Big\|\int \pi^{B}(a|s,\theta_1,x)\\
&\qquad\qquad\mathbb{E}_{s' \sim P(\cdot|s,a)}\Big[\int \pi^{B}(a'|s',\theta_1)\mathbb{E}_{s'' \sim P(\cdot|s',a')}[U^{\pi}(s'', \theta_1, x) - U^{\pi}(s'', \theta_2, x)]\;da\Big]\;da\Big\| \\
\stackrel{(5)}{\leq}&\; Vol(\mathcal{A})\frac{L_{B}L_r}{\tau(1-\gamma)^2}\|\theta_1 - \theta_2\| \\
\end{split}    
\end{equation}

We obtain (3) from (2) by using Assumption \ref{as:lip}.5 for bounding $\nabla r(s,a,x)$ and Lemma \ref{lm:lippiB}. We obtain (4) from (3) by expanding the recursion.

\begin{equation}
\|U^{\pi}(s,\theta_1,x) - U^{\pi}(s,\theta_2,x)\| \leq Vol(\mathcal{A})\frac{L_{B}L_r}{\tau(1-\gamma)^2}\|\theta_1 - \theta_2\|    
\end{equation}
\end{proof}

\begin{lemma}\label{lm:lipzW}
Let $W^{\pi}(s,a,\theta,x)$ be the gradient shifted Q-value function defined in Lemma \ref{lm:nablaqv}. $W^{\pi}(s,a,\theta,x)$ satisfies Lipchitz continuity i.e., $\|W^{\pi}(s,a,\theta_1,x) - W^{\pi}(s,a,\theta_2,x)\| \leq \gamma Vol(\mathcal{A})\frac{L_{B}L_r}{\tau(1-\gamma)^2}\|\theta_1 - \theta_2\|$. Here, $Vol(\mathcal{A})$ is the volume of the action space $\mathcal{A}$, $L_{B}$ is the Lipchitz coefficient of the Boltzmann policy $\pi^{B}(a|s,\theta,x)$ defined in Lemma \ref{lm:lippiB}. $L_r$ is defined in Assumptions \ref{as:lip}, $\tau$ is the temperature coefficient of the Boltzmann policy and $\gamma$ is the discount factor of the MDP. 
\end{lemma}
\begin{proof}

Using recursion for $W^{\pi}(s, a,\theta_1,x)$ defined in Eq. \eqref{eq:app-nablaqv13} from Lemma \ref{lm:app-nablaqv}:

\begin{equation}
\begin{split}
&\|W(s, a,\theta_1,x) - W^{\pi}(s, a,\theta_2,x)\| \\
\leq & \|\nabla  r(s,a, x) + \gamma\mathbb{E}_{s' \sim P(\cdot|s,a)}[U^{\pi}(s',\theta_1,x)] - \nabla  r(s,a, x) - \gamma\mathbb{E}_{s' \sim P(\cdot|s,a)}[U^{\pi}(s',\theta_2,x)] \|\\ 
\leq & \gamma\|\mathbb{E}_{s' \sim P(\cdot|s,a)}[U^{\pi}(s',\theta_1,x)] - \mathbb{E}_{s' \sim P(\cdot|s,a)}[U^{\pi}(s',\theta_2,x)] \| \\
\leq & \gamma\mathbb{E}_{s' \sim P(\cdot|s,a)}[\|U^{\pi}(s',\theta_1,x) - U^{\pi}(s',\theta_2,x) \|] 
\end{split}    
\end{equation}

Using Lemma \ref{lm:lipzU} we get:
\begin{equation}
\|W^{\pi}(s, a,\theta_1,x) - W^{\pi}(s, a,\theta_2,x)\| \leq  \gamma Vol(\mathcal{A})\frac{L_{B}L_r}{\tau(1-\gamma)^2}\|\theta_1 - \theta_2\|   
\end{equation}
\end{proof}

\begin{lemma}\label{lm:5}
Let $\mathbb{E}_{d_i \sim P(d_i, \theta)} [\nabla_x l(\{d_i\}_{i=0}^{I-1}, x)]$ be the partial derivative of the objective function defined in Eq. \eqref{eq:RL-HF} w.r.t. reward parameter $x$. $\mathbb{E}_{d_i \sim P(d_i, \theta)} [\nabla_x l(\{d_i\}_{i=0}^{I-1}, x)]$ is Lipchitz continuous w.r.t. to the policy parameter $\theta$ such that $\|\mathbb{E}_{d_i \sim P(d_i, \theta_1)} [\nabla_x l(\{d_i\}_{i=0}^{I-1}, x)] - \mathbb{E}_{d_i \sim P(d_i, \theta_2)} [\nabla_x l(\{d_i\}_{i=0}^{I-1}, x)]\|
 \leq C_{lx}HI Vol(\mathcal{A})L_{\pi}\|\theta_1 - \theta_2\|$. Here, $\|\nabla_x l(\{d_i\}_{i=0}^{I-1}, x)\| \leq C_{lx}$ (Lemma \ref{lem:grad-l-bound}), $Vol(\mathcal{A})$ is the volume of the action space $\mathcal{A}$, $H$ is the length of the trajectory, $I$ is the number of trajectories and $L_{\pi}$ is a Lipchitz coefficient defined in Assumption \ref{as:lip}.3.  
\end{lemma}
\begin{proof}
Using the Integral Probability Metric (IPM), we have:
\begin{equation}
 \sup_{\|f\|_{\infty}\leq C}|\mathbb{E}_{d \sim P(d, \theta_1)}[f(d)] - \mathbb{E}_{d \sim P(d, \theta_2)}[f(d)]| \leq 2CD_{TV}(P(d,\theta_1), P(d,\theta_2))     
\end{equation}

Let us find a bound on $D_{TV}(P(d,\theta_1), P(d,\theta_2))$. Here $F(x) = \frac{1}{2}|x-1|$.
\begin{equation}
D_{TV}(P(d,\theta_1), P(d,\theta_2)) \leq D_{TV}(P(d,\theta_1), P(d,p)) + D_{TV}(P(d,p), P(d,\theta_2))     
\end{equation}

Let us consider first $D_{TV}(P(d,\theta_1), P(d,p))$:
\begin{equation}
\begin{split}
&D_{TV}(P(d,\theta_1), P(d,p)) \\
=& \int P(d, p) F(\frac{P(d,\theta_1)}{P(d,p)}) \;dd \\
=& \int P(d, p)\\
&\times F\Big(\frac{\rho_{0}(s_0)\pi(a_{0}|s_{0},\theta_1)P(s_{1}|s_{0}, a_{0})\Pi_{i=1}^{H-2}\pi(a_{i}|s_{i},\theta_1)P(s_{i+1}|s_{i}, a_{i})\pi(a_{H-1}|s_{H-1},\theta_1)}{\rho_{0}(s_0)\pi(a_{0}|s_{0},\theta_2)P(s_{1}|s_{0}, a_{0})\Pi_{i=1}^{H-2}\pi(a_{i}|s_{i},\theta_1)P(s_{i+1}|s_{i}, a_{i})\pi(a_{H-1}|s_{H-1},\theta_1)}\Big)\;dd \\
=& \int P(d, p) F\Big(\frac{\pi(a_{0}|s_{0},\theta_1)}{\pi(a_{0}|s_{0},\theta_2)}\Big) \;dd\\
=& \int \rho_0(s_0)\int \pi(a_{0}|s_{0},\theta_2) F\Big(\frac{\pi(a_{0}|s_{0},\theta_1)}{\pi(a_{0}|s_{0},\theta_2)}\Big) \;da_0\;ds_0\\ 
=& \int \rho_0(s_0)D_{TV}(\pi(\cdot|s_{0},\theta_1),\pi(\cdot|s_{0},\theta_2))\;ds_0\\
\end{split}    
\end{equation}

Now, let us consider $D_{TV}(P(d,p), P(d,\theta_2))$:

\begin{equation}
\begin{split}
&D_{TV}(P(d,p), P(d,\theta_2)) \\
=& \int P(d, \theta_2) F\Big(\frac{P(d,p)}{P(d,\theta_2)}\Big) \;dd \\
=& \int P(d, \theta_2)\\
&\times F\Big(\frac{\rho_{0}(s_0)\pi(a_{0}|s_{0},\theta_2)P(s_{1}|s_{0}, a_{0})\Pi_{i=1}^{H-2}\pi(a_{i}|s_{i},\theta_1)P(s_{i+1}|s_{i}, a_{i})\pi(a_{H-1}|s_{H-1},\theta_1)}
{\rho_{0}(s_0)\pi(a_{0}|s_{0},\theta_2)P(s_{1}|s_{0}, a_{0})\Pi_{i=1}^{H-2}\pi(a_{i}|s_{i},\theta_2)P(s_{i+1}|s_{i}, a_{i})\pi(a_{H-1}|s_{H-1},\theta_2)}\Big)\;dd \\
=& \int \rho_{0}(s_0) \int \pi(a_0|s_0,\theta_2)\int P(d_{1}^{H-1}, \theta_1|s_0,a_0) \\
&\times F\Big(\frac{P(s_{1}|s_{0}, a_{0})\Pi_{i=1}^{H-2}\pi(a_{i}|s_{i},\theta_1)P(s_{i+1}|s_{i}, a_{i})\pi(a_{H-1}|s_{H-1},\theta_1)}
{P(s_{1}|s_{0}, a_{0})\Pi_{i=1}^{H-2}\pi(a_{i}|s_{i},\theta_2)P(s_{i+1}|s_{i}, a_{i})\pi(a_{H-1}|s_{H-1},\theta_2)}\Big)\;dd\;d a_0\;d s_0 \\
=& \int \rho_{0}(s_0) \int \pi(a_0|s_0,\theta_2) \int P(d_{1}^{H-1}, \theta_1|s_0,a_0) F\Big(\frac{P(d_{1}^{H-1}, \theta_1|s_0,a_0)}
{P(d_{1}^{H-1}, \theta_2|s_0,a_0)}\Big)\;dd\;d a_0\;d s_0 \\
=& \int \rho_{0}(s_0) \int \pi(a_0|s_0,\theta_2) D_{TV}(P(d_{1}^{H-1}, \theta_1|s_0,a_0),
P(d_{1}^{H-1}, \theta_2|s_0,a_0))\;d a_0\;d s_0 \\
\end{split}
\end{equation}

We can again break $D_{TV}(P(d_{1}^{H-1}, \theta_1|s_0,a_0),
P(d_{1}^{H-1}, \theta_2|s_0,a_0))$ into two terms:

\begin{equation}\label{eq:32}
\begin{split}
&D_{TV}(P(d_{1}^{H-1}, \theta_1|s_0,a_0),
P(d_{1}^{H-1}, \theta_2|s_0,a_0))\\ 
\leq&  D_{TV}(P(d_{1}^{H-1}, \theta_1|s_0,a_0),
P(d_{1}^{H-1}, p'|s_0,a_0))\\
&+  D_{TV}(P(d_{1}^{H-1}, p'|s_0,a_0),
P(d_{1}^{H-1}, \theta_2|s_0,a_0))
\end{split}
\end{equation}

Using the decomposition in \eqref{eq:32} we obtain the following:
\begin{equation}
\begin{split}
D_{TV}(P(d,\theta_1), P(d, \theta_2)) 
&\leq \sum_{h=0}^{H-1}\mathbb{E}_{s_h}D_{TV}(\pi(\cdot|s_h, \theta_1), \pi(\cdot|s_h, \theta_2))\\
&\leq H \max_{s}D_{TV}(\pi(\cdot|s, \theta_1), \pi(\cdot|s, \theta_2))\\    
&\leq H D_{TV}(\pi(\cdot|s^{*}, \theta_1), \pi(\cdot|s^{*}, \theta_2))\\   
&\leq H \int \pi(a|s^{*}, \theta_2)T\Big(\frac{\pi(a|s^{*}, \theta_1)}{\pi(a|s^{*}, \theta_2)}\Big)\;da   \\ 
&\leq \frac{H}{2} \int |\pi(a|s^{*}, \theta_1) - \pi(a|s^{*}, \theta_2)|\;da\\    
&\leq \frac{H}{2} Vol(\mathcal{A})L_{\pi}\|\theta_1 - \theta_2\|   
\end{split}    
\end{equation}

For $I$ trajectories $\{\{d_i\}_{j=0}^{H-1}\}_{i=0}^{I-1}$ we have 
\begin{equation}
D_{TV}(P(\{\{d_i\}_{j=0}^{H-1}\}_{i=0}^{I-1},\theta_1), P(\{\{d_i\}_{j=0}^{H-1}\}_{i=0}^{I-1}, \theta_2)) \leq \frac{HI}{2} Vol(\mathcal{A})L_{\pi}\|\theta_1 - \theta_2\|
\end{equation}

We have,

\begin{equation}
 \sup_{\|f\|_{\infty}\leq C_l}|\mathbb{E}_{d_i \sim P(d_i, \theta_1)}[f(\{d_i\}_{i=0}^{I-1})] - \mathbb{E}_{d_i \sim P(d_i, \theta_2)}[f(\{d_i\}_{i=0}^{I-1})]| \leq C_lHI Vol(\mathcal{A})L_{\pi}\|\theta_1 - \theta_2\|     
\end{equation}

We have $\|\nabla_x l(\{d_i\}_{i=0}^{I-1}, x)\| \leq C_{lx}$, therefore: 

\begin{equation}
|\mathbb{E}_{d_i \sim P(d_i, \theta_1)}[\nabla_x l(\{d_i\}_{i=0}^{I-1}, x)] - \mathbb{E}_{d_i \sim P(d_i, \theta_2)}[\nabla_x l(\{d_i\}_{i=0}^{I-1}, x)]| \leq C_{lx}HI Vol(\mathcal{A})L_{\pi}\|\theta_1 - \theta_2\|     
\end{equation}

\end{proof}

\begin{lemma}\label{lm:7}
Let $\nabla \tilde{\phi}(x,\theta)$ be the approximate hypergradient defined in Theorem \ref{thm:gradient}. $\nabla \tilde{\phi}(x,\theta)$ satisfies Lipchitz continuity w.r.t. to parameter $\theta$, i.e.,  Bound $\|\nabla \tilde{\phi}(x,\theta_1) - \nabla \tilde{\phi}(x,\theta_2)\| \leq L_{\phi}\|\theta_1 - \theta_2\| \leq O(\|\theta_1 - \theta_2\|) + O(\|\theta_1 - \theta_2\|/\tau)+O(\|\theta_1 - \theta_2\|/\tau^2)$. Here, $\tau$ is the temperature coefficient defined for the Boltzmann policy $\pi^{B}$ (Eq. \eqref{eq:boltz}).    
\end{lemma}

\begin{proof}

\begin{equation}
\begin{split}
\nabla  \tilde{\phi}(x, \theta) =\;&  \mathbb{E}_{d_i \sim P(d_i, \theta)} [\nabla_x l(\{d_i\}_{i=0}^{I-1}, x)] \\
&+ \tau^{-1}\mathbb{E}_{d_i \sim P(d_i, \theta)} \Big[l(\{d_i\}_{i=0}^{I-1}, x) \sum_{i}\sum_{h}\big(W^{\pi}(s_h^{i},a_h^{i},\theta,x) - U^{\pi}(s_h^{i},\theta,x)\big)\Big]\\                
\end{split}
\end{equation}

Using Lemma \ref{lm:5}, we have:
    
\begin{equation}\label{eq:38}
\|\mathbb{E}_{d_i \sim P(d_i, \theta_1)}[\nabla_x l(\{d_i\}_{i=0}^{I-1}, x)] - \mathbb{E}_{d_i \sim P(d_i, \theta_2)}[\nabla_x l(\{d_i\}_{i=0}^{I-1}, x)]\| \leq C_lHI Vol(\mathcal{A})L_{\pi}\|\theta_1 - \theta_2\|     
\end{equation}

Let us consider the following term:

\begin{equation}\label{eq:39}
\begin{split}
&\Big\|\mathbb{E}_{d_i \sim P(d_i, \theta_1)} \Big[l(\{d_i\}_{i=0}^{I-1}, x) \sum_{i}\sum_{h}\big(W^{\pi}(s_h^{i},a_h^{i},\theta_1,x) - U^{\pi}(s_h^{i},\theta_1,x)\big)\Big]\\
- & \mathbb{E}_{d_i \sim P(d_i, \theta_1)} \Big[l(\{d_i\}_{i=0}^{I-1}, x) \sum_{i}\sum_{h}\big(W^{\pi}(s_h^{i},a_h^{i},\theta_2,x) - U^{\pi}(s_h^{i},\theta_2,x)\big)\Big]\Big\|\\
\stackrel{(1)}{\leq}&\;\mathbb{E}_{d_i \sim P(d_i, \theta_1)} \Big[l(\{d_i\}_{i=0}^{I-1}, x) \sum_{i}\sum_{h}\big\|W^{\pi}(s_h^{i},a_h^{i},\theta_1,x) - W^{\pi}(s_h^{i},a_h^{i},\theta_2,x)\big\|\Big]\\
& + \mathbb{E}_{d_i \sim P(d_i, \theta_1)} \Big[l(\{d_i\}_{i=0}^{I-1}, x) \sum_{i}\sum_{h}\big\|U^{\pi}(s_h^{i},\theta_2,x) - U^{\pi}(s_h^{i},\theta_1,x)\big\|\Big]\\
\stackrel{(2)}{\leq}&\; (1+\gamma)HIC_lVol(\mathcal{A})\frac{L_{B}L_r}{\tau(1-\gamma)^2}\|\theta_1 - \theta_2\|
\end{split}    
\end{equation}

We obtain (2) from (1) by using Lipchitz continuity of $U^{\pi}$ (Lemma \ref{lm:lipzU}) and $W^{\pi}$ (Lemma \ref{lm:lipzW}) .

Let us now consider the term :

\begin{equation}
\begin{split}
&\Big\|\mathbb{E}_{d_i \sim P(d_i, \theta_1)} \Big[l(\{d_i\}_{i=0}^{I-1}, x) \sum_{i}\sum_{h}\big(W^{\pi}(s_h^{i},a_h^{i},\theta_2,x) - U^{\pi}(s_h^{i},\theta_2,x)\big)\Big]\\
&-\mathbb{E}_{d_i \sim P(d_i, \theta_2)} \Big[l(\{d_i\}_{i=0}^{I-1}, x) \sum_{i}\sum_{h}\big(W^{\pi}(s_h^{i},a_h^{i},\theta_2,x) - U^{\pi}(s_h^{i},\theta_2,x)\big)\Big]\Big\|\\    
\end{split}    
\end{equation}

We have :
\begin{equation}
\|l(\{d_i\}_{i=0}^{I-1}, x) \sum_{i}\sum_{h}\big(W^{\pi}(s_h^{i},a_h^{i},\theta_2,x) - U^{\pi}(s_h^{i},\theta_2,x)\big) \| \leq C_{l}HI\frac{2L_r}{1-\gamma}    
\end{equation}

Using Lemma \ref{lm:5} we know that:
\begin{equation}
 \sup_{\|f\|_{\infty}\leq C}|\mathbb{E}_{d_i \sim P(d_i, \theta_1)}[f(\{d_i\}_{i=0}^{I-1})] - \mathbb{E}_{d_i \sim P(d_i, \theta_2)}[f(\{d_i\}_{i=0}^{I-1})]| \leq CHI Vol(\mathcal{A})L_{\pi}\|\theta_1 - \theta_2\|     
\end{equation}

Therefore we have:

\begin{equation}\label{eq:43}
\begin{split}
&\Big\|\mathbb{E}_{d_i \sim P(d_i, \theta_1)} \Big[l(\{d_i\}_{i=0}^{I-1}, x) \sum_{i}\sum_{h}\big(W^{\pi}(s_h^{i},a_h^{i},\theta_2,x) - U^{\pi}(s_h^{i},\theta_2,x)\big)\Big]\\
- & \mathbb{E}_{d_i \sim P(d_i, \theta_2)} \Big[l(\{d_i\}_{i=0}^{I-1}, x) \sum_{i}\sum_{h}\big(W^{\pi}(s_h^{i},a_h^{i},\theta_2,x) - U^{\pi}(s_h^{i},\theta_2,x)\big)\Big]\Big\|\\
\leq& \;C_{l}\frac{2L_r}{1-\gamma}(HI)^2 Vol(\mathcal{A})L_{\pi}\|\theta_1 - \theta_2\|
\end{split}    
\end{equation}

Let us consider the main bound :

\begin{equation}\label{eq:110}
\begin{split}
&\|\nabla \tilde{\phi}(x,\theta_1) - \nabla \tilde{\phi}(x,\theta_2) \| \\
\leq & \|\mathbb{E}_{d_i \sim P(d_i, \theta_1)}[\nabla_x l(\{d_i\}_{i=0}^{I-1}, x)] - \mathbb{E}_{d_i \sim P(d_i, \theta_2)}[\nabla_x l(\{d_i\}_{i=0}^{I-1}, x)]\| \\
& \frac{1}{\tau} \Big\|\mathbb{E}_{d_i \sim P(d_i, \theta_1)} \Big[l(\{d_i\}_{i=0}^{I-1}, x) \sum_{i}\sum_{h}\big(W^{\pi}(s_h^{i},a_h^{i},\theta_1,x) - U^{\pi}(s_h^{i},\theta_1,x)\big)\Big]\\
& - \mathbb{E}_{d_i \sim P(d_i, \theta_1)} \Big[l(\{d_i\}_{i=0}^{I-1}, x) \sum_{i}\sum_{h}\big(W^{\pi}(s_h^{i},a_h^{i},\theta_2,x) - U^{\pi}(s_h^{i},\theta_2,x)\big)\Big]\Big\| \\
&\frac{1}{\tau}\Big\|\mathbb{E}_{d_i \sim P(d_i, \theta_1)} \Big[l(\{d_i\}_{i=0}^{I-1}, x) \sum_{i}\sum_{h}\big(W^{\pi}(s_h^{i},a_h^{i},\theta_2,x) - U^{\pi}(s_h^{i},\theta_2,x)\big)\Big]\\
&- \mathbb{E}_{d_i \sim P(d_i, \theta_2)} \Big[l(\{d_i\}_{i=0}^{I-1}, x) \sum_{i}\sum_{h}\big(W^{\pi}(s_h^{i},a_h^{i},\theta_2,x) - U^{\pi}(s_h^{i},\theta_2,x)\big)\Big]\Big\|\\    
\end{split}    
\end{equation}

Using Eq. \eqref{eq:38}, \eqref{eq:39} and \eqref{eq:43} in Eq. \eqref{eq:110} we get:
\begin{equation}
\begin{split}
&\|\nabla \tilde{\phi}(x,\theta_1) - \nabla \tilde{\phi}(x,\theta_2) \| \\
\leq & \Big(L_{\pi} +  \Big(\frac{1+\gamma}{(1-\gamma)^2}\Big)\frac{L_BL_r}{\tau^2} + \frac{2L_rHIL_{\pi}}{\tau(1-\gamma)} \Big)C_lHI Vol(\mathcal{A})\|\theta_1 - \theta_2\|  
\end{split}    
\end{equation}

\begin{equation}
\begin{split}
&\|\nabla \tilde{\phi}(x,\theta_1) - \nabla \tilde{\phi}(x,\theta_2) \| 
\leq  O(\|\theta_1 - \theta_2\|) + O\Big(\frac{\|\theta_1 - \theta_2\|}{\tau}\Big) + O\Big(\frac{\|\theta_1 - \theta_2\|}{\tau^2}\Big)  
\end{split}    
\end{equation}
\end{proof}

\begin{lemma}\label{lm:app-inner}
Let $\theta^{K}_t$ be the value of the policy parameter after $t$ iterations of outer loop of Algorithm \ref{alg:1} and $\theta^{*}(x_t) = \arg\min_{\theta} -J(\theta,x_t)$. Then, the following bound holds for $\|\theta^{K}_t - \theta^*(x_t)\|$:    
\begin{equation}
\|\theta^{K}_t - \theta^{*}(x_t)\|^2 \leq O(\exp^{-K}) + O\Big(\frac{1}{B}\Big) + O\Big(\frac{\gamma^{2H}}{B}\Big) + O(\gamma^{2H}) + O(\epsilon_{approx})       
\end{equation}
Here, $\epsilon_{approx}$ is defined in Assumption \ref{as:approx}. $B$ is the batch size used for empirical expectation, $H$ is the horizon length used for estimation of infinite horizon quantities, $K$ is the number of gradient updates for inner level objective optimization, and $\gamma$ is the discount factor for the MDP.
\end{lemma}

\begin{proof}

Using Quadratic Growth condition we have:
\begin{equation}
\begin{split}
\|\theta^{K}_t - \theta^{*}(x_t)\|^2 \leq \frac{2}{\mu}\|J(\theta^{*}(x_t), x_t) - J(\theta^{K}_t, x_t)\|  
\end{split}     
\end{equation} 

Let us now consider $\|J(\theta^{*}(x_t), x_t) - J(\theta^{K}_t, x_t)\|$. Using Lemma 1 from \citep{gaur2025sample}, we obtain:

\begin{equation}
\begin{split}
\|J(\theta^{*}(x_t), x_t) - J(\theta^{K}_t, x_t) \| \leq O(\exp^{-K}) + O(\beta(B,H))   
\end{split}    
\end{equation}

Here, $\mathbb{E}[\|\nabla_{\theta}J(\theta^{k}_t, x_t) - \nabla_{\theta}\hat{J}(\theta^{k}_t, x_t)\|^2] \leq \beta(B,H)$.

We have:
\begin{equation}
\begin{split}
\nabla_{\theta}\hat{J}(\theta^{k}_t, x_t) = \frac{1}{B}\sum_{i=0}^{B-1}\nabla_{\theta}\log(\pi_{\theta}(a_i|s_i))\hat{Q}_{y}^{\pi}(s_i, a_i,\theta^{k}_t,x_t) 
\end{split}    
\end{equation}

Now, 

\begin{equation}
\begin{split}
&\mathbb{E}[\|\nabla_{\theta}J(\theta^{k}_t, x_t) - \nabla_{\theta}\hat{J}(\theta^{k}_t, x_t)\|^2]\\
\leq & \mathbb{E}\Big\|\mathbb{E}[\nabla_{\theta}\log(\pi_{\theta}(a|s))Q^{\pi}(s, a,\theta^{k}_t,x_t)] - \frac{1}{B}\sum_{i=0}^{B-1}\nabla_{\theta}\log(\pi_{\theta}(a_i|s_i))\hat{Q}_{y}^{\pi}(s_i, a_i,\theta^{k}_t,x_t)\Big\|^2 \\
\end{split}
\end{equation}

Using Lemma 3, Eq. 77 and Eq. 81 from \citep{gaur2025sample} we get: 

\begin{equation}
\begin{split}
&\mathbb{E}[\|\nabla_{\theta}J(\theta^{k}_t, x_t) - \nabla_{\theta}\hat{J}(\theta^{k}_t, x_t)\|^2] \leq O\Big(\frac{1}{B}\Big) + O\Big(\frac{\gamma^{2H}}{B}\Big) + O(\gamma^{2H}) + O(\epsilon_{approx}) 
\end{split}
\end{equation}

Therefore,
\begin{equation}
\begin{split}
\|J(\theta^{*}(x_t), x_t) - J(\theta^{K}_t, x_t) \| \leq O(\exp^{-K}) + O\Big(\frac{1}{B}\Big) + O\Big(\frac{\gamma^{2H}}{B}\Big) + O(\gamma^{2H}) + O(\epsilon_{approx})   
\end{split}    
\end{equation}

Hence,
\begin{equation}
\|\theta^{K}_t - \theta^{*}(x_t)\|^2 \leq O(\exp^{-K}) + O\Big(\frac{1}{B}\Big) + O\Big(\frac{\gamma^{2H}}{B}\Big) + O(\gamma^{2H}) + O(\epsilon_{approx})       
\end{equation}
\end{proof}

\begin{lemma}\label{lm:pibbound}
The Boltzmann policy density function $\pi^{B}(a|s,\theta)$ is bounded.    
\end{lemma}
\begin{proof}
 We know that:
\begin{equation}
\pi^{B}(a|s,\theta) = \frac{\exp(Q^{\pi}(s,a,\theta,x)/\tau)}{\exp(V^{\pi,B}(s,\theta,x)/\tau)}  
\end{equation}
Further, using Assumption \ref{as:log_reward} we have:
\begin{equation}\label{eq:pibbound_1}
\begin{split}
Q^{\pi}(s,a,\theta,x) =& \mathbb{E}^{\pi}\Big[\sum_{t=0}^{\infty}\gamma^t(r(s_t,a_t,x) - \tau\log\pi(a_t|s_t))|s_0=s, a_0=a\Big] \\   
|Q^{\pi}(s,a,\theta,x)| \leq& \mathbb{E}^{\pi}\Big[\sum_{t=0}^{\infty}\gamma^t(|r(s_t,a_t,x)| + \tau|\log\pi(a_t|s_t))||s_0=s, a_0=a\Big] \\
\leq & \frac{R_{max}+\tau C_{\log}}{1-\gamma}
\end{split}    
\end{equation}

Also, using Assumption \ref{as:log_reward}, we have:
\begin{equation}\label{eq:pibbound_2}
\begin{split}
V^{\pi}(s,\theta,x) =& \mathbb{E}^{\pi}\Big[\sum_{t=0}^{\infty}\gamma^t(r(s_t,a_t,x) - \tau\log\pi(a_t|s_t))|s_0=s\Big] \\   
|V^{\pi}(s,\theta,x)| \leq& \mathbb{E}^{\pi}\Big[\sum_{t=0}^{\infty}\gamma^t(|r(s_t,a_t,x)| + \tau|\log\pi(a_t|s_t)||s_0=s\Big] \\
\leq & \frac{R_{max}+\tau C_{\log}}{1-\gamma}
\end{split}    
\end{equation}

From Eq. \eqref{eq:app-nablaqv4} of Lemma \ref{lm:app-nablaqv} we know:

\begin{equation}
\begin{split}
V^{\pi,B}(s, \theta, x) =& V^{\pi}(s,\theta, x) + \tau D_{KL} (\pi(\cdot|s,\theta)||\pi^{B}(\cdot|s,\theta, x)) \\
\implies V^{\pi,B}(s, \theta, x) \stackrel{(1)}{\geq}& -\Big(\frac{R_{max}+\tau C_{\log}}{1-\gamma}\Big) 
\end{split}
\end{equation}

Here, (1) is true because the minimum value of $D_{KL} (\pi(\cdot|s,\theta)||\pi^{B}(\cdot|s,\theta, x))$ is 0. 

In order to find the upper bound on $\pi^{B}(a|s,\theta)$. We need the maximum value of the numerator term and the minimum value of the denominator term. Therefore, using Eq. \eqref{eq:pibbound_1} and \eqref{eq:pibbound_2} we have:
\begin{equation}
 |\pi^{B}(a|s,\theta)| \leq C_{B} = \exp\Big(\frac{2R_{max}}{\tau(1-\gamma)}+\frac{2C_{\log}}{1-\gamma}\Big) = O\Big(\exp\Big(\frac{2R_{max}}{\tau(1-\gamma)}\Big)\Big)  
\end{equation}
\end{proof}

\begin{lemma}\label{lm:lippiB}
The Boltzmann policy $\pi^{B}(a|s,\theta,x)$ is Lipchitz continuous w.r.t. parameter $\theta$ i.e., $\|\pi_{B}(a|s,\theta_1,x)- \pi_{B}(a|s,\theta_2,x)\|\leq (L_{B}/\tau)\|\theta_1 - \theta_2\|$.    
\end{lemma}
\begin{proof}
\begin{equation}\label{eq:lippiB1}
\begin{split}
\log(\pi^{B}(a|s,\theta,x)) &= \frac{Q^{\pi}(s,a,\theta,x)}{\tau} - \frac{V^{\pi,B}(s,\theta,x)}{\tau} \\
\nabla_{\theta}\log(\pi^{B}(a|s,\theta,x)) &= \frac{\nabla_{\theta}Q^{\pi}(s,a,\theta,x)}{\tau} - \frac{\nabla_{\theta}V^{\pi,B}(s,\theta,x)}{\tau} \\ 
\end{split}   
\end{equation}

From Assumption \ref{as:lip}.1 we know that $\|\nabla_{\theta}Q^{\pi}(s,a,\theta,x)\| \leq L_{\theta,Q}$.

For finding bound on $\|\nabla_{\theta}V^{\pi}(s,\theta,x)\|$. Let us consider the following:

\begin{equation}
 \exp{\Big(\frac{V^{\pi,B}(s,\theta,x)}{\tau}\Big)} = \int \exp{\Big(\frac{Q^{\pi}(s,a,\theta,x)}{\tau}\Big)} da   
\end{equation}

Taking derivative of the equation w.r.t. $\theta$:

\begin{equation}
\begin{split}    
&\nabla_{\theta} \exp{\Big(\frac{V^{\pi,B}(s,\theta,x)}{\tau}\Big)} = \int \nabla_{\theta} \exp{\Big(\frac{Q^{\pi}(s,a,\theta,x)}{\tau}\Big)} \;da \\
&\exp{\Big(\frac{V^{\pi,B}(s,\theta,x)}{\tau}\Big)}\frac{\nabla_{\theta}  V^{\pi,B}(s,\theta,x)}{\tau} = \int \exp{\Big(\frac{Q^{\pi}(s,a,\theta,x)}{\tau}\Big)}\frac{\nabla_{\theta} Q^{\pi}(s,a,\theta,x)}{\tau} \;da \\
&\frac{\nabla_{\theta}  V^{\pi,B}(s,\theta,x)}{\tau} = \int \exp{\Big(\frac{Q^{\pi}(s,a,\theta,x)}{\tau} - \frac{V^{\pi,B}(s,\theta,x)}{\tau}\Big)}\frac{\nabla_{\theta} Q^{\pi}(s,a,\theta,x)}{\tau} \;da \\
&\nabla_{\theta}  V^{\pi,B}(s,\theta,x) = \int \pi^{B}(a|s,\theta,x)\nabla_{\theta} Q^{\pi}(s,a,\theta,x) \;da \\
&\|\nabla_{\theta}  V^{\pi,B}(s,\theta,x)\| = \int \pi^{B}(a|s,\theta,x)\|\nabla_{\theta} Q^{\pi}(s,a,\theta,x)\|;da \leq L_{\theta,Q} \\
\end{split}
\end{equation}

Using $\|\nabla_{\theta}  V^{\pi,B}(s,\theta,x)\|\leq L_{\theta,Q}$ and $\|\nabla_{\theta}Q^{\pi}(s,a,\theta,x)\| \leq L_{\theta,Q}$ in Eq. \eqref{eq:lippiB1} we get:

\begin{equation}
\begin{split}
&\|\nabla_{\theta}\log(\pi^{B}(a|s,\theta))\| \leq \frac{2L_{\theta,Q}}{\tau}\\
&\|\nabla_{\theta}\pi^{B}(a|s,\theta)\| \leq \pi^{B}(a|s,\theta)\frac{2L_{\theta,Q}}{\tau}\\
\end{split}    
\end{equation}

From Lemma \ref{lm:pibbound} we know that Boltzmann probability density $\pi^{B}$ is bounded. Therefore:
\begin{equation}
\|\nabla_{\theta}\pi^{B}(a|s,\theta)\| \leq \frac{2C_{B}L_{\theta,Q}}{\tau}\\
\end{equation}

So, $\|\nabla_{\theta}\pi^{B}(a|s,\theta)\|$ is bounded. Hence, $\pi^{B}(a|s,\theta)$ is Lipchitz continuous w.r.t. parameter $\theta$ with coefficient $L_{B}/\tau = 2C_{B}L_{\theta,Q}/\tau$. 
\end{proof}

\begin{lemma}\label{lm:woff}
Let $W^{\pi}_{off}(s,a,\theta,x)$ be the off-policy gradient shifted Q-value function and $U^{\pi}_{off}(s,\theta,x)$ be the off-policy gradient shifted value function defined in Eq. \eqref{eq:uoff}. Also, $W^{\pi}(s,a,\theta,x)$ be the gradient shifted Q-value function and $U^{\pi}(s,\theta,x)$ be the gradient shifted value function defined in Lemma \ref{lm:nablaqv}. These functions satisfy the following bounds:   

\begin{equation}
\begin{split}
\|U^{\pi}_{off}(s,\theta,x) - U^{\pi}(s,\theta,x)\| 
\leq \frac{L_r\epsilon_{kl}}{(1-\gamma)^2}
+ \Big(L_{\pi} + \frac{L_B}{\tau}\Big)\frac{Vol(\mathcal{A})L_rL_{\pi}}{(1-\gamma)^2}\|\theta - \theta^{*}(x)\|\\
\|W^{\pi}_{off}(s,a,\theta,x) - W^{\pi}(s,a,\theta,x)\|
\leq \frac{\gamma L_r\epsilon_{kl}}{(1-\gamma)^2}
+ \Big(L_{\pi} + \frac{L_B}{\tau}\Big)\frac{\gamma Vol(\mathcal{A})L_rL_{\pi}}{(1-\gamma)^2}\|\theta - \theta^{*}(x)\|\\
\end{split}
\end{equation}
Here, $Vol(\mathcal{A})$ is the volume of the action space $\mathcal{A}$, $L_{B}$ is the Lipchitz coefficient of the Boltzmann policy $\pi^{B}(a|s,\theta,x)$ defined in Lemma \ref{lm:lippiB}. $L_r$ and $L_{\pi}$ are Lipchitz coefficient defined in Assumptions \ref{as:lip}, $\tau$ is the temperature coefficient of the Boltzmann policy, $\gamma$ is the discount factor of the MDP, and $\epsilon_{kl}$ is the error due to unrealizable policy class (Assumption \ref{as:tv}).
\end{lemma}
\begin{proof}

Let $s^*=\arg\max_{s}\|U^{\pi}_{off}(s,\theta,x) - U^{\pi}(s,\theta,x)\|$

\begin{equation}\label{eq:woff1}
\begin{split}
&\|U^{\pi}_{off}(s^*,\theta,x) - U^{\pi}(s^*,\theta,x)\| \\
\stackrel{(1)}{\leq}& \Big\| \int \pi^{B}(a|s^*,\theta,x)\big(\nabla  r(s^*,a,x) + \gamma\mathbb{E}_{s' \sim P(\cdot|s^*,a)}[U^{\pi}(s',\theta,x)]\big) \;da \\
&- \int \pi(a|s^*,\theta)\big(\nabla  r(s^*,a,x) + \gamma\mathbb{E}_{s' \sim P(\cdot|s^*,a)}[U^{\pi}_{off}(s',\theta,x)]\big) \;da
\Big\| \\
\stackrel{(2)}{\leq}& \|\nabla  r(s^*,a,x)\|\int |\pi^{B}(a|s^*,\theta,x) - \pi(a|s^*,\theta)|\;da\\
&+\gamma\int \pi^{B}(a|s^*,\theta,x)\mathbb{E}_{s' \sim P(\cdot|s^*,a)}[\|U^{\pi}_{off}(s',\theta,x) - U^{\pi}(s',\theta,x)\|] \;da\\
&+\gamma\|U^{\pi}_{off}(s',\theta,x)\|\int |\pi^{B}(a|s^*,\theta,x) - \pi(a|s^*,\theta)|\;da\\
\stackrel{(3)}{\leq}& \frac{L_r}{1-\gamma}\int |\pi^{B}(a|s^*,\theta,x) - \pi(a|s^*,\theta)|\;da\\
&+ \gamma\|U^{\pi}_{off}(s^{*},\theta,x) - U^{\pi}(s^{*},\theta,x)\|\\
&\implies\|U^{\pi}_{off}(s^{*},\theta,x) - U^{\pi}(s^{*},\theta,x)\| \stackrel{(4)}{\leq}\frac{L_r}{(1-\gamma)^2}\int |\pi^{B}(a|s^*,\theta,x) - \pi(a|s^*,\theta)|\;da 
\end{split}    
\end{equation}

Let us now consider the term $\int |\pi^{B}(a|s^*,\theta,x) - \pi(a|s^*,\theta)|\;da$.
\begin{equation}\label{eq:woff2}
\begin{split}
\int |\pi^{B}(a|s^*,\theta,x) - \pi(a|s^*,\theta)|\;da \stackrel{(1)}{\leq}& \int |\pi^{B}(a|s^*,\theta,x) - \pi^{B}(a|s^*,\theta^{*}(x),x)|\;da \\
&+ \int |\pi^{B}(a|s^*,\theta^{*}(x),x) - \pi(a|s^*,\theta^{*}(x))|\;da \\
&+\int |\pi(a|s^*,\theta^{*}(x)) - \pi(a|s^*,\theta)|\;da\\
\stackrel{(2)}{\leq}& \frac{Vol(\mathcal{A})L_{B}}{\tau}\|\theta - \theta^{*}(x)\|\quad(\text{Using Lemma \ref{lm:pibbound}}) \\
&+ \sqrt{2\epsilon_{kl}} \\
&+Vol(\mathcal{A})L_{\pi}\|\theta - \theta^{*}(x)\|\quad(\text{Using Assumption \ref{as:lip}.5})
\end{split}
\end{equation}

Here, we obtained the second term in (2) by using the following inequality:
\begin{equation}
\begin{split}
&\int |\pi^{B}(a|s^*,\theta^{*}(x),x) - \pi(a|s^*,\theta^{*}(x))|\;da \\
=& 2d_{TV}(\pi^{B}(\cdot|s^*,\theta^{*}(x),x),\pi(\cdot|s^*,\theta^{*}(x))) \\
\leq& \sqrt{2D_{KL}(\pi(\cdot|s^*,\theta^{*}(x))||\pi^{B}(\cdot|s^*,\theta^{*}(x),x))}\\
\leq& \sqrt{2\int \pi(a|s,\theta^{*}(x))|\log\pi^{B}(a|s,\theta^*(x),x) - \log\pi(a|s,\theta^{*}(x))|\;da }\\
\leq &\sqrt{2\epsilon_{kl}}  \quad(\text{Using Assumption \ref{as:tv}})
\end{split}
\end{equation}

Using Eq. \eqref{eq:woff1} and Eq. \eqref{eq:woff2}, we have:
\begin{equation}
\begin{split}
\forall\;s\;\;&\|U^{\pi}_{off}(s,\theta,x) - U^{\pi}(s,\theta,x)\|\\ 
\leq&\|U^{\pi}_{off}(s^{*},\theta,x) - U^{\pi}(s^{*},\theta,x)\| \\
\leq& \frac{L_r\sqrt{2\epsilon_{kl}}}{(1-\gamma)^2}
+ \frac{Vol(\mathcal{A})L_rL_{\pi}}{(1-\gamma)^2}\|\theta - \theta^{*}(x)\| + \frac{Vol(\mathcal{A})L_rL_{B}}{\tau(1-\gamma)^2}\|\theta - \theta^{*}(x)\|\\
\end{split}
\end{equation}

Let us now consider the bound on $\|W^{\pi}_{off}(s,a,\theta,x) - W^{\pi}(s,a,\theta,x)\|$.

\begin{equation}
\begin{split}
&\|W^{\pi}_{off}(s,a,\theta,x) - W^{\pi}(s,a,\theta,x)\| \\ 
\leq & \|\nabla  r(s,a, x) + \gamma\mathbb{E}_{s' \sim P(\cdot|s,a)}[U^{\pi}_{off}(s',\theta,x)] - \nabla  r(s,a, x) - \gamma\mathbb{E}_{s' \sim P(\cdot|s,a)}[U^{\pi}(s',\theta,x)] \|\\
\leq &\gamma\mathbb{E}_{s' \sim P(\cdot|s,a)}[\|U^{\pi}_{off}(s',\theta,x) - U^{\pi}(s',\theta,x)\|]\\
\leq &\gamma\|U^{\pi}_{off}(s^{*},\theta,x) - U^{\pi}(s^{*},\theta,x)\| \\
\leq& \frac{\gamma L_r\sqrt{2\epsilon_{kl}}}{(1-\gamma)^2}
+ \frac{Vol(\mathcal{A})\gamma L_rL_{\pi}}{(1-\gamma)^2}\|\theta - \theta^{*}(x)\| + \frac{Vol(\mathcal{A})\gamma L_rL_{B}}{\tau(1-\gamma)^2}\|\theta - \theta^{*}(x)\|\\
\end{split}
\end{equation}
    
\end{proof}

\subsection{Proof of Theorem \ref{thm:conv}}
\begin{theorem}\label{thm:app-conv}
Suppose Assumption \ref{as:lip} to \ref{as:selector} holds true. Then, Algorithm \ref{alg:1} obtains the following convergence rate : 
\begin{equation}
\begin{split}
\frac{1}{T}\sum_{t=0}^{T-1}\|\nabla \Phi(x_t)\|^2 \leq& O\Big(\frac{1}{T}\Big) + O(\epsilon_{kl}^2) + O(\epsilon_{fd}^2) + O\Big(\frac{1}{B}\Big) +O(\epsilon_{approx}^2) + O(\exp^{-K})\\ 
&+O\Big(\frac{\gamma^{2H}}{B}\Big)
+O(\gamma^{2H}) + O(\epsilon_{approx})  
+O(\epsilon_{kl})\\ 
\end{split}
\end{equation}
By choosing, $T = \Theta(\epsilon^{-1})$, $B = \Theta(\epsilon^{-1})$, $K=\Theta(\log(\epsilon^{-1}))$, and $H=\Theta(\log(\epsilon^{-1})/\log(\gamma^{-1}))$, we obtain the iteration complexity of $T = O(\epsilon^{-1})$ and the sample complexity of $T.K.B.H = \tilde{O}(\epsilon^{-2})$. Therefore, we have:
\begin{equation}
\begin{split}
\frac{1}{T}\sum_{t=0}^{T-1}\|\nabla \Phi(x_t)\|^2 &\leq O(\epsilon) + O(\epsilon_{approx}) + O(\epsilon_{fd}) + O(\epsilon_{kl})
\end{split}    
\end{equation}
Here, $\epsilon_{approx}$ is Q-value function approximation error defined in Assumption \ref{as:approx}. Further, $\epsilon_{fd}$ and $\epsilon_{kl}$ are errors introduced because of the limited policy class as defined in Assumption \ref{as:tv}. $B$ is the batch size used for empirical expectation, $H$ is the horizon length used for estimation of infinite horizon quantities, $K$ is the number of gradient updates for inner level objective optimization, $T$ is the number of gradient updates for outer level objective, and $\tau$ is the temperature coefficient defined in Boltzmann policy (Eq. \eqref{eq:boltz}).
\end{theorem}

\begin{proof}
Let $\phi(x) = \mathbb{E}_{d_i \sim P(d_i, \theta^*(x))} [l(\{d_i\}_{i=0}^{I-1}, x)]$.

\begin{equation}
\begin{split}
\nabla  \tilde{\phi}(x, \theta) =\;&  \mathbb{E}_{d_i \sim P(d_i, \theta)} [\nabla_x l(\{d_i\}_{i=0}^{I-1}, x)] \\
&+ \tau^{-1}\mathbb{E}_{d_i \sim P(d_i, \theta)} \Big[l(\{d_i\}_{i=0}^{I-1}, x) \sum_{i}\sum_{h}\big(W^{\pi}(s_h^{i},a_h^{i},\theta,x) - U^{\pi}(s_h^{i},\theta,x)\big)\Big]\\                
\end{split}
\end{equation}

\begin{equation}
\begin{split}
&\nabla \tilde{\phi}(x, \theta, B)\\
= & \frac{1}{B}\sum_{j=0}^{B-1}\nabla_x l(\{d_{i,j}\}_{0}^{I-1}, x) \\
&+ \tau^{-1}\frac{1}{B}\sum_{j=0}^{B-1}\Big(l(\{d_{i,j}\}_{0}^{I-1}, x) \sum_{i}\sum_{h}\big(\widehat{W}^{\pi}_{off,w}(s_h^{i,j},a_h^{i,j},\theta,x) - \widehat{U}^{\pi}_{off,z}(s_h^{i,j},\theta,x)\big)\Big)    
\end{split}    
\end{equation}

Here, $\widehat{W}^{\pi}_{off,w}(s_h^{i,j},a_h^{i,j},\theta,x)$ is as estimator for $W^{\pi}_{off}(s_h^{i},a_h^{i},\theta,x)$ with parameter $w$ and $\widehat{U}^{\pi}_{off,z}(s_h^{i,j},\theta,x)$ is an estimator for $U^{\pi}_{off}(s_h^{i},\theta,x)$ with parameter $z$. The parameters $w$ and $z$ are updated using the Q-learning update rule from Algorithm 1 in \citep{gaur2024closing}.

We use the following update rule for the reward parameter $x$ using approximate hypergradient theorem (Theorem \ref{thm:gradient}):
\begin{equation}
x_{t+1} = x_{t} - \beta_t\nabla \tilde{\phi}(x,\theta,B)    
\end{equation}
Using $L-$smoothness we have:

\begin{equation}\label{eq:50}
\begin{split}
\phi(x_{t+1}) 
&\leq \phi(x_{t}) + \langle \nabla \phi(x_t), x_{t+1} - x_t\rangle + \frac{L}{2}\|x_{t+1} - x_t\|^2 \\ 
&\leq \phi(x_{t}) - \beta_t\langle \nabla \phi(x_t), \nabla \tilde{\phi}(x_t,\theta^{K}_t,B)\rangle + \frac{\beta_t^2L}{2}\|\nabla \tilde{\phi}(x_t,\theta^{K}_t,B)\|^2 \\ 
&\leq \phi(x_{t}) - \beta_t\langle \nabla \phi(x_t), \nabla \tilde{\phi}(x_t,\theta^{K}_t,B)\rangle + \frac{\beta_t^2L}{2}\|\nabla \tilde{\phi}(x_t,\theta^{K}_t,B)\|^2 \\ 
&\leq \phi(x_{t}) -\beta_t\|\nabla \phi(x_t)\|^2 - \beta_t\langle \nabla \phi(x_t), \nabla \tilde{\phi}(x_t,\theta^{K}_t,B) - \nabla \phi(x_t)\rangle + \frac{\beta_t^2L}{2}\|\nabla \tilde{\phi}(x_t,\theta^{K}_t,B)\|^2 \\ 
&\leq \phi(x_{t}) -\frac{\beta_t}{2}\|\nabla \phi(x_t)\|^2 + \frac{\beta_t}{2}\| \nabla \tilde{\phi}(x_t,\theta^{K}_t,B) - \nabla \phi(x_t)\|^2 + \frac{\beta_t^2L}{2}\|\nabla \tilde{\phi}(x_t,\theta^{K}_t,B)\|^2 \\ 
&\leq \phi(x_{t}) -(\frac{\beta_t}{2} -\beta_t^2L)\|\nabla \phi(x_t)\|^2 + (\frac{\beta_t}{2}+\beta_t^2L)\| \nabla \tilde{\phi}(x_t,\theta^{K}_t,B) - \nabla \phi(x_t)\|^2 \\ 
\end{split}
\end{equation}

Let us consider the term $\| \nabla \tilde{\phi}(x_t,\theta^{K}_t,B) - \nabla \phi(x_t)\|$:

\begin{equation}\label{eq:51} 
\begin{split}
\| \nabla \tilde{\phi}(x_t,\theta^{K}_t,B) - \nabla \phi(x_t)\| 
& = \| \nabla \tilde{\phi}(x_t,\theta^{K}_t,B) - \nabla \tilde{\phi}(x_t,\theta^{K}_t) + \nabla \tilde{\phi}(x_t,\theta^{K}_t) - \nabla \phi(x_t)\| \\    
& \leq \| \nabla \tilde{\phi}(x_t,\theta^{K}_t,B) - \nabla \tilde{\phi}(x_t,\theta^{K}_t)\| + \|\nabla \tilde{\phi}(x_t,\theta^{K}_t) - \nabla \phi(x_t)\| \\    
\end{split}
\end{equation}

Let us consider the term $\|\nabla \tilde{\phi}(x_t,\theta^{K}_t) - \nabla \phi(x_t)\|$:

\begin{equation}
\begin{split}
\|\nabla \tilde{\phi}(x_t,\theta^{K}_t) - \nabla \phi(x_t)\| 
&= \|\nabla \tilde{\phi}(x_t,\theta^{K}_t) - \nabla \tilde{\phi}(x_t,\theta^*(x_t)) - \Psi(x_t)\| \\   
&\leq  \|\nabla \tilde{\phi}(x_t,\theta^{K}_t) - \nabla \tilde{\phi}(x_t,\theta^*(x_t))\| + \|\Psi(x_t)\| \\   
&\stackrel{(1)}{\leq}  \|\nabla \tilde{\phi}(x_t,\theta^{K}_t) - \nabla \tilde{\phi}(x_t,\theta^*(x_t))\| + O(\epsilon_{kl}) + O(\epsilon_{fd}) \\   
\end{split}
\end{equation}

Here, (1) is obtained by using Lemma \ref{lm:app-psibound} to provide a bound for $\|\Psi(x_t)\|$.

Using Lemma \ref{lm:7}, we have:
\begin{equation}\label{eq:53}
\begin{split}
&\|\nabla \tilde{\phi}(x_t,\theta^{K}_t) - \nabla \phi(x_t)\|\\ 
\leq& O(\epsilon_{kl}) + O(\epsilon_{fd}) +  
O(\|\theta^{K}_t - \theta^{*}(x_t)\|) + 
O\Big(\frac{\|\theta^{K}_t - \theta^*(x_t)\|}{\tau}\Big)+
O\Big(\frac{\|\theta^{K}_t - \theta^*(x_t)\|}{\tau^2}\Big)
\end{split}
\end{equation}

Let us look at the term $\| \nabla \tilde{\phi}(x_t,\theta^{K}_t,B) - \nabla \tilde{\phi}(x_t,\theta^{K}_t)\|$:

\begin{equation}
\begin{split}
&\| \nabla \tilde{\phi}(x_t,\theta^{K}_t,B) - \nabla \tilde{\phi}(x_t,\theta^{K}_t)\| \\
\leq &\Big\| \mathbb{E}_{d_i \sim P(d_i, \theta)} [\nabla_x l(\{d_i\}_{i=0}^{I-1}, x_t)]\\ 
&+ \tau^{-1}\mathbb{E}_{d_i \sim P(d_i, \theta)} \Big[l(\{d_i\}_{i=0}^{I-1}, x_t) \sum_{i}\sum_{h}\big(W^{\pi}(s_h^{i},a_h^{i},\theta^{K}_t,x_t) - U^{\pi}(s_h^{i},\theta^{K}_t,x_t)\big)\Big] \\ 
&-\frac{1}{B}\sum_{j=0}^{B-1}\nabla_x l(\{d_{i,j}\}_{0}^{I-1}, x_t)\\
&- \tau^{-1}\frac{1}{B}\sum_{j=0}^{B-1}\Big(l(\{d_{i,j}\}_{0}^{I-1}, x_t) \sum_{i}\sum_{h}\big(\widehat{W}^{\pi}_{off,w}(s_h^{i,j},a_h^{i,j},\theta_t^{K},x_t) - \widehat{U}^{\pi}_{off,z}(s_h^{i,j},\theta_t^{K},x_t)\big)\Big) \Big\| \\
\leq &\Big\| \mathbb{E}_{d_i \sim P(d_i, \theta)} [\nabla_x l(\{d_i\}_{i=0}^{I-1}, x_t)] - \frac{1}{B}\sum_{j=0}^{B-1}\nabla_x l(\{d_{i,j}\}_{0}^{I-1}, x_t)\Big\| \\
&+\frac{1}{\tau}\Big\|\mathbb{E}_{d_i \sim P(d_i, \theta)} \Big[l(\{d_i\}_{i=0}^{I-1}, x_t) \sum_{i}\sum_{h}\big(W^{\pi}(s_h^{i},a_h^{i},\theta^{K}_t,x_t) - U^{\pi}(s_h^{i},\theta^{K}_t,x_t)\big)\Big] \\
& - \frac{1}{B}\sum_{j=0}^{B-1}\Big(l(\{d_{i,j}\}_{0}^{I-1}, x_t) \sum_{i}\sum_{h}\big(\widehat{W}^{\pi}_{off,w}(s_h^{i,j},a_h^{i,j},\theta_t^{K},x_t) - \widehat{U}^{\pi}_{off,z}(s_h^{i,j},\theta_t^{K},x_t)\big)\Big)\Big\|\\
\end{split}
\end{equation}

Let us consider the term $\| \mathbb{E}_{d_i \sim P(d_i, \theta)} [\nabla_x l(\{d_i\}_{i=0}^{I-1}, x_t)] - \frac{1}{B}\sum_{j=0}^{B-1}\nabla_x l(\{d_{i,j}\}_{0}^{I-1}, x_t)\|$. Using Hoeffding inequality with probability $1-\delta$ we obtain:

\begin{equation}\label{eq:64}
\begin{split}
\Big\| \mathbb{E}_{d_i \sim P(d_i, \theta)} [\nabla_x l(\{d_i\}_{i=0}^{I-1}, x_t)] - \frac{1}{B}\sum_{j=0}^{B-1}\nabla_x l(\{d_{i,j}\}_{0}^{I-1}, x_t)\Big\| \leq 2C_{lx}\sqrt{\frac{\log(2/\delta)}{2B}}  
\end{split}    
\end{equation}

Let us consider the following term:

\begin{equation}
\begin{split}
&\Big\|
\frac{1}{B}\sum_{j=0}^{B-1}\Big(l(\{d_{i,j}\}_{0}^{I-1}, x_t) \sum_{i}\sum_{h}\big(\widehat{W}^{\pi}_{off,w}(s_h^{i,j},a_h^{i,j},\theta_t^{K},x_t) - \widehat{U}^{\pi}_{off,z}(s_h^{i,j},\theta_t^{K},x_t)\big)\Big) \\
&- \mathbb{E}_{d_i \sim P(d_i, \theta)} \Big[l(\{d_i\}_{i=0}^{I-1}, x_t) \sum_{i}\sum_{h}\big(W^{\pi}(s_h^{i},a_h^{i},\theta^{K}_t,x_t) - U^{\pi}(s_h^{i},\theta^{K}_t,x_t)\big)\Big] \Big\| \\   
\leq & 
\Big\|
\frac{1}{B}\sum_{j=0}^{B-1}\Big(l(\{d_{i,j}\}_{0}^{I-1}, x_t) \sum_{i}\sum_{h}\big(\widehat{W}^{\pi}_{off,w}(s_h^{i,j},a_h^{i,j},\theta_t^{K},x_t) - \widehat{U}^{\pi}_{off,z}(s_h^{i,j},\theta_t^{K},x_t)\big)\Big) \\
&-\mathbb{E}_{d_i \sim P(d_i, \theta)} \Big[l(\{d_i\}_{i=0}^{I-1}, x_t) \sum_{i}\sum_{h}\big(\widehat{W}^{\pi}_{off,w}(s_h^{i},a_h^{i},\theta_t^K,x_t) - \widehat{U}^{\pi}_{off,z}(s_h^{i},\theta_t^K,x_t)\big)\Big] \Big\| \\
&+\Big\|
\mathbb{E}_{d_i \sim P(d_i, \theta)} \Big[l(\{d_i\}_{i=0}^{I-1}, x_t) \sum_{i}\sum_{h}\big(\widehat{W}^{\pi}_{off,w}(s_h^{i},a_h^{i},\theta_t^K,x_t) - \widehat{U}^{\pi}_{off,z}(s_h^{i},\theta_t^K,x_t)\big)\Big] \\
&-\mathbb{E}_{d_i \sim P(d_i, \theta)} \Big[l(\{d_i\}_{i=0}^{I-1}, x_t) \sum_{i}\sum_{h}\big(W^{\pi}_{off}(s_h^{i},a_h^{i},\theta^{K}_t,x_t) - U^{\pi}_{off}(s_h^{i},\theta^{K}_t,x_t)\big)\Big] \Big\| \\
&+ \Big\|\mathbb{E}_{d_i \sim P(d_i, \theta)} \Big[l(\{d_i\}_{i=0}^{I-1}, x_t) \sum_{i}\sum_{h}\big(W^{\pi}_{off}(s_h^{i},a_h^{i},\theta^{K}_t,x_t) - U^{\pi}_{off}(s_h^{i},\theta^{K}_t,x_t)\big)\Big]  \\
&-\mathbb{E}_{d_i \sim P(d_i, \theta)} \Big[l(\{d_i\}_{i=0}^{I-1}, x_t) \sum_{i}\sum_{h}\big(W^{\pi}(s_h^{i},a_h^{i},\theta^{K}_t,x_t) - U^{\pi}(s_h^{i},\theta^{K}_t,x_t)\big)\Big] \Big\| \\
\end{split}    
\end{equation}

Let us now consider the following term. Using Hoeffding inequality with probability $1-\delta$ we obtain:

\begin{equation}\label{eq:66}
\begin{split}
&\Big\|
\frac{1}{B}\sum_{j=0}^{B-1}\Big(l(\{d_{i,j}\}_{0}^{I-1}, x_t) \sum_{i}\sum_{h}\big(\widehat{W}^{\pi}_{off,w}(s_h^{i,j},a_h^{i,j},\theta_t^{K},x_t) - \widehat{U}^{\pi}_{off,z}(s_h^{i,j},\theta_t^{K},x_t)\big)\Big) \\
&- \mathbb{E}_{d_i \sim P(d_i, \theta)} \Big[l(\{d_i\}_{i=0}^{I-1}, x_t) \sum_{i}\sum_{h}\big(\widehat{W}^{\pi}_{off,w}(s_h^{i},a_h^{i},\theta_t^K,x_t) - \widehat{U}^{\pi}_{off,z}(s_h^{i},\theta_t^K,x_t)\big)\Big] \Big\| \\
&\leq 2C_zC_uHI\sqrt{\frac{\log(2/\delta)}{2B}}
\end{split}
\end{equation}

Let us now consider the following term. $\widehat{W}^{\pi}_{off,w}(s_h^{i},a_h^{i},\theta_t^K,x_t)$ is a sub-optimal estimate of $W^{\pi}(s_h^{i},a_h^{i},\theta^{K}_t,x_t)$ using neural network with parameter $y$. Further, $\widehat{U}^{\pi}_{off,z}(s_h^{i},\theta_t^K,x_t)$ is a sub-optimal estimate of $U^{\pi}(s_h^{i},\theta^{K}_t,x_t)$ using neural network with parameter $w$. Therefore, using the result from \citep{gaur2024closing}, we get: 

\begin{equation}\label{eq:67}
\begin{split}
&\Big\|
\mathbb{E}_{d_i \sim P(d_i, \theta)} \Big[l(\{d_i\}_{i=0}^{I-1}, x_t) \sum_{i}\sum_{h}\big(\widehat{W}^{\pi}_{off,w}(s_h^{i},a_h^{i},\theta_t^K,x_t) - \widehat{U}^{\pi}_{off,z}(s_h^{i},\theta_t^K,x_t)\big)\Big] \\
&-\mathbb{E}_{d_i \sim P(d_i, \theta)} \Big[l(\{d_i\}_{i=0}^{I-1}, x_t) \sum_{i}\sum_{h}\big(W^{\pi}_{off}(s_h^{i},a_h^{i},\theta^{K}_t,x_t) - U^{\pi}_{off}(s_h^{i},\theta^{K}_t,x_t)\big)\Big] \Big\|\\
&\leq O(\epsilon_{approx}) + O\Big(\frac{1}{\sqrt{B}}\Big)
\end{split}    
\end{equation}

Let us consider the following term. Using Lemma \ref{lm:woff} we have:

\begin{equation}\label{eq:68}
\begin{split}
&\Big\|\mathbb{E}_{d_i \sim P(d_i, \theta)} \Big[l(\{d_i\}_{i=0}^{I-1}, x_t) \sum_{i}\sum_{h}\big(W^{\pi}_{off}(s_h^{i},a_h^{i},\theta^{K}_t,x_t) - U^{\pi}_{off}(s_h^{i},\theta^{K}_t,x_t)\big)\Big]  \\
&-\mathbb{E}_{d_i \sim P(d_i, \theta)} \Big[l(\{d_i\}_{i=0}^{I-1}, x_t) \sum_{i}\sum_{h}\big(W^{\pi}(s_h^{i},a_h^{i},\theta^{K}_t,x_t) - U^{\pi}(s_h^{i},\theta^{K}_t,x_t)\big)\Big] \Big\|\\
&\leq \frac{(1+\gamma)C_lL_rHI\sqrt{2\epsilon_{kl}}}{(1-\gamma)^2} +\frac{(1+\gamma)C_lHIVol(\mathcal{A})L_r((L_{B}/\tau)+L_{\pi})}{(1-\gamma)^2}\|\theta_t^K - \theta^{*}(x_t)\|
\end{split}    
\end{equation}

Therefore, using Eq. \eqref{eq:64} and Eq.~\eqref{eq:66}--\eqref{eq:68},  we have:

\begin{equation}\label{eq:60}
\begin{split}
\| \nabla \tilde{\phi}(x_t,\theta^{K}_t,B) - \nabla \tilde{\phi}(x_t,\theta^{K}_t)\| \leq &  2C_{lx}\sqrt{\frac{\log(2/\delta)}{2B}} + \frac{2C_zC_uHI}{\tau}\sqrt{\frac{\log(2/\delta)}{2B}} + O\Big(\frac{\epsilon_{approx}}{\tau}\Big)\\
+& O\Big(\frac{1}{\tau\sqrt{B}}\Big)+\frac{(1+\gamma)C_lL_rHI\sqrt{2\epsilon_{kl}}}{(1-\gamma)^2}\\ +&\frac{(1+\gamma)C_lHIVol(\mathcal{A})L_r((L_{B}/\tau)+L_{\pi})}{(1-\gamma)^2}\|\theta_t^K - \theta^{*}(x_t)\|
\end{split}    
\end{equation}

Using Eq. \eqref{eq:60} and Eq. \eqref{eq:53} in Eq. \eqref{eq:51}:

\begin{equation}
\begin{split}
&\| \nabla \tilde{\phi}(x_t,\theta^{K}_t,B) - \nabla \phi(x_t)\|\\
\leq &O(\epsilon_{kl}) + O(\epsilon_{fd}) +  
O(\|\theta^{K}_t - \theta^{*}(x_t)\|) + 2C_{lx}\sqrt{\frac{\log(2/\delta)}{2B}}\\
& + \frac{2C_zC_uHI}{\tau}\sqrt{\frac{\log(2/\delta)}{2B}} + O\Big(\frac{\epsilon_{approx}}{\tau}\Big) + O\Big(\frac{1}{\tau\sqrt{B}}\Big)\\
&+ \frac{(1+\gamma)C_lL_rHI\sqrt{2\epsilon_{kl}}}{\tau(1-\gamma)^2}+ \frac{(1+\gamma)C_lHIVol(\mathcal{A})L_rL_{\pi}}{\tau(1-\gamma)^2}\|\theta^{K}_{t} - \theta^{*}(x_t)\| \\
&+ O\Big(\frac{\|\theta^{K}_t - \theta^*(x_t)\|}{\tau^2}\Big) + O\Big(\frac{\|\theta^{K}_t - \theta^*(x_t)\|}{\tau}\Big)
\end{split}    
\end{equation}

Using Lemma \ref{lm:app-inner} for $O(\|\theta^{K}_t - \theta^*(x_t)\|)$ we get: 
\begin{equation}\label{eq:62}
\begin{split}
&\| \nabla \tilde{\phi}(x_t,\theta^{K}_t,B) - \nabla \phi(x_t)\|^2\\
\leq & O(\epsilon_{kl}) + O(\epsilon_{kl}^2) + O(\epsilon_{fd}^2) + O\Big(\frac{1}{B}\Big) + O\Big(\frac{1}{\tau^2B}\Big)+O\Big(\frac{\epsilon_{approx}^2}{\tau^2}\Big)+O\Big(\frac{\epsilon_{kl}^2}{\tau^2}\Big) \\
&+ O(\exp^{-K}) + O\Big(\frac{1}{B}\Big) + O\Big(\frac{\gamma^{2H}}{B}\Big) + O(\gamma^{2H}) + O(\epsilon_{approx}) \\
&+
O\Big(\frac{\exp^{-K}}{\tau^2}\Big) + 
O\Big(\frac{1}{\tau^2 B}\Big) +
O\Big(\frac{\gamma^{2H}}{\tau^2B}\Big) + O\Big(\frac{\gamma^{2H}}{\tau^2}\Big) + O\Big(\frac{\epsilon_{approx}}{\tau^2}\Big)\\
&+
O\Big(\frac{\exp^{-K}}{\tau^4}\Big) + 
O\Big(\frac{1}{\tau^4 B}\Big) +
O\Big(\frac{\gamma^{2H}}{\tau^4B}\Big) + O\Big(\frac{\gamma^{2H}}{\tau^4}\Big) + O\Big(\frac{\epsilon_{approx}}{\tau^4}\Big)
\end{split}    
\end{equation}

Let us recall Eq. \eqref{eq:50} and set $\beta_t = \beta$.Also, $\beta L < 1/2$.

\begin{equation}
\begin{split}
&\phi(x_{t+1}) \leq  \phi(x_{t}) -(\frac{\beta}{2} -\beta^2L)\|\nabla \phi(x_t)\|^2 + (\frac{\beta}{2}+\beta^2L)\| \nabla \tilde{\phi}(x_t,\theta^{K}_t,B) - \nabla \phi(x_t)\|^2 \\
\Rightarrow& \|\nabla \phi(x_t)\|^2 \leq \frac{2}{\beta - 2\beta^2L}(\phi(x_t) - \phi(x_{t+1})) + \frac{2}{1-2\beta L}\| \nabla \tilde{\phi}(x_t,\theta^{K}_t,B) - \nabla \phi(x_t)\|^2 \\ 
\Rightarrow& \frac{1}{T}\sum_{t=0}^{T-1}\|\nabla \phi(x_t)\|^2 \leq \frac{2}{\beta - 2\beta^2L}\frac{(\phi(x_0) - \phi(x_{T}))}{T} + \frac{2}{(1-2\beta L)T}\sum_{t=0}^{T-1}\|\nabla \tilde{\phi}(x_t,\theta^{K}_t,B) - \nabla \phi(x_t)\|^2 \\
\end{split}
\end{equation}

Using Eq. \eqref{eq:62}:

\begin{equation}
\begin{split}
\Rightarrow \frac{1}{T}\sum_{t=0}^{T-1}\|\nabla \phi(x_t)\|^2 \leq& O\Big(\frac{1}{T}\Big) + O(\epsilon_{kl}^2) + O(\epsilon_{fd}^2) + O\Big(\frac{1}{B}\Big) + O\Big(\frac{1}{\tau^2B}\Big)+O\Big(\frac{\epsilon_{approx}^2}{\tau^2}\Big)+O\Big(\frac{\epsilon_{kl}^2}{\tau^2}\Big) \\
&+ O(\exp^{-K}) + O\Big(\frac{1}{B}\Big) + O\Big(\frac{\gamma^{2H}}{B}\Big) + O(\gamma^{2H}) + O(\epsilon_{approx}) + O(\epsilon_{kl}) \\
&+
O\Big(\frac{\exp^{-K}}{\tau^2}\Big) + 
O\Big(\frac{1}{\tau^2 B}\Big) +
O\Big(\frac{\gamma^{2H}}{\tau^2B}\Big) + O\Big(\frac{\gamma^{2H}}{\tau^2}\Big) + O\Big(\frac{\epsilon_{approx}}{\tau^2}\Big)\\
&+
O\Big(\frac{\exp^{-K}}{\tau^4}\Big) + 
O\Big(\frac{1}{\tau^4 B}\Big) +
O\Big(\frac{\gamma^{2H}}{\tau^4B}\Big) + O\Big(\frac{\gamma^{2H}}{\tau^4}\Big) + O\Big(\frac{\epsilon_{approx}}{\tau^4}\Big)
\end{split}    
\end{equation}

Let $\tau = \Theta(1)$. We obtain:

\begin{equation}
\begin{split}
\Rightarrow \frac{1}{T}\sum_{t=0}^{T-1}\|\nabla \phi(x_t)\|^2 \leq& O\Big(\frac{1}{T}\Big) + O(\epsilon_{kl}^2) + O(\epsilon_{fd}^2) + O\Big(\frac{1}{B}\Big) +O(\epsilon_{approx}^2) + O(\exp^{-K})\\ 
&+O\Big(\frac{\gamma^{2H}}{B}\Big)
+O(\gamma^{2H}) + O(\epsilon_{approx})  
+O(\epsilon_{kl})\\ 
\end{split}    
\end{equation}

Let us substitute $T = \Theta(\epsilon^{-1})$, $B = \Theta(\epsilon^{-1})$, $K=\Theta(\log(\epsilon^{-1}))$, and $H=\Theta(\log(\epsilon^{-1})/\log(\gamma^{-1}))$.

\begin{equation}
\begin{split}
\Rightarrow \frac{1}{T}\sum_{t=0}^{T-1}\|\nabla \phi(x_t)\|^2 &\leq O(\epsilon) + O(\epsilon_{approx}) + O(\epsilon_{fd})+ O(\epsilon_{kl})
\end{split}    
\end{equation}

Therefore, the iteration complexity is $T = O(\epsilon^{-1})$ and the sample complexity is $T.K.B.H = \tilde{O}(\epsilon^{-2})$

\end{proof}


\end{document}